\documentclass{ijicc}
\usepackage{graphicx,times,amsmath,amssymb,epsfig}

\begin{document}
\def\prob#1{\mbox{Prob$\{\mbox{#1}\}$}}
\def\diag#1{\mbox{diag\hspace{.01in}}(#1)}
\def\colour#1{\mbox{\bf Col}(#1)}

\newtheorem{thm}{Theorem} %[section]
\newtheorem{lem}[thm]{Lemma}
\newtheorem{cor}[thm]{Corollary}
\newtheorem{prop}[thm]{Proposition}

\newtheorem{dftemp}[thm]{Definition}
\newtheorem{extemp}[thm]{Example}
\newtheorem{rmktemp}[thm]{Remark}
\newtheorem{convtemp}[thm]{Convention}
\newenvironment{defn}{\begin{dftemp}\normalfont}{\end{dftemp}}
\newenvironment{ex}{\begin{extemp}\normalfont}{\end{extemp}}
\newenvironment{rem}{\begin{rmktemp}\normalfont}{\end{rmktemp}}
\newenvironment{conv}{\begin{convtemp}\normalfont}{\end{convtemp}}

\def\square{\hfill${\vcenter{\vbox{\hrule height.4pt \hbox{\vrule width.4pt
height7pt \kern7pt \vrule width.4pt} \hrule height.4pt}}}$}

\newenvironment{pf}{{\it Proof:}\quad}{\square \vskip 12pt}

\catchline{}{}{}{}{}
%
%%%%%%%%%%%%%%%%%%%%%%%%%%%%%%%%%%%%%%%%%%%%%%%%%%%%%%%%%%%%%%%%%%%%

\markboth{Mitavskiy, Rowe, Cannings} {Geiringer Theorem, Partially Observable Markov decision Processes and other Monte-Carlo search Methods.}

\title{A Version of Geiringer-like Theorem for Decision Making in the Environments with Randomness and Incomplete Information. %\LaTeX\footnote{For the title, try not to use more
%than 3 lines. Typeset the title in 10 pt bold and uppercase.}
}

\author{Boris Mitavskiy\thanks{EPSRC EP/D003/05/1
``Amorphous Computing" and EPSRC EP/I009809/1 ``Evolutionary Approximation Algorithms for Optimization: Algorithm Design and Complexity Analysis" Grants.}}

\address{Department of Computer Science, Aberystwyth University,\\
Aberystwyth, Ceredigion, SY23 3DB  \\
\email{bom4@aber.ac.uk}}

\author{Jonathan Rowe}

\address{School of Computer Science, University of Birmingham,\\
Edgbaston Birmingham, England, B15 2TT\\
\email{J.E.Rowe@cs.bham.ac.uk}}

\author{Chris Cannings}

\address{School of Mathematics and Statistics, University of Sheffield,\\
Sheffield, U.K., S3 7RH \\
\email{C.Cannings@sheffield.ac.uk}}

\maketitle

\begin{history}
\received{(?? ?? 2011)} \revised{(Day Month Year)}
\end{history}

\begin{abstract}
\textbf{Purpose-} In recent years Monte-Carlo sampling methods, such as Monte Carlo tree search, have achieved tremendous success in model free reinforcement learning. A combination of the so called upper confidence bounds policy to preserve the ``exploration vs. exploitation" balance to select actions for sample evaluations together with massive computing power to store and to update dynamically a rather large pre-evaluated game tree lead to the development of software that has beaten the top human player in the game of Go on a 9 by 9 board. Much effort in the current research is devoted to widening the range of applicability of the Monte-Carlo sampling methodology to partially observable Markov decision processes with non-immediate payoffs. The main challenge introduced by randomness and incomplete information is to deal with the action evaluation at the chance nodes due to drastic differences in the possible payoffs the same action could lead to. The aim of this article is to establish a version of a theorem that originated from population genetics and has been later adopted in evolutionary computation theory that will lead to novel Monte-Carlo sampling algorithms that provably increase the AI potential. Due to space limitations the actual algorithms themselves will be presented in the sequel papers, however, the current paper provides a solid mathematical foundation for the development of such algorithms and explains why they are so promising.\\
\textbf{Design/Methodology/Approach-} In the current paper we set up a mathematical framework, state and prove a version of a Geiringer-like theorem that is very well-suited for the development of Mote-Carlo sampling algorithms to cope with randomness and incomplete information to make decisions. From the framework it will be clear that such algorithm increase what seems like a limited sample of rollouts exponentially in size by exploiting the symmetry within the state space at little or no additional computational cost. Appropriate notions of recombination (or crossover) and schemata are introduced to stay inline with the traditional evolutionary computation terminology. The main theorem is proved using the methodology developed in the PhD thesis of the first author, however the general case of non-homologous recombination presents additional challenges that have been overcome thanks to a lovely application of the classical and elementary tool known as the ``Markov inequality" together with the lumping quotients of Markov chains techniques developed and successfully applied by the authors in the previous research for different purposes. This methodology will be mildly extended to establish the main result of the current article. In addition to establishing the Geiringer-like theorem for Monte Carlo sampling, which is the central objective of this paper, we also strengthen the applicability of the core theorem from the PhD thesis of the first author on which our main result rests. This provides additional theoretical justification for the anticipated success of the presented theory.\\
\textbf{Findings-} This work establishes an important theoretical link between classical population genetics, evolutionary computation theory and model free reinforcement learning methodology.
Not only the theory may explain the success of the currently existing Monte-Carlo tree sampling methodology, but it also leads to the development of novel Monte-Carlo sampling techniques guided
by rigorous mathematical foundation.\\
%\textbf{Research limitations/implications-} Most of the runtime complexity results
%established so far for the type of algorithms we study apply only to the case of trees.
%Unfortunately \\
\textbf{Practical implications-} The theoretical foundations established in the current work provide guidance for the design of powerful Monte-Carlo sampling algorithms in model free reinforcement learning to tackle numerous problems in computational intelligence.\\
\textbf{Originality/value-} Establishing a Geiringer-like theorem with non-homologous recombination was a long standing open problem in evolutionary computation theory. Apart from overcoming this challenge, in a mathematically elegant fashion and establishing a rather general and powerful version of the theorem, this work leads directly to the development of novel provably powerful algorithms for decision making in the environment involving randomness, hidden or incomplete information.\\
\end{abstract}
\section{Introduction}\label{IntroSect}
A great number of questions in machine learning, computer game intelligence, control theory, and numerous other applications involve the design of algorithms for decision-making by an agent under a specified set of circumstances. In the most general setting, the problem can be described mathematically in terms of the state and action pairs as follows. A state-action pair is an ordered pair of the form $(s, \, \vec{\alpha})$ where $\vec{\alpha} = \{\alpha_1, \, \alpha_2, \ldots, \alpha_n\}$ is the set of actions (or moves, in case the agent is playing a game, for instance) that the agent is capable of taking when it is in the state (or, in case of a game, a state might be sometimes referred to as a position) $s$. Due to randomness, hidden features, lack of memory, limitation of the sensor capabilities etc, the state may be only partially observable by the agent. Mathematically this means that there is a function $\phi: S \rightarrow O$ (as a matter of fact, a random variable with respect the unknown probability space structure on the set $S$) where $S$ is the set of all states which could be either finite or infinite while $O$ is the set (usually finite due to memory limitations) of observations having the property that whenever $\phi(s_1) = \phi(s_2)$ (i.e. whenever the agent can not distinguish states $s_1$ and $s_2$) then the corresponding state action pairs $(s_1, \, \vec{\alpha})$ and $(s_2, \, \vec{\beta})$ are such that $\vec{\alpha} = \vec{\beta}$ (i.e. the agent knows which actions it can possibly take based only on the observation it makes). The general problem of reinforcement learning is to decide which action is best suited given the agent's knowledge (that is the observation that the agent has made as well as the agent's past experience). In computational settings ``suitability" is naturally described in terms of a numerical reward value. In the probability theoretic sense the agent aims to maximize the expected reward (the expected reward considered as a random variable on the enormous and unknown conditional probability space of states given a specific observation and an action taken). Most common models such as POMDPs (partially observable Markov decision processes) assume that the next state and the corresponding numerical rewards depend stochastically only on the current observation and action. In a number of situations the immediate rewards after executing a single action are unknown. The so-called ``model free" reinforcement learning methods, such as Monte Carlo techniques (i.e. algorithms based on repeated random sampling) are exploited to tackle problems of this type. In such algorithms a large number of \emph{rollouts} (i.e. simulations or self-plays) are made and actions are assigned numerical payoffs that get updated dynamically (i.e at every simulation of an algorithm). While the simulated self-plays started with a specific chosen action, say $\alpha$, are entirely random, the action $\alpha$ itself is chosen with respect to a dynamically updated probability distribution which ensures the exploration versus exploitation balance: the technique known as UCB (Upper Confidence Bounds). It may be worth emphasizing that the UCB methodology is based on a solid mathematical foundation (see \cite{UCBAgraval}, \cite{UCBKaelbling} and \cite{UCBAuer}). A combination of UCB with Monte Carlo sampling lead to tremendous break through in computer Go performance level (see \cite{MCTGoChesl} and \cite{MCTGoCoulomR}, for instance) and much research is currently undergoing to widen the applicability of the method. Some of the particularly challenging and interesting directions involve decision making in the environments (or games) involving randomness, hidden information and uncertainty or in ``continuous" environments where appropriate similarities on the set of states must be constructed due to runtime and memory limitations and also action evaluation polices must be enhanced to cope with drastic changes in the payoffs as well as an enormous combinatorial explosion in the branching factor of the decision tree. In recent years a number of heuristic approaches have been proposed based on the existing  probabilistic planning methodology. Despite some of these newly developed methods have already achieved surprisingly powerful performance levels: see \cite{FF-Replan} and \cite{ZinkewitchBucket}, the authors believe there is still room for drastic improvement based on the rigorous mathematical theory originated from classical population genetics (\cite{GeirOrigion}) and later adopted in traditional evolutionary computation theory (\cite{PoliGeir}, \cite{MitavRowGeirMain} \cite{MitavRowGeirGenProgr}). Theorems of this type are known as Geiringer-like results and they address the limiting ``frequency of occurrence" of various sets of ``genes" as recombination is repeatedly applied over time. The main objective of the current work is to establish a rather general and powerful version of a Geiringer-like theorem with ``non-homologous" recombination operators in the setting of Monte Carlo sampling. This theorem leads to simple dynamic algorithms that exploit the intrinsic similarity within the space of observations to increase exponentially the size of the already existing sample of rollouts yielding significantly more informative action-evaluation at very little or even no additional computational cost at all. The details of how this is done will be described in sections~\ref{equivSimSect} and \ref{settingSect}. Due to space limitations, the actual algorithms will appear in sequel papers. As a matter of fact, we believe the interested readers may actually design such algorithms on their own after studying sections~\ref{equivSimSect} and \ref{settingSect}.
\section{Overview}\label{OverViewSect}
Due to the interdisciplinary nature of this work the authors did their best to make the paper accessible on various levels to a potentially wide audience having diverse backgrounds and research interests ranging from practical software engineering to applied mathematics, theoretical computer science and high-level algorithm design based on solid mathematical foundation. The next section (section~\ref{equivSimSect}) is essential for understanding the main idea of the paper. It provides the notation and sets up a rigorous mathematical framework, while the informal comments motivating the various notions introduced, assist the reader's comprehension. Section~\ref{settingSect} contains all the necessary definitions and concepts required to state and to explain the results of the article. It ends with the statement of Geiringer-like theorem aimed at applications to decision making in the environments with randomness and incomplete information where no immediate rewards are available. This is the central aim of the paper. A reader who is only after a calculus level understanding with the aim of developing applications within an appropriate area of software engineering may be satisfied reading section~\ref{settingSect} and finishing their study at this point. Section~\ref{GeneralThmSect} is devoted to establishing and deriving the main results of the article in a mathematically rigorous fashion. Clearly this is fundamentally important for understanding where these results come from and how one may modify them as needed. We strongly encourage all the interested readers to attempt understanding the entire section~\ref{GeneralThmSect}. Subsection~\ref{stageSubsect} does require familiarity with elementary group theory. A number of textbooks on this subject are available (see, for instance, \cite{DummittFoote}) but all of them contain way more material than necessary to understand our work. To get the minimal necessary understanding, the reader is invited to look at the previous papers on finite population Geiringer theorems of the first two authors: \cite{MitavRowGeirMain} and \cite{MitavRowGeirGenProgr}. Finally, section~\ref{furthStrengthSect} is included only for the sake of strengthening the general finite-population Geiringer theorem to emphasize its validity for nonhomogenious time Markov chains, namely theorem~\ref{GeiringerExtThm}. Example~\ref{intuitiveInterpEx} explains why this is of interest for the algorithm development. The material in section~\ref{furthStrengthSect} is entirely independent of the rest of the paper. One could read it either at the beginning or at the end. The authors suspect this theory is known in modern math, but the literature emphasizing theorems~\ref{MarkovNonMarkovSubtleExtCor} and \ref{MarkovNonMarkovSubtleExtCorExt} is virtually impossible to locate. Moreover, mathematics behind these theorems is classical, general, simple and elegant. While section~\ref{furthStrengthSect} is probably not of any interest to software engineers (theorem~\ref{GeiringerExtThm} may be thought to strengthen the justification of the main ideas), more mathematically inclined audience will find it enjoyable and easy to read.
\section{Equivalence/Similarity Relation on the States}\label{equivSimSect}
Let $S$ denote the set of states (enormous but finite in this framework). Formally each state $\vec{s} \in S$ is an ordered pair $(s, \vec{\alpha})$ where $\vec{\alpha}$ is the set of actions an agent can possibly take when in the state $\vec{s}$. Let $\sim$ be an equivalence relation on $S$. Without loss of generality we will denote every equivalence class by an integer $1, \, 2, \ldots, i, \ldots, \in \mathbb{N}$ so that each element of $S$ as an ordered pair $(i, \, a)$ where $i \in \mathbb{N}$ and $a \in A$ with $A$ being some finite alphabet. With this notation $(i, \, a) \sim (j, \, b)$ iff $i=j$. Intuitively, $S$ is the set of states and $\sim$ is the similarity relation on the states. For example in a card game if the 2 states corresponding to the same player have cards of roughly equivalent value (for that specific game) and their opponent's cards are unknown (and there might be some more hidden and random effects) then the 2 states will be considered equivalent under $\sim$. We will also require that for two equivalent states $\vec{s_1} = \{s_1, \vec{\alpha}_1\}$ and $\vec{s_2} = \{s_2, \vec{\alpha}_2\}$ under $\sim$ there are bijections $f_1: \vec{\alpha}_1 \rightarrow \vec{\alpha}_2$ and $f_2: \vec{\alpha}_2 \rightarrow \vec{\alpha}_1$. For the time being, these bijections should be obvious from the representation of the environment (and actions) and reflect the similarity between these actions.
\begin{rem}\label{formalreqCommsRem}
In theory we want functions $f_1$ and $f_2$ to be bijections and inverses of one another for the theoretical model to be perfectly rigorous, but in practice there should probably be no strict requirement on that. In fact, we believe that in practice one may even want to relax the assumption on $\sim$ to be an equivalence relation.
\end{rem}
As described in sections~\ref{IntroSect} and \ref{OverViewSect}, the most challenging question when applying an MCT type of an algorithm to deal with randomness and incomplete information or simply with a large branching factor of the game tree is to evaluate the actions under consideration making the most out of the sample of independent rollouts. Quite surprisingly, very powerful programs have already been developed and tested in practice against human players (see \cite{KocsisL}), however the action-evaluation algorithms used in these software are purely heuristic and no theoretical foundation is presented to explain their success. In fact, most of these methods use some kind of a voting mechanism to deal with rather weak classifiers. In the next section we will set up the stage to state the main result of this paper which motivates new algorithms for evaluating actions (or moves) at the chance nodes and hopefully will provide some understanding for the success of the already existing techniques in the future research.
\section{Mathematical Framework, Notion of Crossover/Recombination and Statement of the Finite Population Geiringer Theorem for Action Evaluation.}\label{settingSect}
\begin{defn}\label{treeRootedByChanceNode}
Suppose we are given a chance node $\vec{s} = (s, \vec{\alpha})$ and a sequence $\{\alpha_i\}_{i=1}^b$ of actions in $\vec{\alpha}$ (it is possible that $\alpha_i = \alpha_j$ for $i \neq j$). We may then call $\vec{s}$ a \emph{root state}, or a \emph{state in question}, the sequence $\{\alpha_i\}_{i=1}^b$, the \emph{sequence of moves (actions) under evaluation} and the set of moves $\mathcal{A} = \{\alpha \, | \, \alpha = \alpha_i$ for some $i$ with $1 \leq i \leq b\}$, the set of actions (or moves) under evaluation.
\end{defn}
\begin{defn}\label{RolloutDefn}
A \emph{rollout} with respect to the state in question $\vec{s} = (s, \vec{\alpha})$ and an action $\alpha \in \vec{\alpha}$ is a sequence of states following the action $\alpha$ and ending with a terminal label $f \in \Sigma$ where $\Sigma$ is an arbitrary set of labels\footnote{Intuitively, each terminal label in the set $\Sigma$ represents a terminal state that we can assign a numerical value to via a function $\phi: \, \Sigma \rightarrow \mathbb{Q}$. The reason we introduce the set $\Sigma$ of formal labels as opposed to requiring that each terminal label is a rational number straight away, is to avoid confusion in the upcoming definitions}, which looks as $\{(\alpha, \, s_1, \, s_2, \ldots, s_{t-1}, \, f)\}$. For technical reasons which will become obvious later we will also require that $s_i \neq s_j$ for $i \neq j$ (it is possible and common to have $s_i \sim s_j$ though). We will say that the total number of states in a rollout (which is $k-1$ in the notation of this definition) is the \emph{height} of the rollout.
\end{defn}
\begin{rem}\label{rolloutDefRem}
Notice that in definition~\ref{RolloutDefn} we included only the initial move $\alpha$ made at the state in question (see definition~\ref{treeRootedByChanceNode}) which is the move under evaluation (see definition~\ref{treeRootedByChanceNode}). The moves between the intermediate states are chosen randomly and are not evaluated so that there is no reason to consider them.
\end{rem}
\begin{rem}\label{crossoverConvRepresRem}
In subsection~\ref{equivSimSect} we have introduced a convenient notation for states to emphasize their respective equivalence classes. With such notation a typical rollout would appear as a sequence $\{(\alpha, \, (i_1, \, a_1), \, (i_2, \, a_2), \ldots, (i_{t-1}, a_{t-1}), \, f)\}$ with $i_j \in \mathbb{N}$ while $a_i \in A$. According to the requirement in definition~\ref{RolloutDefn}, $i_j = i_k$ for $j \neq k \, \Longrightarrow a_k \neq a_j$.
\end{rem}
A single rollout provides rather little information about an action particularly due to the combinatorial explosion in the branching factor of possible moves of the player and the opponents. Normally a large, yet comparable with total resource limitations, number of rollouts is thrown to evaluate the actions at various positions. The challenging question which the current work addresses is how one can take full advantage of the parallel sequence of rollouts. Since the main idea is motivated by Geiringer theorem which is originated from population genetics (\cite{GeirOrigion}) and later has also been involved in evolutionary computation theory (\cite{PoliGeir}, \cite{MitavRowGeirMain} and \cite{MitavRowGeirGenProgr}) we shall exploit the terminology of the evolutionary computation community here.
\begin{defn}\label{popOfRolloutsDefn}
Given a state in question $\vec{s} = (s, \vec{\alpha})$ and a sequence $\{\alpha_i\}_{i=1}^b$ of moves under evaluation (in the sense of definition~\ref{treeRootedByChanceNode}) then a \emph{population} $P$ with respect to the state $\vec{s} = (s, \vec{\alpha})$ and the sequence $\{\alpha_i\}_{i=1}^b$ is a sequence of rollouts $P = \{r_i^{l(i)}\}_{i=1}^b$ where $r_i = \{(\alpha_i, \, s_1^i, \, s_2^i, \ldots, s^i_{l(i)-1}, \, f_i)\}$. Just as in definition~\ref{RolloutDefn} we will assume that $s_k^i \neq s_q^j$ whenever $i \neq j$ (which, in accordance with definition~\ref{RolloutDefn}, is as strong as requiring that $s_k^i \neq s_q^j$ whenever $i \neq j$ or $k \neq q$)\footnote{The last assumption that all the states in a population are formally distinct (although they may be equivalent) will be convenient later to extend the crossover operators from pairs to the entire populations. This assumption does make sense from the intuitive point of view as well since the exact state in most games involving randomness or incomplete information is simply unknown.} Moreover, we also assume that the terminal labels $f_i$ are also all distinct within the same population, i.e. for $i \neq j$ the terminal labels $f_i \neq f_j$\footnote{This assumption does not reduce any generality since one can choose an arbitrary (possibly a many to one) assignment function $\phi: \Sigma \rightarrow \mathbb{Q}$, yet the complexity of the statements of our main theorems will be mildly alleviated.} In a very special case when $s_j^i \sim s_k^q \Longrightarrow j=k$ we will say that the population $P$ is \emph{homologous}. Loosely speaking, a homologous population is one where equivalent states can not appear at different ``heights".
\end{defn}
\begin{rem}\label{popOfRolloutsRem}
Each rollout $r_i^{l(i)}$ in definition~\ref{popOfRolloutsDefn} is started with the corresponding move $\alpha_i$ of the sequence of moves under evaluation (see definition~\ref{treeRootedByChanceNode}). It is clear that if one were to permute the rollouts without changing the actual sequences of states the corresponding populations should provide identical values for the corresponding actions under evaluation. In fact, most authors in evolutionary computation theory (see \cite{VoseM}, for instance) do assume that such populations are equivalent and deal with the corresponding equivalence classes of multisets corresponding to the individuals (these are sequences of rollouts). Nonetheless, when dealing with finite-population Geiringer-like theorems it is convenient for technical reasons which will become clear when the proof is presented (see also \cite{MitavRowGeirMain} and \cite{MitavRowGeirGenProgr}) to assume the \emph{ordered multiset model} i.e. the populations are considered formally \emph{distinct} when the individuals are permuted. Incidentally, ordered multiset models are useful for other types of theoretical analysis in \cite{ShmittL1} and \cite{ShmittL2}.
\end{rem}
\begin{ex}\label{PopRolloutEx}
A typical population with the convention as in remark~\ref{popOfRolloutsRem} might look as below.
$$\begin{array}{clcrrrrrrrrrrrrrrrrrrrrrr}
\alpha & \mapsto 1a & \mapsto 5a & \mapsto 6a & \mapsto 3d & \mapsto 7a & \mapsto f_1\\
\beta & \mapsto 2a & \mapsto 1b & \mapsto 3c & \mapsto 6d & \mapsto f_2 \\
\gamma & \mapsto 4a & \mapsto 6b & \mapsto 5b & \mapsto f_3\\
\alpha & \mapsto 1c & \mapsto 4b & \mapsto 2b & \mapsto 7b & \mapsto 5c & \mapsto f_4 \\
\xi & \mapsto 3a & \mapsto 2c & \mapsto 4c & \mapsto f_5\\
\xi & \mapsto 2d & \mapsto f_6\\
\pi & \mapsto 3b & \mapsto 1d & \mapsto 2e & \mapsto 6c & \mapsto f_7 \\
\end{array}$$
The height of the first rollout in the population pictured above would then be $5$ since it contains $5$ states. The reader can easily see that the heights of the rollouts in this population read from top to bottom are $5$, $4$, $3$, $5$, $3$, $1$ and $4$ respectively. Clearly, the total number of states within the population is the sum of the heights of all the rollouts in the population. In fact, this very simple observation is rather valuable when establishing the main result of the current article as will become clear in subsection~\ref{specificThmDerivationSubsect} of section~\ref{GeneralThmSect}.
\end{ex}
%\begin{figure}[htp]
%\centerline{\includegraphics[height=10cm]{popEx.eps}}
%\caption{An example of a population consisting of seven rollouts. Equivalence classes of states are denoted by distinct numbers so that the letters written next to these numbers distinguish the individual states as in remark~\ref{crossoverConvRepresRem}. Distinct actions under evaluation (see definition~\ref{treeRootedByChanceNode}) are denoted by different letters of Greek alphabet.}
%\label{PopOfRollsFig}
%\end{figure}
The main idea is that the random actions taken at the equivalent states should be interchangeable since they are chosen somehow at random during the simulation stage of the MCT algorithm. In the language of evolutionary computing, such a swap of moves is called a crossover. Due to randomness or incomplete information (together with the equivalence relation which can be defined using the expert knowledge of a specific game being analyzed) in order to obtain the most out of a sample (population in our language) of the parallel rollouts it is desirable to explore all possible populations obtained by making various swaps of the corresponding rollouts at the equivalent positions. Computationally this task seems expensive if one were to run the type of genetic programming described precisely below, yet, it turns out that we can predict exactly what the limiting outcome of this ``mixing procedure" would be.\footnote{In this paper we will need to ``inflate" the population first and then take the limit of a sequence of these limiting procedures as the inflation factor increases. All of this will be rigorously presented and discussed in subsection~\ref{GeirThmStSubsect} and in section~\ref{GeneralThmSect}.} We now continue with the rigorous definitions of crossover.

Representation of rollouts suggested in remark~\ref{crossoverConvRepresRem} is convenient to define crossover operators for two given rollouts. We will introduce two crossover operations below.
\begin{defn}\label{rolloutPartCrossDefn}
Given two rollouts $r_1 = (\alpha_1, \, (i_1, \, a_1), \, (i_2, \, a_2), \ldots, (i_{t(1)-1}, a_{t(1)-1}), \, f)$ and $r_2 = (\alpha_2, \, (j_1, \, b_1), \, (j_2, \, b_2), \ldots, (j_{t(2)-1}, b_{t(2)-1}), \, g)$ of lengths $t(1)$ and $t(2)$ respectively that share no state in common (i.e., as in definition~\ref{RolloutDefn}, ) there are two (non-homologous) crossover (or recombination) operators we introduce here. For an equivalence class label $m \in \mathbb{N}$ and letters $c, \, d \in A$ define the \emph{one-point non-homologous crossover} transformation $\chi_{m, \, c, \, d}(r_1, \, r_2) = (t_1, \, t_2)$ where $t_1 = (\alpha_1, \, (i_1, \, a_1), \, (i_2, \, a_2), \ldots, (i_{k-1}, \, a_{k-1}), \, (j_q, \, b_q), \, (j_{q+1}, \, b_{q+1}), \ldots, (j_{t(2)-1}, b_{t(2)-1}), \, g)$ and $t_2 = (\alpha_2, \, (j_1, \, b_1), \, (j_2, \, b_2), \ldots, (j_{q-1}, \, b_{q-1}), \, (i_k, \, a_k), \, (i_{k+1}, \, a_{k+1}), \ldots, (i_{t(1)-1}, a_{t(1)-1}), \, f)$ if [$i_k = j_q = m$ and either $(a_k = c$ and $b_q = d)$ or $(a_k = d$ and $b_q = c)$] and $(t_1, \, t_2) = (r_1, \, r_2)$ otherwise.

Likewise, we introduce a \emph{single position swap crossover} $\nu_{m, \, c, \, d}(r_1, \, r_2) = (v_1, \, v_2)$ where $v_1 = (\alpha_1, \, (i_1, \, a_1), \, (i_2, \, a_2), \ldots, (i_{k-1}, \, a_{k-1}), \, (j_q, \, b_q), \, (i_{k+1}, \, a_{k+1}), \ldots, (i_{t(1)-1}, a_{t(1)-1}), \, f)$ while $v_2 = (\alpha_2, \, (j_1, \, b_1), \, (j_2, \, b_2), \ldots, (j_{q-1}, \, b_{q-1}), \, (i_k, \, a_k), \, (j_{q+1}, \, b_{q+1}), \ldots, (j_{t(2)-1}, b_{t(2)-1}), \, g)$ if [$i_k = j_q = m$ and either $(a_k = c$ and $b_q = d)$ or $(a_k = d$ and $b_q = c)$] and $(v_1, \, v_2) = (r_1, \, r_2)$ otherwise. In addition, a singe swap crossover is defined not only on the pairs of rollouts but also on a single rollout swapping equivalent states in the analogous manner: If $$r = (\alpha, \, (i_1, \, a_1), \, (i_2, \, a_2), \ldots, (i_{j-1}, \, a_{j-1}), \, (i_j, \, a_j), \, (i_{j+1}, \, a_{j+1}), \ldots $$$$ \ldots, (i_{k-1}, \, a_{k-1}), \, (i_k, \, a_k), \, (i_{k+1}, \, a_{k+1}), \ldots, (i_{t(1)-1}, a_{t(1)-1}), \, f)$$ and [$i_j = i_k$ and either $(a_j = c$ and $a_k = d)$ or $(a_j = d$ and $a_k = c)$] then $$\nu_{m, \, c, \, d}(r) = (\alpha, \, (i_1, \, a_1), \, (i_2, \, a_2), \ldots, (i_{j-1}, \, a_{j-1}), \, (i_j, \, a_k), \, (i_{j+1}, \, a_{j+1}), \ldots $$$$ \ldots, (i_{k-1}, \, a_{k-1}), \, (i_k, \, a_j), \, (i_{k+1}, \, a_{k+1}), \ldots, (i_{t(1)-1}, a_{t(1)-1}), \, f)$$ and, of course, $\nu_{m, \, c, \, d}(r)$ fixes $r$ (i.e. $\nu_{m, \, c, \, d}(r) = r$) otherwise.
%Suppose $i_k = j_q$ for some $k$ and $q$ with $1 \leq k < t(1)$ and $1 \leq q < t(2)$ (recall from section~\ref{equivSimSect} that this means the corresponding states at positions $k$ and $q$ are equivalent). We then define a \emph{one-point non-homologous crossover} of $r_1$ and $r_2$ to be the pair $\chi(r_1, \, r_2) = (t_1, \, t_2)$ where $t_1 = (\alpha_1, \, (i_1, \, a_1), \, (i_2, \, a_2), \ldots, (i_{k-1}, \, a_{k-1}), \, (j_q, \, b_q), \, (j_{q+1}, \, b_{q+1}), \ldots, (j_{t(2)-1}, b_{t(2)-1}), \, g)$ and $t_2 = (\alpha_2, \, (j_1, \, b_1), \, (j_2, \, b_2), \ldots, (j_{q-1}, \, b_{q-1}), \, (i_k, \, a_k), \, (i_{k+1}, \, a_{k+1}), \ldots, (i_{t(1)-1}, a_{t(1)-1}), \, f)$. Likewise, we introduce a \emph{single position swap crossover} $\nu(r_1, \, r_2) = (v_1, \, v_2)$ where $v_1 = (\alpha_1, \, (i_1, \, a_1), \, (i_2, \, a_2), \ldots, (i_{k-1}, \, a_{k-1}), \, (j_q, \, b_q), \, (i_{k+1}, \, a_{k+1}), \ldots, (i_{t(1)-1}, a_{t(1)-1}), \, f)$ while $v_2 = (\alpha_2, \, (j_1, \, b_1), \, (j_2, \, b_2), \ldots, (j_{q-1}, \, b_{q-1}), \, (i_k, \, a_k), \, (j_{q+1}, \, b_{q+1}), \ldots, (j_{t(2)-1}, b_{t(2)-1}), \, g)$. Moreover, The single position swap crossover also applies to a single rollout $r = (\alpha, \, (i_1, \, a_1), \, (i_2, \, a_2), \ldots, (i_{t(1)-1}, a_{t(1)-1}), \, f)$ defined as $\nu_{\{(i_k, \, a_k), (i_j, a_j)\}}(r) = r$ unless
\end{defn}
\begin{rem}\label{distConvRem}
Notice that definition~\ref{rolloutPartCrossDefn} makes sense thanks to the assumption that no rollout contains an identical pair of states in definition~\ref{RolloutDefn}.
\end{rem}
%One more technical operation that will be allowed is swapping the order in which the rollouts appear (obviously this makes no practical difference but will be convenient for theoretical purposes in section~\ref{theoryFoundSect})
\begin{rem}\label{motivationRem}
Intuitively, performing one point crossover means that the corresponding player might have changed their strategy in a similar situation due to randomness and a single swap crossover corresponds to the player not knowing the exact state they are in due to incomplete information, for instance.
\end{rem}
Just as in case of defining crossover operators for pairs of rollouts, thanks to the assumption that all the states in a population of rollouts are formally distinct (see definition~\ref{popOfRolloutsDefn}), it is easy to extend definition~\ref{rolloutPartCrossDefn} to the entire populations of rollouts.
%\begin{rem}\label{subtleFormalismRem}
%To suit the framework of the extended finite-population Geiringer theorem we need to assume that all the states appearing in a population of parallel rollouts are formally distinct. There may be as many equivalent states as one feels appropriate, but these will always have distinct second coordinate.  The reason this assumption is convenient is that it allows us to extend the domain of crossover operators in definition~\ref{rolloutPartCrossDefn} to the full set of ordered pairs of rollouts by defining a separate crossover operator for every pair of states $(i, \, x)$ and $(i, \, y)$ for $i \in \mathbb{N}$ and $x \neq y \in A$ to the full set of all possible ordered pairs of rollouts which may appear in any allowable population of rollouts by  defining $\chi_{\{(i, \, x), \, (i, \, y)\}}(r_1, r_2) = (r_1, r_2)$ unless the states $(i, \, x)$ and $(i, \, y)$ appear in the rollouts $r_1$ and $r_2$ in which case everything is as in definition~\ref{rolloutPartCrossDefn} keeping in mind that the states $(i, \, x)$ and $(i, \, y)$ appear in unique positions in the pair of rollouts thanks to the assumption this remark is devoted to. Of course the same applies to the single position swap crossover.
%\end{rem}
In view of remark~\ref{motivationRem}, to get the most informative picture out of the sequence of parallel rollouts one would want to run the genetic programming routine without selection and mutation and using only the crossover operators specified above for as long as possible and then, in order to evaluate a certain move $\alpha$, collect the weighted average of the terminal values (i. e. the values assigned to the terminal labels via some rational-valued assignment function) of all the rollouts starting with the move $\alpha$ which ever occurred in the process. We now describe precisely what the process is and give an example.
\begin{defn}\label{recombActOnPopsDef}
Given a population $P$ and a transformation of the form $\chi_{i, \, x, \, y}$, there exists at most one pair of distinct rollouts in the population $P$, namely the pair of rollouts $r_1$ and $r_2$ such that the state $(i, \, x)$ appears in $r_1$ and the state $(i, \, y)$ appears in $r_2$. If such a pair exists, then we define the recombination transformation $\chi_{i, \, x, \, y}(P) = P'$ where $P'$ is the population obtained from $P$ by replacing the pair of rollouts $(r_1, \, r_2)$ with the pair $\chi_{i, \, x, \, y}(r_1, \, r_2)$ as in definition~\ref{rolloutPartCrossDefn}. In any other case we do not make any change, i.e. $\chi_{i, \, x, \, y}(P) = P$. The transformation $\nu_{i, \, x, \, y}(P)$ is defined in an entirely analogous manner with one more amendment: if the states $(i, \, x)$ and $(i, \, y)$ appear within the same individual (rollout), call it $$r = (\alpha, \, (j_1, \, a_1), \, (j_2, \, a_2), \ldots, (i, \, x), \ldots, \, (i, \, y), \ldots, (i_{t(1)-1}, a_{t(1)-1}), \, f),$$ and the state $(i, \, x)$ precedes the state $(i, \, y)$, then these states are interchanged obtaining the new rollout $$r' = (\alpha, \, (j_1, \, a_1), \, (j_2, \, a_2), \ldots, (i, \, y), \ldots, \, (i, \, x), \ldots, (i_{t(1)-1}, a_{t(1)-1}), \, f).$$ Of course, it could be that the state $(i, \, y)$ precedes the state $(i, \, x)$ instead, in which case the definition would be analogous: if $$r = (\alpha, \, (j_1, \, a_1), \, (j_2, \, a_2), \ldots, (i, \, y), \ldots, \, (i, \, x), \ldots, (i_{t(1)-1}, a_{t(1)-1}), \, f)$$ then replace the rollout $r$ with the rollout $$r'=(\alpha, \, (j_1, \, a_1), \, (j_2, \, a_2), \ldots, (i, \, x), \ldots, \, (i, \, y), \ldots, (i_{t(1)-1}, a_{t(1)-1}), \, f).$$
\end{defn}
\begin{rem}\label{BijectRem}
It is very important for the main theorem of our paper that each of the crossover transformations $\chi_{i, \, x, \, y}$ and $\nu_{i, \, x, \, y}$ is a bijection on their common domain, that is the set of all populations of rollouts at the specified chance node. As a matter of fact, the reader can easily verify by direct computation from definitions~\ref{recombActOnPopsDef} and \ref{rolloutPartCrossDefn} that each of the transformations $\chi_{i, \, x, \, y}$ and $\nu_{i, \, x, \, y}$ is an involution on its domain, i.e. $\forall \, i, \, x, \, y$ we have $\chi_{i, \, x, \, y}^2 = \nu_{i, \, x, \, y}^2 = \mathbf{1}$ where $\mathbf{1}$ is the identity transformation.
\end{rem}
Examples below illustrate the important extension of recombination operators to arbitrary populations pictorially.
\begin{ex}\label{popCrossEx1}
Suppose we were to apply the recombination (crossover) operator $\chi_{1, \, c, \, d}$ to the population of seven rollouts in example~\ref{PopRolloutEx}.
Once the unique location of states $(1, \, c)$ and $(1, \, d)$ in the population has been 
identified (the first state in the forth rollout and the second state in the seventh rollout), applying 
the crossover operator $\chi_{1, \, c, \, d}$ yields the population pictured below:
%\begin{figure}[htp]
%\centerline{\includegraphics[height=10cm]{popExSq1.eps}}
%\caption{The unique states $(1, \, c)$ and $(1, \, d)$ in the population pictured in figure~\ref{PopOfRollsFig} are enclosed in dashed squares.}
%\label{PopOfRollsFigSq1}
%\end{figure}
%\begin{figure}[htp]
%\centerline{\includegraphics[height=10cm]{popExAfterCross1.eps}}
%\caption{The subrollouts rooted at the states $(1, \, c)$ and $(1, \, d)$ in the population pictured in figure~\ref{PopOfRollsFigSq1} are pruned and then swapped.}
%\label{PopOfRollsAfterCross1}
%\end{figure}
$$\begin{array}{clcrrrrrrrrrrrrrrrrrrrrrr}
\alpha & \mapsto 1a & \mapsto 5a & \mapsto 6a & \mapsto 3d & \mapsto 7a & \mapsto f_1\\
\beta & \mapsto 2a & \mapsto 1b & \mapsto 3c & \mapsto 6d & \mapsto f_2 \\
\gamma & \mapsto 4a & \mapsto 6b & \mapsto 5b & \mapsto f_3\\
\alpha & \mapsto 1d & \mapsto 2e & \mapsto 6c & \mapsto f_7 \\
\xi & \mapsto 3a & \mapsto 2c & \mapsto 4c & \mapsto f_5\\
\xi & \mapsto 2d & \mapsto f_6\\
\pi & \mapsto 3b & \mapsto 1c & \mapsto 4b & \mapsto 2b & \mapsto 7b & \mapsto 5c & \mapsto f_4\\
\end{array}$$
On the other hand, applying the crossover transformation $\nu_{1, \, c, \, d}$ to the population in example~\ref{PopRolloutEx} results in the population below:
%\begin{figure}[htp]
%\centerline{\includegraphics[height=10cm]{popExAfterCross2.eps}}
%\caption{The uniquely positioned labels $(1, \, c)$ and $(1, \, d)$ are enclosed within the dashed squares in figure~\ref{PopOfRollsFigSq1} are swapped.}
%\label{PopOfRollsAfterCross2}
%\end{figure}
$$\begin{array}{clcrrrrrrrrrrrrrrrrrrrrrr}
\alpha & \mapsto 1a & \mapsto 5a & \mapsto 6a & \mapsto 3d & \mapsto 7a & \mapsto f_1\\
\beta & \mapsto 2a & \mapsto 1b & \mapsto 3c & \mapsto 6d & \mapsto f_2 \\
\gamma & \mapsto 4a & \mapsto 6b & \mapsto 5b & \mapsto f_3\\
\alpha & \mapsto 1d & \mapsto 4b & \mapsto 2b & \mapsto 7b & \mapsto 5c & \mapsto f_4 \\
\xi & \mapsto 3a & \mapsto 2c & \mapsto 4c & \mapsto f_5\\
\xi & \mapsto 2d & \mapsto f_6\\
\pi & \mapsto 3b & \mapsto 1c & \mapsto 2e & \mapsto 6c & \mapsto f_7 \\
\end{array}.$$
\end{ex}
\begin{ex}\label{popCrossEx2}
Consider now the population $Q$ pictured below: 
$$\begin{array}{clcrrrrrrrrrrrrrrrrrrrrrr}
\alpha & \mapsto 1b & \mapsto 3c & \mapsto 6d & \mapsto f_2\\
\beta & \mapsto 2b & \mapsto 7b & \mapsto 5c & \mapsto f_4 \\
\gamma & \mapsto 4a & \mapsto 6b & \mapsto 5a & \mapsto 6a & \mapsto 3d & \mapsto 7a & \mapsto f_1\\
\alpha & \mapsto 1d & \mapsto 2c & \mapsto 4c & \mapsto f_5 \\
\xi & \mapsto 3a & \mapsto 2e & \mapsto 6c & \mapsto f_7\\
\xi & \mapsto 2d & \mapsto f_6\\
\pi & \mapsto 3b & \mapsto 1c & \mapsto 4b & \mapsto 2a & \mapsto 1a & \mapsto 5b & \mapsto f_3 \\
\end{array}.$$
Suppose we apply the transformations $\chi_{6, \, a, \, b}$ and $\nu_{6, \, a, \, b}$ to the population $Q$.
%\begin{figure}[htp]
%\centerline{\includegraphics[height=10cm]{popEx2.eps}}
%\caption{A population of rollouts $Q$.}
%\label{PopOfRollsFig2}
%\end{figure}
The states $(6, \, a)$ and $(6, \, b)$ both appear in the third rollout in the population $Q$.
%\begin{figure}[htp]
%\centerline{\includegraphics[height=10cm]{popExSq2.eps}}
%\caption{The uniquely positioned labels $(6, \, a)$ and $(6, \, b)$ are enclosed within the dashed squares.}
%\label{PopOfRollsFig3}
%\end{figure}
Since these states appear within the same rollout, according to definition~\ref{recombActOnPopsDef}, the crossover transformation $\chi_{6, \, a, \, b}$ fixes the population $Q$ (i.e. $\chi_{6, \, a, \, b}(Q)=Q$). On the other hand, the population $\nu_{6, \, a, \, b}(Q)$ is pictured below:
%\begin{figure}[htp]
%\centerline{\includegraphics[height=10cm]{popExAfterCross3.eps}}
%\caption{The uniquely positioned labels $(6, \, a)$ and $(6, \, b)$ which are enclosed within the dashed squares on figure~\ref{PopOfRollsFig3} are interchanged to obtain the population $\nu_{6, \, a, \, b}(Q)$ pictured above.}
%\label{PopOfRollsAfterCross3}
%\end{figure}
$$\begin{array}{clcrrrrrrrrrrrrrrrrrrrrrr}
\alpha & \mapsto 1b & \mapsto 3c & \mapsto 6d & \mapsto f_2\\
\beta & \mapsto 2b & \mapsto 7b & \mapsto 5c & \mapsto f_4 \\
\gamma & \mapsto 4a & \mapsto 6a & \mapsto 5a & \mapsto 6b & \mapsto 3d & \mapsto 7a & \mapsto f_1\\
\alpha & \mapsto 1d & \mapsto 2c & \mapsto 4c & \mapsto f_5 \\
\xi & \mapsto 3a & \mapsto 2e & \mapsto 6c & \mapsto f_7\\
\xi & \mapsto 2d & \mapsto f_6\\
\pi & \mapsto 3b & \mapsto 1c & \mapsto 4b & \mapsto 2a & \mapsto 1a & \mapsto 5b & \mapsto f_3 \\
\end{array}.$$
\end{ex}
\begin{defn}\label{RecombStagePopTransDefn}
Let $\mathbf{n} = \{1, \, 2, \ldots, n\}$ denote the set of first $n$ natural numbers. Consider any probability distribution $\mu$ on the set of all finite sequences of crossover transformations $$\mathcal{F} = \left(\bigcup_{n=1}^{\infty}\left(\{\chi_{i, \, x, \, y} \, | \, x, \, y \in A \text{ and } i \in \mathbb{N}\} \cup \{\nu_{i, \, x, \, y} \, | \, x, \, y \in A \text{ and }i \in \mathbb{N}\}\right)^n \right) \cup \{\mathbf{1}\}$$ which assigns a positive probability to the singleton sequences\footnote{This technical assumption may be altered in various manner as long as the induced Markov chain remains irreducible.} and to the \emph{identity element} $\mathbf{1}$. (i.e. to every element of the subset $$\mathcal{S} = \{\mathbf{1}\} \cup \left(\{\chi_{i, \, x, \, y} \, | \, x, \, y \in A \text{ and } i \in \mathbb{N}\}\cup\{\nu_{i, \, x, \, y} \, | \, x, \, y \in A \text{ and } i \in \mathbb{N}\} \right)^1.$$ Given a sequence of transformations $\vec{\Theta} = \{\Theta_{i(j), \, x(j), \, y(j)}\}_{j=1}^n$ where each $\Theta$ is either $\chi$ or $\nu$ (i.e. $\forall \, j$ either $\Theta_{i(j), \, x(j), \, y(j)} = \chi_{i(j), \, x(j), \, y(j)}$ or $\Theta_{i(j), \, x(j), \, y(j)} = \nu_{i(j), \, x(j), \, y(j)}$), consider the transformation $$\widetilde{\Theta} = \Theta_{i(n), \, x(n), \, y(n)} \circ \Theta_{i(n-1), \, x(n-1), \, y(n-1)} \circ \ldots \circ \Theta_{i(2), \, x(2), \, y(2)} \circ \Theta_{i(1), \, x(1), \, y(1)}$$ on the set of all populations starting at the specified chance node obtained by composing all the transformations in the sequence $\vec{\Theta}$. The identity element $\mathbf{1}$ stands for the identity map on the set of all possible populations of rollouts. Now define the Markov transition Matrix $M_{\mu}$ on the set of all populations of rollouts (see definition~\ref{popOfRolloutsDefn} and remark~\ref{crossoverConvRepresRem}) as follows: given populations $X$ and $Y$ of the same size $k$, the probability  of obtaining the population $Y$ from the population $X$ after performing a single crossover stage, $p_{X \rightarrow Y} = \mu(\mathcal{S}_{X \rightarrow Y})$ where $$\mathcal{S}_{X \rightarrow Y}=\{\Gamma \, | \, \Gamma \in \mathcal{F} \text{ and } T(\Gamma)(X) = Y\}$$ where $$T(\Gamma) = \begin{cases}
\widetilde{\Theta} & \text{ if } \Gamma = \vec{\Theta} \\
\text{The identity map} \text{ if } \Gamma = \mathbf{1}.
\end{cases}$$
\end{defn}
Example~\ref{seqCompExample} below illustrates the first part of definition~\ref{RecombStagePopTransDefn}.
\begin{ex}\label{seqCompExample}
Consider the sequence of five recombination transformations $$\vec{\Theta} = (\chi_{1, \, c, \, d}, \, \chi_{2, \, c, \, e}, \, \chi_{5, \, a, \, b}, \chi_{1, \, a, \, b}, \, \chi_{2, \, a, \, b}).$$ According to definition~\ref{RecombStagePopTransDefn} the sequence $\vec{\Theta}$ gives rise to the composed recombination transformation $$\widetilde{\Theta} = \chi_{2, \, a, \, b} \circ \chi_{1, \, a, \, b} \circ \chi_{5, \, a, \, b} \circ \chi_{2, \, c, \, e} \circ \chi_{1, \, c, \, d}.$$ The reader may verify as a small exercise that $\widetilde{\Theta}(P) = Q$ where $P$ is the population displayed on figure~\ref{PopOfRollsFig} while the population $Q$ is the one appearing in figure~\ref{PopOfRollsFig2}. If one were to append the recombination transformation $\nu_{6, \, a, \, b}$ to the sequence of rollouts $\vec{\Theta}$ obtaining the sequence $$\overrightarrow{\Theta_1} = (\chi_{1, \, c, \, d}, \, \chi_{2, \, c, \, e}, \, \chi_{5, \, a, \, b}, \chi_{1, \, a, \, b}, \, \chi_{2, \, a, \, b}, \, \nu_{6, \, a, \, b})$$ then, by associativity of composition, we have $\widetilde{\Theta_1} = \nu_{6, \, a, \, b} \circ \widetilde{\Theta}$ so that $\widetilde{\Theta_1}(P) = \nu_{6, \, a, \, b}(\widetilde{\Theta}(P)) = \nu_{6, \, a, \, b}(Q)$ where $Q$, as above, is the population displayed on figure~\ref{PopOfRollsFig2} so that, according to example~\ref{popCrossEx2}, the population $\overrightarrow{\Theta_1}(P)$ is the one appearing in figure~\ref{PopOfRollsAfterCross3}.
\end{ex}
\begin{rem}\label{seqVsTransfRem}
Evidently the map $T: \mathcal{F} \rightarrow P^P$ introduced at the end of definition~\ref{RecombStagePopTransDefn} can be regarded as a random variable on the set $\mathcal{F}$ described at the beginning of definition~\ref{RecombStagePopTransDefn} where $P$ denotes the set of all populations of rollouts containing $k$ individuals so that $P^P$ is the set of all endomorphisms (functions with the same domain and codomain) on $P$ and the probability measure $\mu_T$ on $P^P$ is the ``pushforward" measure induced by $T$, i.e. $\mu_T(S) = \mu(T^{-1}(S))$.\footnote{The sigma algebra on $P^P$ is the one generated by $T$ with respect to the sigma-algebra that is originally chosen on $\mathcal{F}$, however in practical applications the sets involved are finite and so all the sigma-algebras can be safely assumed to be power sets.} To alleviate the complexity of verbal (or written) presentation we will usually abuse the language and use the set $\mathcal{F}$ in place of $P^P$ so that a transformation $F \in P^P$ is identified with the entire set $T^{-1}(F) \in \mathcal{F}$. For example, $$\text{if we write } \mu(\{F \, | \, F \in \mathcal{F} \text{ and } F(X) = Y\}) \text{we mean } \mu(\{\Gamma \, | \, \Gamma \in \mathcal{F} \text{ and } T(\Gamma)(X) = Y\}).$$ It may be worth pointing out that the set $T^{-1}$ is not necessarily a singleton, i.e. the map $T$ is not one-to-one as example~\ref{seqTransfEx1} below demonstrates.
\end{rem}
\begin{ex}\label{seqTransfEx1}
Consider any $i \neq j$ and any $a, \, b, \, c$ and $d \in A$. Notice that the transformations $\nu_{i, \, a, \, b}$ and $\nu_{j, \, c, \, d}$ commute since the order in which elements of distinct equivalence classes are interchanged within the same population of rollouts is irrelevant. Thus the sequences $\vec{\chi_1} = (\nu_{i, \, a, \, b}, \, \nu_{j, \, c, \, d})$ and $\vec{\chi_2} = (\nu_{j, \, c, \, d}, \, \nu_{i, \, a, \, b})$ induce exactly the same transformation $\Theta$ on the set of populations of rollouts. Here is another very important example. Notice that every transformation $\Theta_{i, \, a, \, b}$ where $\Theta$ could be either $\chi$ or $\nu$ is an involution on the set of populations of rollouts i.e. $\Theta_{i, \, a, \, b} \circ \Theta_{i, \, a, \, b} = e$ where $e$ is the identity map since performing a swap at identical positions twice brings back the original population of rollouts. Therefore any ordered pair $(\Theta_{i, \, a, \, b}, \, \Theta_{i, \, a, \, b})$ of repeated transformations induce exactly the same transformation as the symbol $\mathbf{1}$, namely the identity transformation on the population of rollouts.
\end{ex}
One more remark is in order here.
\begin{rem}\label{closureUnderConcatRem}
Notice that any concatenation of sequences in $\mathcal{F}$ (which is what corresponds to the composition of the corresponding functions) stays in $\mathcal{F}$. In other words, the family of maps induced by $\mathcal{F}$ is closed under composition.
\end{rem}
Of course, running the Markov process induced by the transition matrix in definition~\ref{RecombStagePopTransDefn} infinitely long is impossible, but fortunately one does not have to do it. The central idea of the current paper is that the limiting outcome as time goes to infinity can be predicted exactly using the Geiringer-like theory and the desired evaluations of moves can be well-estimated at rather little computational cost in most cases. As pointed out in example~\ref{seqTransfEx1} above, each of the transformations $\Theta_{i, \, a, \, b}$ is an involution and, in particular, is bijective. Therefore, every composition of these transformations is a bijection as well. We deduce, thereby, that the family $\mathcal{F}$ consists of bijections only (see remark~\ref{seqVsTransfRem}). The finite population Geiringer theorem (see \cite{MitavRowGeirMain}) now applies and tells us the following:
\begin{defn}\label{equivRelForMCTDefn}
Given populations $P$ and $Q$ of rollouts at a specified state in question as in definition~\ref{popOfRolloutsDefn} (see also remark~\ref{crossoverConvRepresRem}), we say that $P \sim Q$ if there is a transformation $F \in \mathcal{F}$ such that $Q = F(P)$.
\end{defn}
\begin{thm}[The Geiringer Theorem for POMDPs]\label{GeirThmForMCTMain}
The relation $\sim$ introduced in definition~\ref{equivRelForMCTDefn} is an equivalence relation. Given a population $P$ of rollouts at a specified state in question, the restriction of the Markov transition matrix introduced in definition~\ref{RecombStagePopTransDefn} to the equivalence class $[P]$ of the population $P$ under $\sim$ is a well-defined Markov transition matrix which induces an irreducible and aperiodic Markov chain on $[P]$ and the unique stationary distribution of this Markov chain is the uniform distribution on $[P]$.
\end{thm}
In fact, thanks to the application of the classical contraction mapping principle\footnote{This simple and elegant classical result about complete metric spaces lies in the heart of many important theorems such as the ``existence uniqueness" theorem in the theory of differential equations, for instance.} described in section~\ref{furthStrengthSect} of the current paper (namely theorem~\ref{MarkovNonMarkovSubtleExtCorExt}; interested reader is welcome to familiarize themselves with section~\ref{furthStrengthSect}, although this is not essential to understand the main objective of the paper), the stationary distribution is uniform in a rather strong sense described below.
\begin{thm}\label{GeiringerExtThm}
Suppose we are given finitely many probability measures $\mu_1$, $\mu_2, \ldots, \mu_N$ on the collection of sequences of transformations $\mathcal{F}$ as in definition~\ref{RecombStagePopTransDefn} where each probability measure $\mu_i$ satisfies the conditions of definition~\ref{RecombStagePopTransDefn}. Denote by $M_i$ the corresponding Markov transition matrix induced by the probability measure $\mu_i$. Let $\mathcal{M} = \{M_i\}_{i=1}^N$. Now consider the following stochastic process $\{(\Phi_n, \, X_n)\}_{n=0}^{\infty}$ on the state space $\mathcal{M} \times [P]$ where $[P]$ is the equivalence class of the initial population of rollouts at the state in question as in theorem~\ref{GeirThmForMCTMain}:
$\Phi_n$ is an arbitrary stochastic process (not necessarily Markovian) on $\mathcal{M}$ which satisfies the following requirement:
\begin{equation}\label{reqOnFirstProcEq}
\text{The random variable }\Phi_n \text{ is independent of the random variables } X_n, \, X_{n+1}, \ldots
\end{equation}
The random variable $\Phi_0$ is arbitrary while $X_0 = P$ (recall that $P$ is the initial population of rollouts at the node in question) with probability $1$.
$$\forall \, n \in \mathbb{N} \text{ the probability distribution of  the random variable } X_n, \text{ namely }$$
\begin{equation}\label{reqOnFirstProcEq1}
\text{Prob}(X_n = \cdot) = \Phi_{n-1}(w) \cdot \text{Prob}(X_{n-1} = \cdot).
\end{equation}
It follows then that $\lim_{n \rightarrow \infty}\text{Prob}(X_n = \cdot) = \pi$ where $\pi$ is the uniform distribution on $[P]$.
\end{thm}
We now pause and take some time to interpret theorem~\ref{GeiringerExtThm} intuitively. Example~\ref{intuitiveInterpEx} below illustrates a scenario where theorem~\ref{GeiringerExtThm} applies.
\begin{ex}\label{intuitiveInterpEx}
Consider the set $\mathcal{S}$ of all finite sequences of populations in the equivalence class $[P]$ of the initial population $P$ which start with the initial population $P$ (notice that $\mathcal{S}$ is a countably infinite set since $[P]$ is a finite set). Intuitively, each sequence in $\mathcal{S}$ represents prior history. Every sequence $\vec{P} = P, \, P_1, \, P_2, \, P_3, \ldots, \, P_t$ is associated with a probability measure $\eta(\vec{P})$ on the set of populations $[P]$. Suppose further that to every population $Q \in [P]$ we assign a probability measure $\mu_Q$ on the family of recombination transformations induced by $\mathcal{F}$ where each measure $\mu_Q$ satisfies definition~\ref{RecombStagePopTransDefn}. Intuitively, each probability distribution $\mu_Q$ might represent the probability that the swaps (or sequences of swaps) are reasonable to perform in a specific population regardless of the knowledge of the prior history or experience in playing the game, for instance. Starting with the initial population $P$ we apply the probability measure $\eta(P)$ (here $P$ denotes a singleton sequence) to obtain a population $Q_1 \in [P]$. \emph{Independently} we now apply the Markov transition matrix induced by the probability measure $\mu_P$ to obtain another population $P_1 \in [P]$. Next, we select a population $Q_2$ with respect to the probability measure $\eta(P, P_1)$ and, again \emph{independently}, apply the Markov transition matrix $\mu_{Q_1}$ to the population $P_1$ to obtain a population $P_2$ in the next generation. Continuing recursively, let's say after time $t \in \mathbb{N}$ we obtained a population $Q_t$ at step $t$ and a sequence of populations $\vec{P}_t = P, \, P_1, \, P_2, \dots, P_t$. Select a population $Q_{t+1}$ with respect to the probability measure $\eta(\vec{P}_t)$. \emph{Independently} select a population $P_{t+1}$ via an application of the Markov transition matrix induced by the probability measure $\mu_{Q_t}$ to the population $P_t$. Theorem~\ref{GeiringerExtThm} applies now and tells us that in the limit as $t \rightarrow \infty$ we are equally likely to encounter any population $Q \in [P]$ regardless of the choice of the measures involved as long as the probability measures $\mu_Q$ satisfy definition~\ref{RecombStagePopTransDefn}. A word of caution is in order here: it is not in vain that we emphasize that selection is made ``independently" here. Theorem~\ref{GeiringerExtThm} simply does not hold without this assumption.
\end{ex}
Evidently example~\ref{intuitiveInterpEx} represents just one of numerous possible interpretations of theorem~\ref{GeiringerExtThm}. We hope that other authors will elaborate on this point. Knowing that the limiting frequency of occurrence of a any two given populations $Q_1$ and $Q_2 \in [P]$ is the same, it is possible to compute the limiting frequency of occurrence of any specific rollout and even certain subsets of rollouts using the machinery developed in \cite{MitavRowGeirMain} and \cite{MitavRowGeirGenProgr} which is also presented in section~\ref{theoryFoundSect} of the current paper for the sake of self-containment.
To state and derive these ``Geiringer-like" results we need to introduce the appropriate notions of schemata (see, for instance, \cite{Antonisse} and \cite{PoliSchema}) here.
\subsection{Schemata for MCT Algorithm}\label{schemataSubsectShallow}
\begin{defn}\label{schemaForMCTPopDefHolland}
Given a state $(s, \vec{\alpha})$ in question (see definition~\ref{treeRootedByChanceNode}), a rollout \emph{Holland-Poli schema} is a sequence consisting of entries from the set$\vec{\alpha} \cup \mathbb{N} \cup \{\#\} \cup \Sigma$ of the form $h = \{x_i\}_{i=1}^k$ for some $k \in \mathbb{N}$ such that for $k>1$ we have $x_1 \in \vec{\alpha}$, $x_i \in \mathbb{N}$ when $1 < i < k$ represents an equivalence class of states, and $x_k \in \{\#\} \cup \Sigma$ could represent either a terminal label if it is a member of the set of terminal labels $\Sigma$, or any substring defining a valid rollout if it is a $\#$ sign.\footnote{This notion of a schema is somewhat of a mixture between Holland's and Poli's notions.} For $k=1$ there is a unique schema of the form $\#$. Every schema uniquely determines a set of rollouts
$S_h = \begin{cases}
\{(x_1, \, (x_2, a_2), \, (x_3, a_3), \ldots, (x_{k-1}, a_{k-1}), x_k) \, \\| \, a_i \in A \text{ for } 1<i<k\} & \text{ if } k>1 \text{ and } x_k \in \Sigma\\
\{(x_1, \, (x_2, a_2), \, (x_3, a_3), \ldots, (x_{k-1}, a_{k-1}), \, \\(y_k, a_k), \, (y_{k+1}, a_{k+1}), \ldots, f)\\
\, | \, a_i \in A \text{ for } 1<i<k, \, y_j \in \mathbb{N} \text{ and } a_j \in A\} & \text{ if } k > 1 \text{ and } x_k = \#\\
\text{the entire set of all possible rollouts} & \text{ if } k = 1 \text{ or, equivalently, } h=\#.
\end{cases}$ which fit the schema in the sense mentioned above. We will often abuse the language and use the same word schema to mean either the schema $h$ as a formal sequence as above or schema as a set $S_h$ of rollouts which fit the schema. For example, if $h$ and $h^*$ is a schema, we will write $h \cap h^*$ as a shorthand notation for $S_h \cap S_{h^*}$ where $\cap$ denotes the usual intersection of sets. Just as in definition~\ref{RolloutDefn}, we will say that $k-1$, the number of states in the schema $h$ is the \emph{height} of the schema $h$.
\end{defn}
We illustrate the important notion of a schema with an example below:
\begin{ex}\label{schemaDefEx}
Suppose we are given a schema $h=(\alpha, \, 1, \, 2, \, \#)$. Then the rollouts $(\alpha, \, 1a, \, 2c, \, 5a, \, 3c, \, f)$ and $(\alpha, \, 1d, \, 2a, \, 3a, \, 3d, \, g) \in S_h$ or one could say that both of them fit the schema $h$. On the other hand the rollout $(\beta, \, 1a, \, 2c, \, 5a, \, 3c, \, f) \notin S_h$ (or does not fit the schema $h$) unless $\alpha = \beta$. A rollout $(\alpha, \, 1a, \, 3a, \, 5a, \, 3c, \, f) \notin S_h$ does not fit the schema $h$ either since $x_2 = 2 \neq 3$. Neither of the rollouts above fit the schema $h^* = (\alpha, \, 1, \, 2, \, f)$ since the appropriate terminal label is not reached in the $4^{\text{th}}$ position. An instance of a rollout which fits the schema $h^*$ would be $(\alpha, \, 1c, \, 2b, \, f)$.
\end{ex}
The notion of schema is useful for stating and proving Geiringer-like results largely thanks to the following notion of partial order.
\begin{defn}\label{schemaPosetDef}
Given schemata $h$ and $g$ we will write $h > g$ either if $h=\#$ and $g \neq \#$ or $h=(x_1, \, x_2, \, x_3, \ldots, x_{k-1}, \, \#)$ while $g = (x_1, \, x_2, \, x_3, \ldots, x_{k-1}, \, y_{k}, \, y_{k+1}, \ldots, y_{l-1}, \, y_l)$ where $y_l$ could be either of the allowable values: a $\#$ or a terminal label $f \in \Sigma$. However, if $y_l = \#$ then we require that $l > k$.
\end{defn}
An obvious fact following immediately from definitions~\ref{schemaForMCTPopDefHolland} and \ref{schemaPosetDef} is the following.
\begin{prop}\label{schemaPartOrderHollandShallow}
Suppose we are given schemata $h$ and $g$. Then $h \geq g \Longrightarrow S_h \supseteq S_g$.
\end{prop}
\subsection{The Statement of Geiringer-like Theorems for the POMDPs}\label{GeirThmStSubsect}
In evolutionary computation Geiringer-like results address the limiting frequency of occurrence of a set of individuals fitting a certain schema (see \cite{PoliGeir}, \cite{MitavRowGeirMain} and \cite{MitavRowGeirGenProgr}). In this work our theory rests on the finite population model based on stationary distribution of the Markov chain of all populations potentially encountered in the process (see theorems~\ref{GeirThmForMCTMain} and \ref{GeiringerExtThm} and example~\ref{intuitiveInterpEx}). The ``limiting frequency of occurrence" (rigorous definition appears in section~\ref{GeneralThmSect}, subsection~\ref{MethodologySubsect}, definitions~\ref{SetPopulationCountDefn} and \ref{limitFreqOfOccurrenceRandVarDefn}, however for the readers who aim only at ``calculus-level" understanding with the goal of applying the main ideas directly in their software engineering work we will discuss the intuitive idea in more detail below) of a certain subset of individuals determined by a Holland-Poli schema $h$ among all the populations in the equivalence class $[P]$ as time increases (i.e. as $t \rightarrow \infty$) of the initial population of rollouts $P$ will be expressed solely in terms of the initial population $P$ and schema $h$. These quantities are defined below.
\begin{defn}\label{popStateCountDownDef}
For any action under evaluation $\alpha$ define a set-valued function $\alpha \downarrow$ from the set $\Omega^b$ of populations of rollouts to the power set of the set of natural numbers $\mathcal{P}(\mathbb{N})$ as follows: $\alpha \downarrow(P) = \{i \, | \, i \in \mathbb{N}$ and at least one of the rollouts in the population $P$ fits the Holland schema $(\alpha, \, i, \, \#)\}$. Likewise, for an equivalence class label $i \in \mathbb{N}$ define a set valued function on the populations of size $b$, as $i \downarrow (P) = \{j \, | \, \exists \, x$ and $y \in A$ and a rollout $r$ in the population $P$ such that $r = (\ldots, (i, \, x), \, (j, \, y), \ldots) \, \} \cup \{f \, | \, f \in \Sigma$ and $\exists$ an $x \in A$ and a rollout $r$ in the population $P$ such that $r = (\ldots, (i, \, x), \, f) \, \}$. In words, the set $i \downarrow (P)$ is the set of all equivalence classes together with the terminal labels which appear after the equivalence class $i$ in at least one of the rollouts from the population $P$. Finally, introduce one more function, namely $i \downarrow_\Sigma: \Omega^b \rightarrow \mathbb{N} \cup \{0\}$ by letting $i \downarrow_\Sigma(P) = |\{f \, | \, f \in \Sigma \cap i \downarrow (P)\}|$, that is, the total number of terminal labels (which are assumed to be all formally distinct for convenience) following the equivalence class $i$ in a rollout of the population $P$.
\end{defn}
As always, we illustrate definition~\ref{popStateCountDownDef} in example~\ref{popSchemaDownFunctEx} below.
\begin{ex}\label{popSchemaDownFunctEx}
Continuing with example~\ref{PopRolloutEx}, we return to the population $P$ in figure~\ref{PopOfRollsFig}. From the picture we see that the only equivalence class $i$ such that a rollout from the population $P$ fits the Holland schema $(\alpha, \, i, \, \#)$ is $i = 1$ so that $\alpha \downarrow (P) = \{1\}$. Likewise, the only equivalence class following the action $\beta$ is $2$, the only equivalence class following the action $\gamma$ is $4$ and the only one following $\pi$ is $3$ so that $\beta \downarrow (P) = \{2\}$, $\gamma \downarrow (P) = \{4\}$ and $\pi \downarrow (P) = \{3\}$. The only equivalence classes $i$ following the action $\xi$ in the population $P$ are $i=3$ and $i=2$ so that the set $\xi \downarrow (P) = \{2, \, 3\}$.

Likewise the fragment $(1, \, a), (5, \, a)$ appears in the first (leftmost) rollout in $P$, $(1, \, b), (3, \, c)$ in the second rollout, $(1, \, c), (4, \, b)$ in the forth tollout and $(1, \, d), (2, \, e)$ in the last, seventh rollout. No other equivalence class or a terminal label follows the equivalence class of the state $1$ in the population $P$ and so it follows that $1 \downarrow (P) = \{5, \, 3, \, 4, \, 2\}$ and $1 \downarrow_\Sigma (P) = |\{\emptyset\}|=0$. Likewise, equivalence class $1$ follows the equivalence class $2$ in the second rollout, $7$ follows $2$ in the forth rollout, $4$ follows $2$ in the fifth rollout and $6$ follows $2$ in the last, seventh rollout. The only terminal label that follows the equivalence class $2$ is $f_6$ in the $6^{\text{th}}$ rollout. Thus we have $2 \downarrow (P) = \{1, \, 7, \, 4, \, 6, \, f_6\}$ and $2 \downarrow_\Sigma (P) = |\{f_6\}| = 1$. We leave the reader to verify that $$3 \downarrow (P) = \{7, \, 6, \, 2, \, 1\} \text{ so that } 3 \downarrow_\Sigma (P) = 0,$$ $$4 \downarrow (P) = \{6, \, 2, \, f_5\} \text{ so that } 4 \downarrow_\Sigma (P)  = 1,$$ $$5 \downarrow (P) = \{6, \, f_3, \, f_4\} \text{ and so } 5 \downarrow_\Sigma (P) = 2,$$ $$6 \downarrow (P) = \{3, \, 5, \, f_2, \, f_7\} \text{ and so } 6 \downarrow_\Sigma (P) = 2$$ and, finally, $7 \downarrow (P) = \{5, \, f_1\} \text{ so that } 7 \downarrow_\Sigma (P)  = 1$.
\end{ex}
\begin{rem}\label{totalNumberOfTermLblsRem}
Note that according to the assumption that all the terminal labels within the same population are distinct (see definition~\ref{popOfRolloutsDefn} together with the comment in the footnote there). But then, since every rollout ends with a terminal label, we must have $\sum_{i=1}^{\infty}i \downarrow_\Sigma (P) = b$ (of course, only finitely many summands, namely these equivalence classes that appear in the population $P$ may contribute nonzero values to $\sum_{i=1}^{\infty}i \downarrow_\Sigma (P)$) where $b$ is the number of rollouts in the population $P$, i.e. the size of the population $P$. For instance, in example~\ref{popSchemaDownFunctEx} $b = 7$ and there are totally $7$ equivalence classes, namely $1, \, 2, \, 3, \, 4, \, 5, \, 6$ and $7$ that occur within the population in figure~\ref{PopOfRollsFig} so that we have $\sum_{i=1}^{\infty}i \downarrow_\Sigma (P) = \sum_{i=1}^7 i \downarrow_\Sigma (P) = 0 + 1 + 0 + 1 + 2 + 2 + 1 = 7 = b$.
\end{rem}
Another important and related definition we need to introduce is the following:
\begin{defn}\label{PopStateOrderSeqDef}
Given a population $P$ and integers $i$ and  $j \in \mathbb{N}$ representing equivalence classes, let $$\text{Order}(i \downarrow j, \, P) = \begin{cases}
0 & \text{if } i(P) = 0 \text{ or } j \notin i \downarrow (P) \\
|\{((i, a), \, (j, \, b)) \, | \, \text{ the segment } & \text{ }\\
((i, a), \, (j, \, b)) \text{ appears in one of the} & \text{ }\\
\text{rollouts in the population }P\}| & \text{ otherwise }
\end{cases}.$$
Loosely speaking, $\text{Order}(i \downarrow j, \, P)$ is the total number of times the equivalence class $j$ follows the equivalence class $i$ within the
population of rollouts $P$.

Likewise, given a population of rollouts $P$, an action $\alpha$ under evaluation and an integer $j \in \mathbb{N}$, let $$\text{Order}(\alpha \downarrow j, \, P) = \begin{cases}
0 & \text{if } i(P) = 0 \text{ or } j \notin \alpha \downarrow j \\
|\{(\alpha, \, (j, \, b)) \, | \, \text{ the segment } & \text{ }\\
(\alpha, \, (j, \, b)) \text{ appears in one of the} & \text{ }\\
\text{rollouts in the population }P\}| & \text{ otherwise }
\end{cases}.$$
Alternatively, $\text{Order}(\alpha \downarrow j, \, P)$ is the number of rollouts in the population $P$ fitting the rollout Holland schema $(\alpha, \, j, \, \#)$.
\end{defn}
We now provide an example to illustrate definition~\ref{PopStateOrderSeqDef}.
\begin{ex}\label{PopStateOrderEx}
Continuing with example~\ref{popSchemaDownFunctEx} and population $P$ appearing in figure~\ref{PopOfRollsFig}, we recall that $\alpha \downarrow (P) = \{1\}$. we immediately deduce that $\text{Order}(\alpha, \, j, \, \#) = 0$ unless $j=1$. There are two rollouts, namely the first and the forth, that fit the schema $(\alpha, \, 1, \, \#)$ so that $\text{Order}(\alpha \downarrow 1, \, P) = 2$. Likewise, $\beta \downarrow (P) = \{2\}$ and exactly one rollout, namely the second one, fits the Holland schema $(\beta, \, 2, \, \#)$ so that $\text{Order}(\beta, \, j, \, \#) = 0$ unless $j=2$ while $\text{Order}(\beta \downarrow 2, \, P) = 1$. Continuing in this manner (the reader may want to look back at example~\ref{popSchemaDownFunctEx}), we list all the nonzero values of the function $\text{Order}(\text{action}, \Box, \, P)$ for the population $P$ in figure~\ref{PopOfRollsFig}: $\text{Order}(\gamma \downarrow 4, \, P) = \text{Order}(\xi \downarrow 3, \, P) = \text{Order}(\xi \downarrow 2, \, P) = \text{Order}(\pi \downarrow 3, \, P) = 1$.

Likewise, recall from example~\ref{popSchemaDownFunctEx}, that $1 \downarrow (P) = \{5, \, 3, \, 4, \, 2\}$ so that $\text{Order}(1 \downarrow j, \, P) = 0$ unless $j = 5$ or $j=3$ or $j=4$ or $j=1$. It happens so that a unique rollout exists in the population $P$ fitting each fragment $(1, \, (j, \, \text{something in }A))$ for $j = 5$, $j=3$, $j=4$ and $j=2$ respectively, namely the first, the second, the forth and the last (seventh) rollouts. According to definition~\ref{PopStateOrderSeqDef}, we then have $\text{Order}(1 \downarrow 5, \, P) = \text{Order}(1 \downarrow 3, \, P) = \text{Order}(1 \downarrow 4, \, P) = \text{Order}(1 \downarrow 2, \, P) = 1$. Analogously, $2 \downarrow (P) = \{1, \, 7, \, 4, \, 6, \, f_6\}$ so that $\text{Order}(2 \downarrow j, \, P) = 0$ unless $j = 1, \, 7, \, 4$ or $6$. The only rollout in the population $P$ involving the fragment with $1$ following $2$ is the second one, the only one involving $7$ following $2$ is the forth, the only one involving $4$ following $2$ is the fifth, and the only one involving $6$ following $2$ is the last (the seventh) rollouts respectively so that $\text{Order}(2 \downarrow 1, \, P) = \text{Order}(2 \downarrow 7, \, P) = \text{Order}(2 \downarrow 4, \, P) = \text{Order}(2 \downarrow 6, \, P) = 1$. Continuing in this manner, we list all the remaining nonzero values of the ``Order" function introduced in definition~\ref{PopStateOrderSeqDef} for the population $P$ in figure~\ref{PopOfRollsFig}: $$\text{Order}(3 \downarrow 7, \, P) = \text{Order}(3 \downarrow 6, \, P) = \text{Order}(3 \downarrow 2, \, P) = \text{Order}(3 \downarrow 1, \, P) = 1,$$
$$\text{Order}(4 \downarrow 6, \, P) = \text{Order}(4 \downarrow 2, \, P) = 1,$$
$$\text{Order}(5 \downarrow 6, \, P) = \text{Order}(6 \downarrow 3, \, P) = \text{Order}(6 \downarrow 5, \, P) = \text{Order}(7 \downarrow 5, \, P) =1.$$
\end{ex}
\begin{rem}\label{equivClassIndepRem}
It must be noted that all the functions introduced in definitions~\ref{popStateCountDownDef} and \ref{PopStateOrderSeqDef} remain invariant if one were to apply the ``primitive" recombination transformations from the family $\mathcal{S}$ as in definitions~\ref{RecombStagePopTransDefn} and \ref{recombActOnPopsDef} to the population in the argument. More explicitly, given any population of rollouts $P$, an action $\alpha$ under evaluation, an equivalence class $i \in \mathbb{N}$, a Holland-Poli schema $h = (\alpha, \, i_1, \, i_2, \ldots, i_{k-1}, x_k)$ an integer $j$ with $1 \leq j \leq k$, and any recombination transformation $\mathcal{R} \in \mathcal{S}$, we have $$\alpha \downarrow (P) = \alpha \downarrow (\mathcal{R}(P)), \; i \downarrow (P) = i \downarrow (\mathcal{R}(P)),$$$$i \downarrow_\Sigma (P) = i \downarrow_\Sigma (\mathcal{R}(P)), \; \text{Order}(q \downarrow r, \, P) = \text{Order}(q \downarrow r, \, \mathcal{R}(P)).$$
Indeed, the reader may easily verify that performing a swap of the elements of the same equivalence class, or of the corresponding subtrees pruned at equivalent labels, preserves all the states which are present within the population and creates no new ones. Moreover, the equivalence class sequel is also preserved and hence the invariance of the functions $\alpha \downarrow$ and $i \downarrow$ etc. follows. Since every transformation in the family $\mathcal{F}$ is a composition of the crossover transformations from the family $\mathcal{S}$, it follows at once that all of the functions introduced in definitions~\ref{popStateCountDownDef} and \ref{PopStateOrderSeqDef} are constant on the equivalence classes of populations under the equivalence relation introduced in definition~\ref{equivRelForMCTDefn}.
\end{rem}
\begin{ex}\label{RemEx}
Recall from example~\ref{popCrossEx1} that the populations in figures~\ref{PopOfRollsFig}, \ref{PopOfRollsAfterCross1} and \ref{PopOfRollsAfterCross2} are equivalent and, likewise, according to example~\ref{popCrossEx2}, the populations in figures~\ref{PopOfRollsFig2} and \ref{PopOfRollsAfterCross3} are equivalent. Moreover, example~\ref{seqTransfEx1} demonstrates that the populations displayed in figures~\ref{PopOfRollsFig} and \ref{PopOfRollsFig2} are also equivalent. Thus all of the populations that appear in figures~\ref{PopOfRollsFig}, \ref{PopOfRollsAfterCross1}, \ref{PopOfRollsAfterCross2}, \ref{PopOfRollsFig2} and \ref{PopOfRollsAfterCross3} belong to the same equivalence class under the relation $\sim$ introduced in definition~\ref{equivRelForMCTDefn}. In view of remark~\ref{equivClassIndepRem}, all the functions appearing in definitions~\ref{popStateCountDownDef} and \ref{PopStateOrderSeqDef} produce identical values on the populations displayed on figures~\ref{PopOfRollsFig}, \ref{PopOfRollsAfterCross1}, \ref{PopOfRollsAfterCross2}, \ref{PopOfRollsFig2} and \ref{PopOfRollsAfterCross3}
\end{ex}
Observe that applying any recombination transformation of the form $\chi_{i, \, a, \, b}$ or $\nu_{i, \, a, \, b}$ to a population $P$ of rollouts neither removes any states from the population nor adds any new ones, and hence the following invariance property of the equivalent populations that will largely alleviate theoretical analysis in section~\ref{GeneralThmSect} follows.
\begin{rem}\label{averageHightRem}
Given any population $Q \in [P]$, the total number of states in the population $Q$ is the same as that in the population $P$. Apparently, as we already mentioned, the the total number of states in a population is the sum of the heights of all rollouts in that population (see definition~\ref{RolloutDefn} and \ref{popOfRolloutsDefn}). It follows then, that the sum of the heights of all rollouts within a population is an invariant quantity under the equivalence relation in definition~\ref{equivRelForMCTDefn}. In other words, if $Q \sim P$ then the sum of the heights of the rollouts in the population $Q$ is the same as the sum of the heights of the rollouts in the population $P$.
\end{rem}
There is yet one more important notion, namely that of the ``limiting frequency of occurrence" of a schema as one runs the genetic programming routine with recombination only we need to introduce to state the Geiringer-like results of the current paper. A rigorous definition in the most general framework appears in subsection~\ref{MethodologySubsect} of section~\ref{GeneralThmSect} (namely, definitions~\ref{SetPopulationCountDefn} and \ref{limitFreqOfOccurrenceRandVarDefn}), nonetheless, for less patient readers, who aim only at the ``calculus level" understanding, we explain informally what the limiting frequency of occurrence is.

$\;$

\emph{Informal Description of the Limiting Frequency of Occurrence:} Given a schema $h$ and a population $P$ of size $m$, suppose we run the Markov process $\{X_n\}_{n=0}^{\infty}$ on the populations in the equivalence class $[P]$ of the initial population of rollouts $P$ as in definition~\ref{RecombStagePopTransDefn}, or, more generally, the non-homogenous time Markov process as described in theorem~\ref{GeiringerExtThm} (where the Markov transition matrices introduced in definition~\ref{RecombStagePopTransDefn} are chosen randomly with respect to another stochastic process (not necessarily Markovian) that does not depend on the current population but may depend on the entire history of former populations as well as on other external parameters independent of the current population). As discussed in the preceding paragraph, this corresponds to ``running the genetic programming routine forever" and each recombination models the changes in player's strategies due to incomplete information, randomness personality etc. Up to time $t$ a total of $m \cdot t$ individuals (counting repetitions) have been encountered. Among these a certain number, say $h(t)$, fit the schema $h$ in the sense of definition~\ref{schemaForMCTPopDefHolland}. We now let $\Phi(P, \, h, \, t) = \frac{h(t)}{m \cdot t}$ to be the proportion of these individuals fitting the schema $h$ out of the total number of individuals encountered up to time $t$. It follows from theorem~\ref{GeirThmForMCTMain} via the instruments presented in section~\ref{MethodologySubsect} (also available in \cite{MitavRowGeirMain} and \cite{MitavsGPGeir}) that $\lim_{t \rightarrow \infty} \Phi(P, \, h, \, t)$ exists and the formula for it will be given purely in terms of the parameters of the initial population $P$ (more specifically, in terms of the functions described in definitions~\ref{popStateCountDownDef} and \ref{PopStateOrderSeqDef}. Although it may be possible to derive the formulas for $\lim_{t \rightarrow \infty} \Phi(P, \, h, \, t)$ in the most general case when the initial population of rollouts $P$ is non-homologous (in other words when the states representing the same equivalence class may appear at various ``heights" in the same population of rollouts: see definition~\ref{popOfRolloutsDefn}), the formulas obtained in this manner would definitely be significantly more cumbersome and would not be as well suited for algorithm development\footnote{This is an open question, yet it's practical importance is highly unclear} as the limiting result with respect to ``inflating" the initial population $P$ in the sense described below. Remarkably, the formula for the limiting result in the general non-homologous initial population case coincides with the one for the homologous populations.
\begin{defn}\label{popInflationDef}
Given a population $P = \{r_i^{l(i)}\}_{i=1}^b$ of rollouts in the sense of definition~\ref{popOfRolloutsDefn}, where $r_i = \{(\alpha_i, \, (j_1^i, a_1^i), \, (j_2^i, a_2^i), \ldots, (j^i_{l(i)-1}, a_{l(i)-1}^i) \, f_i)\}$ and a positive integer $m$, we first increase the size of the alphabet $A$ by a factor of $m$: formally, let the alphabet
$$A \times m = \{(a, \, i) \, | \, a \in A, \, i \in \mathbb{N} \text{ and } 1 \leq i \leq m\}.$$ Likewise, we also increase the terminal set of labels $\Sigma$ by a factor of $m$ so that $$\Sigma \times m = \{(f, \, i) \, | \, f \in \Sigma, \, i \in \mathbb{N} \text{ and } 1 \leq i \leq m\}.$$ Now we let $$P_m = \{r_{i, \, k}^{l(i)}\}_{1 \leq i \leq b \text{ and }1 \leq k \leq m}$$ where $$r_{i, \, k}^{l(i)} = \{(\alpha_i, \, (j_1^i, (a_1^i, \, k)), \, (j_2^i, (a_2^i, \, k)), \ldots, (j^i_{l(i)-1}, (a_{l(i)-1}^i, \, k)), \, (f_i, k))\}.$$ We will say that the population $P_m$ is an \emph{inflation} of the population $P$ by a factor of $m$.
\end{defn}
Essentially, a population $P_m$ consists of $m$ \emph{formally distinct} copies of each rollout in the population $P$. Intuitively speaking, the stochastic information captured in the sample of rollouts comprising the population $P_m$ (such as the frequency of obtaining a state in the equivalence class of $j$ after a state in the equivalence class of $i$) is the same as the one contained within the population $P$ emphasized by the factor of $m$. In fact, the following rather important obvious facts make some of this intuition precise:
\begin{prop}\label{popRatioFacts}
Given a population $P$ of rollouts and a positive integer $m$ consider the inflation of the population $P$ by a factor of $m$, $P_m$ as in definition~\ref{popInflationDef}. Then the following are true:
$$\alpha \downarrow (P_m) = \alpha \downarrow (P), \; i \downarrow (P_m) = i \downarrow (P), i \downarrow_\Sigma (P_m) = m \cdot i \downarrow_\Sigma (P)$$  while
\begin{equation}\label{inflatedQuantitiesEq}
\emph{Order}(\alpha \downarrow j, \, P_m) = m \cdot \emph{Order}(\alpha \downarrow j, \, P), \; \emph{Order}(q \downarrow r, \, P_m) = m \cdot \emph{Order}(q \downarrow r, \, P)
\end{equation}
For any population of rollouts $Q$ let $\emph{Total}(Q)$ denote the total number of states in the population $Q$ which is, of course, the same thing as the sum of the heights
of all rollouts in the population $Q$. Then clearly $\emph{Total}(P_m) = m \cdot \emph{Total}(P)$. In the special case when $P$ is a homologous population, $\forall \, m \in \mathbb{N}$ so is the population $P_m$.
\end{prop}
When using Holland-Poli schemata with respect to any population $Q \in [P_m]$ we will adopt the following convention:  \begin{rem}\label{terminalHollandShemataRem}
Given a Holland-Poli schema $h = (\alpha, i_1, \, i_2, \ldots, i_{k-1}, \, f)$ and a population $Q \in [P_m]$, an individual (i.e. a rollout) $r$ of the population $Q$ fits the schema $h$ if and only if it is of the form
$r = (\alpha, (i_1, \, (a_1, \, j_1)), \, (i_2, \, (a_2, \, j_2)) \ldots, (i_{k-1}, a_{k-1}, j_{k-1}), \, (f, \, j_k))$. Informally speaking, everything is as in definition~\ref{schemaForMCTPopDefHolland} with the exception that the terminal symbol of the schema $h$, namely $f \in \Sigma$ while the terminal symbol of the rollout $r$ is an ordered pair of the terminal symbol $f$ coupled with a numerical label between $1$ and $m$ so that we require only the first element of the ordered pair, namely the function label $f$, to match.
\end{rem}
We are finally ready to state the main result of the current paper.
\begin{thm}[The Geiringer-Like Theorem for MCT]\label{GeiringerLikeThmForMCTMain}
Repeat verbatim the assumptions of theorem~\ref{GeiringerExtThm}. Let $$h = (\alpha, \, i_1, \, i_2, \ldots, i_{k-1}, x_k)$$ where $x_k \in \{\#\} \cup \Sigma$ be a given Holland-Poli schema. For $m \in \mathbb{N}$ consider the random variable $\Phi(P_m, \, h, \, t)$ described in the paragraph just above (alternatively, a rigorous definition in the most general framework appears in subsection~\ref{MethodologySubsect} of section~\ref{GeneralThmSect}: definitions~\ref{SetPopulationCountDefn} and \ref{limitFreqOfOccurrenceRandVarDefn}) with respect to the Markov process $X_n^m$ where $m$ indicates that the initial population of rollouts is the inflated population $P_m$ as in definition~\ref{popInflationDef} with the new alphabet $A \times m$ labeling the states (see also example~\ref{intuitiveInterpEx} for help with understanding of the Markov process $X_n$). Then
\begin{equation}\label{GeiringerThmMainEq}
\lim_{m \rightarrow \infty}\lim_{t \rightarrow \infty}\Phi(P_m, \, h, \, t) = \frac{\emph{Order}(\alpha \downarrow i_1, \, P)}{b} \times$$$$ \times \left(\prod_{q = 2}^{k-1}\frac{\emph{Order}(i_{q-1}, \, i_q, \, P)}{\sum_{j \in i_{q-1} \downarrow} \emph{Order}(i_{q-1}, \, j, \, P) + i_{q-1} \downarrow_\Sigma (P)}\right) \cdot \text{\emph{LF}(P, h)}
\end{equation}
where $$\text{\emph{LF}(P, h)} = \begin{cases}
1 & \text{if } x_k = \#\\
0 & \text{if } x_k = f \in \Sigma \text{ and } f \notin x_{k-1} \downarrow (P)\\
%\; & \\ %\; & \text{rollout in the population }P\\
\emph{Fraction} & \text{if } x_k = f \in \Sigma \text{ and } f \in x_{k-1} \downarrow_{\Sigma} (P)
\end{cases}$$ where $$\emph{Fraction} = \frac{1}{\sum_{j \in i_{k-1} \downarrow (P)}\emph{Order}(i_{k-1}, \, j, \, P) +  i_{k-1} \downarrow_\Sigma (P)}$$ (we write ``LF" as short for ``Last Factor"). Furthermore, in the special case when the initial population $P$ is homologous (see definition~\ref{popOfRolloutsDefn}), one does not need to take the limit as $m \rightarrow \infty$ in the sense that $\lim_{t \rightarrow \infty}\Phi(P_m, \, h, \, t)$ is a constant independent of $m$ and its value is given by the right hand side of equation~\ref{GeiringerThmMainEq}.\footnote{The case of homologous recombination has been established in a different but mathematically equivalent framework in \cite{MitavRowGeirMain} and \cite{MitavRowGeirGenProgr} nonetheless we will derive it along with the general fact expressed in equation~\ref{GeiringerThmMainEq} to illustrate the newly enhanced methodology based on the lumping quotients of Markov chains described in subsection~\ref{LumpQuotSubsect}.}

An important comment is in order here: it is possible that the denominator of one of the fractions involved in the product is $0$. However, in such a case, the numerator is also $0$ and we adopt the convention (in this theorem only) that if the numerator is $0$ then, regardless of the value of the denominator (i.e. even if the denominator is $0$), then the fraction is $0$. As a matter of fact, a denominator of some fraction involved is $0$ if and only if one of the following holds: $\alpha(P) = 0$ or if there exists an index $q$ with $1 \leq q \leq k-1$ such that no state in the equivalence class of $i_q$ appears in the population $P$ (and hence in either of the inflated populations $P_m$).
\end{thm}
Theorem~\ref{GeiringerLikeThmForMCTMain} tells us that given any Holland-Poli rollout schema and a generating population $P$, $\forall \, \epsilon > 0$ $\exists$ a sufficiently large $M$ so that the right hand side of equation~\ref{GeiringerThmMainEq} provides an approximation of the limiting frequency of occurrence of the set of rollouts fitting the schema $h$ starting with the initial population $P_m$ which is the inflation of the population $P$ by a factor of $m > M$, namely $\lim_{t \rightarrow \infty}\Phi(P_m, \, h, \, t)$, with an error at most $\epsilon$.

Theorem~\ref{GeiringerLikeThmForMCTMain} is the main result of the current work. It motivates a variety of algorithms for evaluating the actions based on the entire, fairly large and seemingly pairwise disconnected sample of independent parallel rollouts that fully take advantage of the exponentially many possibilities already available within that sample and, at the same time, should be rather efficient in many situations. These algorithms will be the subject of sequel papers.
\section{Deriving Geiringer-like Theorems for POMDPs}\label{GeneralThmSect}
\subsection{Setting, Notation and the General Finite-Population Geiringer Theorem}\label{stageSubsect}
Throughout section~\ref{GeneralThmSect} (the current section) the following notation will be used: $\Omega$ is a \emph{finite set}, called a \emph{search space}. We fix an integer $b \in \mathbb{N}$ and we call $\Omega^b = \{(x_1, \, x_2, \, \ldots x_b) \, | \, x_i \in \Omega\}$ the set of \emph{populations of size} $b$; every element $\vec{x} = (x_1, \, x_2, \, \ldots x_b)^T \in \Omega^b$ is called a \emph{population of size} $b$ and every element $x \in \Omega$ is called an individual. Notice that we prefer to think of a population as a ``column vector" (hence the ``transpose symbol"). Of course, this is just the matter of preference, but normally when we list the individuals it is natural to write each individual as a string of ``genes or alleles" which appear on the same row and so the $b$ individuals appear on $b$ separate rows. It is important to emphasize here that populations are \emph{ordered} $b$-tuples so that $(x_1, \, x_2, \, \ldots x_b)^T \neq (x_b, \, x_2, \, \ldots x_1)^T$ unless $x_1 = x_b$. By a \emph{family of recombination transformations} we mean a family of functions $\mathcal{F} = \{F \, | \, F: \Omega^b \rightarrow \Omega^b\}$. The general finite population Geiringer theorem then says the following:
\begin{thm}[The Finite Population Geiringer Theorem for Evolutionary Algorithms]\label{GeneralGeiringerThm}
Suppose we are given a probability measure on the family of recombination transformations $\mathcal{F}$ on the set of populations $\Omega^b$ of size $b$ as described above. Suppose further there is a subfamily $\mathcal{S} \subseteq \mathcal{F}$ which generates the entire family $\mathcal{F}$ in the sense that $\forall \, F \in \mathcal{F}$ $\exists$ a finite sequence of transformations $S_1, \, S_2, \ldots, S_l \in \mathcal{S}$ such that $F = S_1 \circ S_2 \circ \ldots \circ S_l$. Assume the following about the probability measure $\mu$:
\begin{equation}\label{assumpGenGeir1}
\forall \, S \in \mathcal{S} \text{ we have } \mu(S) > 0.
\end{equation}
\begin{equation}\label{assumpGenGeir2}
\text{The identity map }\mathbf{1}: \Omega^b \rightarrow \Omega^b \text{ is in } \mathcal{S}
\end{equation}
Most importantly, assume that every recombination transformation $S \in \mathcal{S}$ is bijective (i.e. a one-to-one and onto function on $\Omega^b$). Consider the Markov transition matrix $M$ with state space $\Omega^b$ defined as follows: given populations $\vec{x}$ and $\vec{y} \in \Omega^b$, we let
\begin{equation}\label{MarkovChainDefEq}
p_{\vec{x} \rightarrow \vec{y}} = \mu(\{F \, | \, F \in \mathcal{F} \text{ and } F(\vec{x}) = \vec{y}\}).
\end{equation}
Now define a relation $\sim$ on $\Omega^b$ as follows: $\vec{x} \sim \vec{y}$ if and only if $\exists \, k \in \mathbb{N}$ and recombination transformations $F_1, \, F_2, \ldots, F_k \in \mathcal{F}$ such that $[F_1 \circ F_2 \circ \ldots \circ F_k](\vec{x}) = \vec{y}$. We now assert the following facts:
\begin{equation}\label{assertionEq1}
\sim \text{ is an equivalence relation.}
\end{equation}
$$\text{Given an equivalence class of some population } \vec{x}, \text{ call it } [\vec{x}],$$
$$\text{ the restriction of the Markov transition matrix } M \text{ to } [\vec{x}]$$
\begin{equation}\label{assertionEq2}
\text{ is a well-defined Markov transition matrix on the state space } [\vec{x}], \text{ call it }M|_{[\vec{x}]}.
\end{equation}
$$\forall \, \vec{x} \in \Omega^b \text{ the Markov transition matrix }M|_{[\vec{x}]}\text{ is doubly stochastic and}$$
\begin{equation}\label{assertionEq3}
\text{ it defines an irreducible and
aperiodic Markov chain on }[\vec{x}].
\end{equation}
\begin{equation}\label{assertionEq4}
\forall \, \vec{x} \in \Omega^b \text{ the unique stationary distribution of }M|_{[\vec{x}]}\text{ is the uniform distribution on }[\vec{x}].
\end{equation}
\end{thm}
Theorem~\ref{GeneralGeiringerThm} is a simple yet elegant consequence from basic group theory. In this paper we assume that the reader is familiar with fundamental notions about groups and group actions. Nearly any standard textbook in Abstract Algebra such as, for instance, \cite{DummittFoote} contains way more group theoretic material than necessary for our purpose. For a brief introduction we invite the reader to study \cite{MitavRowGeirMain}.
\begin{proof}
Since the family of transformations $\mathcal{S}$ consists entirely of bijections and any composition of bijections is also a bijection, the family $\mathcal{F}$ also consists solely of bijections. It follows then that the family $\mathcal{F}$ generates a subgroup $G$ of the group of all permutations on the finite set $\Omega^b$. Notice that the probability measure $\mu$ naturally extends to the entire group $G$ generated by $\mathcal{F}$ by defining $\mu_{\text{ext}}(g) = \begin{cases}
\mu(g) & \text{if } g \in \mathcal{F} \\
0 & \text{otherwise.}
\end{cases}$. Clearly the Markov process defined in the statement of theorem~\ref{GeneralGeiringerThm} (see \ref{MarkovChainDefEq}) can be redefined as
\begin{equation}\label{MarkovChainDefEqProof}
p_{\vec{x} \rightarrow \vec{y}} = \mu_{\text{ext}}(\{g \, | \, g \in G \text{ and } g(\vec{x}) = \vec{y}\}).
\end{equation}
Furthermore, notice that the group $G$ is of size no bigger than $|\Omega^b|! < \infty$ since $|\Omega| < \infty$. It follows then that every element $g \in G$ can be written as a finite composition $g = F_1 \circ F_2 \circ \ldots \circ F_k$ for $F_1, \, F_2, \ldots, F_k \in \mathcal{F}$ (because every element $F \in \mathcal{F} \subseteq G$ is a torsion element of $G$ i.e. $F^l = \mathbf{1}$ for some $l \in \mathbb{N}$ so that $F^{l-1} = F^{-1}$). But then the relation $\sim$ can be redefined as $\vec{x} \sim \vec{y}$ if and only if $\exists \, g \in G$ such that $g(\vec{x}) = \vec{y}$. We now quickly recognize that the relation $\sim$ is the orbit-defining equivalence relation which partitions the set of all populations of size $b$, $\Omega^b$, into the orbits under the action of the group $G$. The assertions expressed in equations~\ref{assertionEq1} and \ref{assertionEq2} now follow at once. To verify equation~\ref{assertionEq3} we choose any $\vec{y} \in \Omega^b$ and compute directly $$\sum_{\vec{x} \in \Omega^b} p_{\vec{x} \rightarrow \vec{y}} = \sum_{\vec{x} \in \Omega^b} \mu_{\text{ext}}(\{g \, | \, g \in G \text{ and } g(\vec{x}) = \vec{y}\}) = $$$$=\sum_{\vec{x} \in \Omega^b} \mu_{\text{ext}}(\{g \, | \, g \in G \text{ and } g^{-1}(\vec{y}) = \vec{x}\}) = \mu_{\text{ext}}(G) = 1$$ since the sets $K(x) = \{g \, | \, g \in G \text{ and } g^{-1}(\vec{y}) = \vec{x}\}$ clearly form a partition of $G$. We have now shown that the Markov transition matrix $M$ is doubly stochastic. Irreducibility follows from finiteness together with the fact that $\mathcal{S}$ generates $\mathcal{F}$. Since $\mathbf{1} \in \mathcal{S}$, aperiodicity follows as well. Now the classical result about Markov chains tells us that there is unique stationary distribution and since $M$ is doubly stochastic it must be the uniform distribution so that the final assertions expressed in equations~\ref{assertionEq3} and \ref{assertionEq4} follow at once.
\end{proof}
\subsection{A Methodology for the Derivation of Geiringer-like Results}\label{MethodologySubsect}
The classical Geiringer theorem (see \cite{GeirOrigion}) from population genetics tells us something about the ``limiting frequency of occurrence of certain individuals in a population" rather than referring to the limiting distribution of populations. In fact, the mathematical model of the classical Geiringer theorem in \cite{GeirOrigion} is entirely different from that of the finite-population Geiringer theorem described in the previous section. Nonetheless, the finite-population Markov chain model is much more suited when dealing with evolutionary algorithms since all the structures, including the search space and populations, in the computational setting are finite and the model in \cite{MitavRowGeirMain} and \cite{MitavRowGeirGenProgr} as well as in the current paper describes exactly what happens during a stochastic simulation. Knowing that some stochastic process $\{X_t\}_{t=0}^{\infty}$ on some equivalence class of populations $[\vec{x}]$ tends to the uniform distribution over the populations (i.e. $\forall \, \vec{y} \in [\vec{x}]$ we have $\lim_{t \rightarrow \infty} P(X_t = \vec{y}) = 1/|[\vec{x}]|$) it is often possible to deduce what we call Geiringer-like theorems which express the limiting frequency of occurrence of specific individuals and specific sets of individuals in terms of the information contained in a single representative of the equivalence class only (say, the initial population). Of course, we need to formulate precisely what the ``limiting frequency of occurrence" is.
\begin{defn}\label{SetPopulationCountDefn}
Consider a function $\mathcal{X}: \mathcal{P}(\Omega) \times \Omega^b \rightarrow \{0, \, 1, \, 2, \ldots, b\}$ where $\mathcal{P}(\Omega)$ denotes the power set of $\Omega$ (i.e. the set of all subsets of $\Omega$) and $\Omega^b$ is the set of all populations of size $b$, as usual, defined as follows: given a subset $S \subseteq \Omega$ and a population $\vec{x}=(x_1, \, x_2, \ldots, x_b) \in \Omega^b$, we define a function $\mathcal{X}(S, \vec{x}) = |\{i \, | \, 0 \leq i \leq b, \, x_i \in S\}|$ to be the number of individuals in the population $\vec{x}$ which belong to the subset $S$ (counting their multiplicities).
\end{defn}
\begin{ex}\label{countFunctIllustrEx}
Let's say $S = \{a\}$ is a singleton set, $b = 3$ and $\vec{x} = (u, \, v, \, u)$ where $u \neq v$. Then $\mathcal{X}(S, \vec{x}) = 2$ since $x_1 = x_3 = u \in S$ while $x_2 = v \notin S$.
\end{ex}
\begin{rem}\label{fixingAVariableRem}
Observe that if we fix a subset $S \subseteq \Omega$ and let the second argument in the function $\mathcal{X}$ vary, then we get a function of one variable $\mathcal{X}(S, \, \Box): \Omega^b \rightarrow \{0, \, 1, \, 2, \ldots, b\}$ defined naturally by plugging a population of size $b$ in place of the $\Box$.
\end{rem}
\begin{defn}\label{limitFreqOfOccurrenceRandVarDefn}
Choose a subset $S \subseteq \Omega$ an equivalence class $[\vec{x}]$ of populations of size $b$ and let $\{X_t\}_{t=0}^{\infty}$ be any stochastic process on $[\vec{x}]$ ($\vec{x}$ could be an initial population, for instance). It makes sense now to define a random variable $$\Phi(S, \, \vec{x}, \, t) = \frac{\sum_{i=0}^{t-1} \mathcal{X}(S, \, X_i)}{b \cdot t}.$$
\end{defn}
Clearly the random variable $\Phi(S, \, \vec{x}, \, t)$ counts the fraction of occurrence (or frequency of encountering) the individuals from the set $S$ before time $t$. In general $\lim_{t \rightarrow \infty}\Phi(S, \, \vec{x}, \, t)$ does not exist. However, under ``nice" circumstances described below everything works out rather well.
\begin{lem}\label{limitFreqOfOccurrLem1}
Suppose there is an ``attractor" probability distribution $\rho$ on the equivalence class $[\vec{x}]$ for the stochastic process $\{X_t\}_{t=0}^{\infty}$ in the sense that if $X_0 = x$ with probability $1$ then $\lim_{t \rightarrow \infty} P(X_t = \cdot) = \rho$ where $P(X_t = \cdot)$ denotes the probability distribution of the random variable $X_t$ which can be thought of in terms of a vector in $\mathbb{R}^{|[\vec{x}]|}$ so that the $\lim_{t \rightarrow \infty}$ is taken with respect to the $L_1$ norm, let's say\footnote{It is well-known that any two norms on finite dimensional real or complex vector spaces are equivalent so that the choice of the norm is irrelevant here}. Then $$\lim_{t \rightarrow \infty}\Phi(S, \, \vec{x}, \, t) = \frac{1}{b}E_{\rho}\left(\mathcal{X}(S, \, \Box)|_{[\vec{x}]}\right)$$ where $E_{\rho}$ denotes the expectation with respect to the probability distribution $\rho$ on $[\vec{x}]$, while $\mathcal{X}(S, \, \Box)|_{[\vec{x}]}$ is the restriction of the function $\mathcal{X}(S, \, \Box)$ introduced in remark~\ref{fixingAVariableRem} to the equivalence class $[\vec{x}]$.
\end{lem}
\begin{proof}[A sketch of the proof]
Consider a ``constant" stochastic process $Y_t$ where each random variable $Y_t$ is distributed according to $\rho$. By assumption $\|P(X_t = \cdot) - P(Y_t = \cdot)\|_{L_1} \rightarrow 0$ as $t \rightarrow \infty$. On the other hand, by the law of large numbers, $$E_{\rho}\left(\mathcal{X}(S, \, \Box)|_{[\vec{x}]}\right) = \lim_{t \rightarrow \infty}\frac{\sum_{i=0}^{t-1}\mathcal{X}(S, \, Y_i)}{t} \overset{\text{after routine } \epsilon \text{-details}}{=} \lim_{t \rightarrow \infty}\frac{\sum_{i=0}^{t-1}\mathcal{X}(S, \, X_i)}{t} =$$$$= b \cdot \lim_{t \rightarrow \infty}\frac{\sum_{i=0}^{t-1}\mathcal{X}(S, \, X_i)}{b \cdot t} = \lim_{t \rightarrow \infty} \Phi(S, \, \vec{x}, \, t)$$ so that the desired assertion follows after dividing both sides of the equation above by $b$.
\end{proof}
In our specific case, thanks to theorem~\ref{GeneralGeiringerThm}, the probability distribution $\rho$ in lemma~\ref{limitFreqOfOccurrLem1} is the uniform distribution on the equivalence class $[\vec{x}]$.

Notice that a random variable
\begin{equation}\label{sumOfIndicsEq}
\mathcal{X}(S, \, \Box) = \sum_{1=1}^b \mathcal{I}_i(S, \, \Box)
\end{equation}
where $\mathcal{I}_i(S, \, \Box)$ is the indicator function of the $i^th$ individual in the argument population with respect to the membership in the subset $S$. More explicitly, if we are given a population $\vec{x} = (x_1, \, x_2, \, \ldots x_b)^T$ then
\begin{equation}\label{indicatorFunctDefEq}
\mathcal{I}_i(S, \, \vec{x}) = \begin{cases}
1 & \text{if } x_i \in S\\
0 & \text{otherwise}.
\end{cases}
\end{equation}
Assume now that all transpositions of individuals within the same population are among the transformations in the family $\mathcal{S}$ (see the statement of theorem~\ref{GeneralGeiringerThm}). In other words, $\forall \, i<j$ the transformation $T_{i, \, j}$ sending a population $\vec{x} = (x_1, \, x_2, \ldots, x_{i-1}, \, x_i, \, x_{i+1}, \ldots, x_{j-1}, \, x_j, \, x_{j+1}, \ldots, x_b)^T$ into the population $T_{i, \, j}(\vec{x}) = (x_1, \, x_2, \ldots, x_{i-1}, \, x_j, \, x_{i+1}, \ldots, x_{j-1}, \, x_i, \, x_{j+1}, \ldots, x_b)^T$ has positive probability of being chosen. Notice that this is usually a very reasonable assumption since the order of individuals in a population should not matter in practical applications. Then we immediately deduce that any given population $\vec{y} \in [\vec{x}]$ if and only if the corresponding population $T_{i, \, j}(\vec{y})$ obtained by swapping the $i^{\text{th}}$ and the $j^{\text{th}}$ individuals in the population $\vec{y}$ is a member of $[\vec{x}]$. When $\rho$ is the uniform distribution (as in theorem~\ref{GeneralGeiringerThm}), this is equivalent to saying that all the indicator random variables $\mathcal{I}_i(S, \, \Box)$ defined in equation~\ref{indicatorFunctDefEq} above are identically distributed independently of the index $i$. In particular, they are all distributed as $\mathcal{I}_1(S, \, \Box)$. Using equation~\ref{sumOfIndicsEq} together with linearity of expectation, we now deduce that if $\pi$ denotes the uniform distribution on $[\vec{x}]$ then
$$E_{\pi}\left(\mathcal{X}(S, \, \Box)|_{[\vec{x}]}\right) = \sum_{1=1}^b E_{\pi}\left(\mathcal{I}_i(S, \, \Box)|_{[\vec{x}]}\right) = b \cdot E_{\pi}\left(\mathcal{I}_1(S, \, \Box)|_{[\vec{x}]}\right) =$$
\begin{equation}\label{indicatorExpEq}
= b \cdot \pi(\{\vec{y} \, | \, \vec{y} = (y_1, \, y_2, \ldots, y_b)^T \in [\vec{x}] \text{ and } y_1 \in S\}) = b \cdot \frac{|\mathcal{V}(\vec{x}, \, S)|}{|[\vec{x}]|}.
\end{equation}
where
\begin{equation}\label{firstIndivPopEqDef}
\mathcal{V}(\vec{x}, \, S) = \{\vec{y} \, | \, \vec{y} = (y_1, \, y_2, \ldots, y_b)^T \in [\vec{x}] \text{ and } y_1 \in S\}
\end{equation}
is the subset of $[\vec{x}]$ consisting solely of populations in $[\vec{x}]$ the first individuals of which are members of the subset $S \subseteq \Omega$. combining equation~\ref{indicatorExpEq} with the conclusion of lemma~\ref{limitFreqOfOccurrLem1} immediately produces the following very useful fact.
\begin{lem}\label{mainGeiringerLikeLemma}
Under exactly the same setting and assumptions as in theorem~\ref{GeneralGeiringerThm} together with an additional assumption that all the ``swap" transformations defined and discussed in the paragraph following equation~\ref{indicatorFunctDefEq} are members of the subfamily $\mathcal{S}$ of the family $\mathcal{F}$ of recombination transformations, it is true that $\forall \, S \subseteq \Omega$ we have
$$\lim_{t \rightarrow \infty}\Phi(S, \, \vec{x}, \, t) = \frac{|\mathcal{V}(\vec{x}, \, S)|}{|[\vec{x}]|}$$
where the set $\mathcal{V}(\vec{x}, \, S)$ is defined in \ref{firstIndivPopEqDef}.
\end{lem}
Lemma~\ref{mainGeiringerLikeLemma} allows us to derive Geiringer-like theorems in a rather straightforward fashion for several classes of evolutionary algorithms via the following simple strategy: suppose we are given a subset $S \subseteq \Omega$. According to lemma~\ref{mainGeiringerLikeLemma}, all we have to do to compute the desired limiting frequency of occurrence of a certain subset $S \subseteq \Omega$ is to calculate the ratio $\frac{|\mathcal{V}(\vec{x}, \, S)|}{|[\vec{x}]|}$. For some subsets of the search space such a ratio is quite obvious, yet for others it may be combinatorially non-achievable. In evolutionary computation, it is often possible to define an appropriate notion of schemata (this is precisely what we have done in section~\ref{schemataSubsectShallow} for the case of MCT) which has, intuitively speaking, a ``product-like flavor" that allows us to exploit the following observation: suppose we can find a sequence of subsets $S_1 \supseteq S_2 \supseteq \ldots \supseteq S_{n-1} \supseteq S_n = S$. We can then write
$$\lim_{t \rightarrow \infty}\Phi(S, \, \vec{x}, \, t) = \frac{|\mathcal{V}(\vec{x}, \, S)|}{|[\vec{x}]|} = \frac{|\mathcal{V}(\vec{x}, \, S)|}{|\mathcal{V}(\vec{x}, \, S_{n-1})|} \cdot \frac{|\mathcal{V}(\vec{x}, \, S_{n-1})|}{|\mathcal{V}(\vec{x}, \, S_{n-2})|} \cdot \ldots \cdot \frac{|\mathcal{V}(\vec{x}, \, S_1)|}{|[\vec{x}]|} =$$
\begin{equation}\label{computingTheRatioEq}
\overset{\text{by lemmas~\ref{mainGeiringerLikeLemma} and \ref{limitFreqOfOccurrLem1}}}{=} \frac{1}{b}E_{\rho}\left(\mathcal{X}(S_1, \, \Box)|_{[\vec{x}]}\right) \cdot \prod_{k=1}^{n-1}\frac{|\mathcal{V}(\vec{x}, \, S_{k+1})|}{|\mathcal{V}(\vec{x}, \, S_k)|}
\end{equation}
The idea is that the individual ratios in the right hand side of equation~\ref{computingTheRatioEq} may be quite simple to compute as happens to be the case when deriving finite population Geiringer-like theorems for GP with homologous crossover (see \cite{MitavRowGeirMain} and \cite{MitavRowGeirGenProgr}). When deriving the finite population version Geiringer-like theorem with non-homologous recombination in the limit of large population size, rather than computing the ratios in equation~\ref{computingTheRatioEq}, we will instead estimate each one of them from above and from below exploiting the main Geiringer theorem (theorem~\ref{GeneralGeiringerThm}) together with the methodology for estimating the stationary distributions of Markov chains based on the lumping quotient construction appearing in (\cite{LumpQuotMitavRoweWright}, \cite{LumpQuotGPEM} and \cite{LumpQuotMitavCannings}). All of the necessary apparatus and one enhanced lemma will be summarized and presented in the next subsection for the sake of completeness.
\subsection{Lumping Quotients of Markov Chains and Markov Inequality}\label{LumpQuotSubsect}
Throughout the current subsection we shall be dealing with a Markov chain $\mathcal{M}$ (not necessarily irreducible) over a finite state space
$\mathcal{X}$. $\{p_{\mathbf{x}\rightarrow \mathbf{y}}\}$ denotes
the Markov transition matrix with the convention that
$p_{\mathbf{x}\rightarrow \mathbf{y}}$ is the probability of getting
$\mathbf{y}$ in the next stage given $\mathbf{x}$. Let $\pi$ denote a stationary distribution of the Markov chain $M$ (here we will assume that at least one stationary distribution does exist). Furthermore we will assume that the stationary distribution $\pi$ has the property that $\forall \, x \in \mathcal{X}$ $\pi(x) \neq 0$. Suppose we are
given an equivalence relation $\sim$ partitioning the state space
$\mathcal{X}$. The aim of the current section is to construct
a Markov chain over the equivalence classes under
$\sim$ (i.e. over the set $\mathcal{X}/\sim$) whose stationary
distribution is compatible with the distribution $\pi$ and then to exploit the constructed lumped quotient chain to estimate certain ratios of the stationary distribution values. In fact, this methodology has been successfully used to establish some properties of the stationary distributions of the irreducible Markov chains modeling a wide class of evolutionary algorithms (see \cite{LumpQuotMitavRoweWright}, \cite{LumpQuotGPEM} and \cite{LumpQuotMitavCannings}).
\begin{defn}\label{quatDef}
Given a Markov chain $\mathcal{M}$ over a finite state
space $\mathcal{X}$ determined by the transition matrix
$\{p_{\mathbf{x}\rightarrow \mathbf{y}}\}$, an equivalence
relation $\sim$ on $\mathcal{X}$, and a stationary distribution  $\pi$ of the Markov chain $\mathcal{M}$ satisfying the property that $\forall \, x \in \mathcal{X}$ $\pi(x) \neq 0$, define
the \textit{quotient} Markov chain $\mathcal{M}/\sim$ over the state
space $\mathcal{X}/\sim$ of equivalence classes via $\sim$ to be
determined by the transition matrix
$\{\tilde{p}_{\mathcal{U}\rightarrow \, \mathcal{V}}\}_{\mathcal{U},
\, \mathcal{V} \in \mathcal{X}/\sim}$ given as
$$\tilde{p}_{\mathcal{U}\rightarrow \,
\mathcal{V}}=\frac{1}{\pi(\mathcal{U})}\sum_{\mathbf{x} \in
\mathcal{U}}\pi(\mathbf{x}) \cdot p_{\mathbf{x}\rightarrow
\mathcal{V}}=\frac{1}{\pi(\mathcal{U})}\sum_{\mathbf{x} \in
\mathcal{U}}\sum_{\mathbf{y} \in \mathcal{V}}\pi(\mathbf{x}) \cdot
p_{\mathbf{x}\rightarrow \mathbf{y}}.$$ Here
$p_{\mathbf{x}\rightarrow \mathcal{V}}$ denotes the transition
probability of getting somewhere inside of $\mathcal{V}$ given
$\mathbf{x}$. Since $\mathcal{V}=\bigcup_{y \in \mathcal{V}}\{y\}$
it follows that $p_{\mathbf{x}\rightarrow \mathcal{V}}=\sum_{y \in
\mathcal{V}} p_{\mathbf{x}\rightarrow \mathbf{y}}$ and hence the
equation above holds.
\end{defn}
Intuitively, the quotient Markov chain $\mathcal{M}/\sim$ is
obtained by running the original chain $\mathcal{M}$ starting with
the stationary distribution $\pi$ and computing the transition
probabilities of the assiciated stochastic process conditioned with respect to the stationary input. Thereby, the following fact should not be a surprise:
\begin{thm}\label{CompatDistr}
Let $\pi$ denote a stationary distribution of a
Markov chain $\mathcal{M}$ determined by the transition matrix
$\{p_{\mathbf{x}\rightarrow \mathbf{y}}\}_{\mathbf{x}, \, \mathbf{y}
\in \mathcal{X}}$ and having the property that $\forall \, x \in \mathcal{X}$ $\pi(x) \neq 0$. Suppose we are given an equivalence relation
$\sim$ partitioning the state space $\mathcal{X}$. Then the probability distribution $\tilde{\pi}$ defined as $\tilde{\pi}(\{\mathcal{O}\})=\pi(\mathcal{O})$ is a stationary distribution of the quotient Markov chain $\mathcal{M}/\sim$ assigning nonzero probability to every state (i.e. to every equivalence class under $\sim$).
\end{thm}
\begin{pf}
This fact can be verified by direct computation. Indeed, we obtain $$\sum_{\mathcal{O} \in
\mathcal{X}/\sim}\tilde{\pi}(\{\mathcal{O}\}) \cdot
\tilde{p}_{\mathcal{O}\rightarrow \,
\mathcal{U}}=\sum_{\mathcal{O} \in \mathcal{X}/\sim}
\pi(\mathcal{O}) \cdot \frac{1}{\pi(\mathcal{O})}\sum_{\mathbf{x}
\in \mathcal{O}}\sum_{\mathbf{z} \in \mathcal{U}} \pi(\mathbf{x})
\cdot p_{\mathbf{x}\rightarrow \mathbf{z}}=$$$$=\sum_{\mathbf{x} \in
\mathcal{X}}\sum_{\mathbf{z} \in \mathcal{U}} \pi(\mathbf{x}) \cdot
p_{\mathbf{x}\rightarrow \mathbf{z}}=\sum_{\mathbf{z} \in
\mathcal{U}} \sum_{\mathbf{x} \in \mathcal{X}} \pi(\mathbf{x}) \cdot
p_{\mathbf{x}\rightarrow \mathbf{z}} \overset{\text{by stationarity
of } \pi}{=}$$$$= \sum_{\mathbf{z} \in \mathcal{U}}
\pi(\mathbf{z})=\pi(\mathcal{U})=\tilde{\pi}(\{\mathcal{U}\}).$$
This establishes the stationarity of $\tilde{\pi}$ and
theorem~\ref{CompatDistr} now follows.
\end{pf}
Although theorem~\ref{CompatDistr} is rather elementary it allows us
to deduce interesting and insightful results (see \cite{LumpQuotMitavRoweWright}, \cite{LumpQuotGPEM} and \cite{LumpQuotMitavCannings}) via the observations
presented below. To state these results it is convenient to
generalize the notion of transition probabilities in the following
manner (which is coherent with definition~\ref{quatDef}):
\begin{defn}\label{transProbDefn}
Given a Markov chain $\mathcal{M}$ with state space
$\mathcal{X}$ and a stationary distribution $\pi$, for any two subsets $A$ and $B \subseteq
\mathcal{X}$, we define $p_{A \rightarrow B}=\sum_{a \in
A}\frac{\pi(a)}{\pi(A)}p_{a \rightarrow B}$ where $p_{a \rightarrow
B} = \sum_{b \in B}p_{a \rightarrow b}$.
\end{defn}
\begin{rem}\label{cohDefn}
It is worth emphasizing that in case when $B=A$ or $A \cap B =
\emptyset$, the transition probabilities $p_{A \rightarrow B}$ are
precisely the transition probabilities of various quotient Markov
chains with states which have $A$ and $B$ as their states according
to definition~\ref{quatDef}. In particular, if we consider the
quotient Markov chain comprised of the states, $A$ and $A^c$ where
$A^c$ denotes the complement of $A$, we have $1-p_{A \rightarrow
A}=p_{A \rightarrow A^c}$.
\end{rem}
In the current paper we will use a lumping quotient chain consisting of only $2$
equivalence classes, $A$ and $B = A^c$ (i.e. the complement of $A$ in the state space $\mathcal{X}$).
For a $2$ by $2$ Markov transition matrix we easily see that if $\pi$ denotes the unique stationary distribution of the original Markov chain $\mathcal{M}$ then, thanks to theorem~\ref{CompatDistr}, we have $\pi(A)p_{A \rightarrow A} + \pi(B)p_{B \rightarrow A} = \pi(A)$ so that $\pi(B)p_{B \rightarrow A} = \pi(A)(1 - p_{A \rightarrow A}) = \pi(A)p_{A \rightarrow B}$  and, if neither $A$ nor $B$ is empty, we have
\begin{equation}\label{stationaryDistribRatioEq}
\frac{\pi(A)}{\pi(B)} = \frac{p_{B \rightarrow A}}{p_{A \rightarrow B}}
\end{equation}
Equation~\ref{stationaryDistribRatioEq}, tells us that in order to estimate the ratio of the stationary distribution values of the Markov chain $\mathcal{M}$ on a pair of complementary subsets of the state space $A$ and $B = A^c$, it is sufficient to estimate the ratio of the generalized transition probabilities $p_{B \rightarrow A}$ and $p_{A \rightarrow B}$. Although these transition probabilities do depend on the stationary distribution itself, it is sometimes possible to estimate them using a convexity-based bound appearing in (\cite{LumpQuotMitavRoweWright}, \cite{LumpQuotGPEM} and \cite{LumpQuotMitavCannings}). For the purpose of the present work we need to introduce a mild generalization of this bound appearing below:
\begin{lem}\label{estimationOfTransLem}
Suppose, as in definition~\ref{transProbDefn}, $A$ and $B \subseteq \mathcal{X}$ and $U \subseteq \mathcal{X}$ such that $$\frac{\pi(U \cap A)}{\pi(A)} \leq \epsilon < 1.$$ Suppose further that for some constant $\kappa$ with $0 \leq \kappa \leq 1$ the following is true: $\forall \, a \in A \cap U^c$ we have $p_{a \rightarrow B} \leq \kappa$. Then we have $p_{A \rightarrow B} \leq (1 - \epsilon)\kappa + \epsilon$. Dually, assume that for a constant $\lambda$ with $0 \leq \lambda \leq 1$ it is true that $\forall \, a \in A \cap U^c$ we have $p_{a \rightarrow B} \geq \lambda$. Then $p_{A \rightarrow B} \geq (1 - \epsilon)\lambda$.
\end{lem}
\begin{proof}
Indeed, we have
\begin{equation}\label{estimRatioLemmaPfEq1}
p_{A \rightarrow B} = \sum_{a \in
A}\frac{\pi(a)}{\pi(A)}p_{a \rightarrow B} = \sum_{a \in
A \cap U^c}\frac{\pi(a)}{\pi(A)}p_{a \rightarrow B}+\sum_{a \in
A \cap U}\frac{\pi(a)}{\pi(A)}p_{a \rightarrow B}.
\end{equation}
Notice that
$$\sum_{a \in A \cap U^c}\frac{\pi(a)}{\pi(A)} = \frac{\pi(A \cap U^c)}{\pi(A)} = 1 - \frac{\pi(U \cap A)}{\pi(A)} \geq 1 - \epsilon$$
\begin{equation}\label{etimRatioLemmaPfEq2}
\text{while} \; \; \; \; 0 \leq \sum_{a \in
A \cap U}\frac{\pi(a)}{\pi(A)} = \frac{\pi(A \cap U)}{\pi(A)}<\epsilon
\end{equation}
The desired inequalities now follow when we plug in the bounds in the assumptions into equation~\ref{estimRatioLemmaPfEq1} and then use the inequalities in equation~\ref{etimRatioLemmaPfEq2} together with the fact that probabilities are always between $0$ and $1$.
\end{proof}
In a special case when $U = \emptyset$ lemma~\ref{estimationOfTransLem} entails the following.
\begin{cor}\label{EstimOfTransCor}
Given any two subsets $A$ and $B \subseteq \mathcal{X}$, if for some constant $\kappa$ with $0 \leq \kappa \leq 1$ it is true that $\forall \, a \in A$ we have $p_{a \rightarrow B} \leq \kappa$ then $p_{A \rightarrow B} \leq \kappa$. Dually, if for some constant $\lambda$ with $0 \leq \lambda \leq 1$ it is true that $\forall \, a \in A$ we have $p_{a \rightarrow B} \geq \lambda$ then $p_{A \rightarrow B} \geq \lambda$. Consequently, if for some constant $\gamma$ it happens that $\forall \, a \in A$ we have $p_{a \rightarrow B} = \gamma$ then $p_{A \rightarrow B} = \gamma$.
\end{cor}
Combining equation~\ref{stationaryDistribRatioEq} with lemma~\ref{estimationOfTransLem} readily gives us the following.
\begin{lem}\label{estimStationaryRatiosLem}
Suppose $A$ and $B \subseteq \mathcal{X}$ is a complementary pair of subsets (i.e. $A \cap B = \emptyset$ and $A \cup B = \mathcal{X}$). Suppose further that $U \subseteq \mathcal{X}$ is such that $$\frac{\pi(U \cap A)}{\pi(A)} < \epsilon < 1 \text{ and }\frac{\pi(U \cap B)}{\pi(B)} < \delta < 1.$$ Assume now that we find constants $\lambda_1$, $\lambda_2$, $\kappa_1$ and $\kappa_2$ such that $\forall \, b \in U^c \cap B$ we have $\lambda_1 \leq p_{b \rightarrow A} \leq \kappa_1$ and $\forall \, a \in U^c \cap A$ we have $\lambda_2 \leq p_{a \rightarrow B} \leq \kappa_2$. Then we have $$\frac{(1 - \delta)\lambda_1}{(1 - \epsilon)\kappa_2 + \epsilon} \leq \frac{\pi(A)}{\pi(B)} \leq \frac{(1 - \delta)\kappa_1 + \delta}{(1 - \epsilon)\lambda_2}$$
\end{lem}
In order to apply lemma~\ref{estimStationaryRatiosLem} effectively we need to know that both, $\frac{\pi(U \cap A)}{\pi(A)}$ and $\frac{\pi(U \cap B)}{\pi(B)}$ are small. As we shall see in the next subsection, the inductive hypothesis will imply that at least one of these ratios is small. The following simple lemma will allow us to deduce that the remaining ratio is also small as long as a certain ratio of generalized transition probabilities is bounded below.
\begin{lem}\label{littleTechnVarifyLem}
Suppose $A$ and $B \subseteq \mathcal{X}$ with $A \cap B = \emptyset$ (notice that we do not require $A \cup B = \mathcal{X}$). Then $$\pi(A) \geq \pi(B) \cdot \frac{p_{B \rightarrow A}}{p_{A \rightarrow A^c}}.$$
\end{lem}
\begin{proof}
Let $C = \mathcal{X} \cap (A \cup B)^c$. Consider the lumped Markov chain on the state space $\{A, \, B, \, C\}$. Since $\pi$ is the stationary distribution of the Markov chain $\mathcal{M}$, by theorem~\ref{CompatDistr} (see also definition~\ref{transProbDefn} and remark~\ref{cohDefn}) we have $$\pi(A) = \pi(B)p_{B \rightarrow A} + \pi(A)p_{A \rightarrow A} + \pi(C)p_{C \rightarrow A}$$ so that $$(1-p_{A \rightarrow A})\pi(A) = \pi(B)p_{B \rightarrow A} + \pi(C)p_{C \rightarrow A} \geq \pi(B)p_{B \rightarrow A}$$ since probabilities are nonnegative. The desired conclusion now follows when dividing both sides of the inequality above by $1-p_{A \rightarrow A} = p_{A \rightarrow A^c}$.
\end{proof}
Finally, there is another very simple and general classical inequality that will be elegantly exploited in the next section to set the stage for the application of lemma~\ref{estimStationaryRatiosLem} allowing us to avoid unpleasant combinatorial complications.
\begin{lem}[Markov Inequality]\label{markovInequality}
Suppose $H$ is a non-negative valued random variable on a probability space $\Omega$ with probability measure $\emph{Pr}$. Then $\forall \, \lambda > 0$ we have $$0 < \emph{Pr}(H > \lambda \cdot E(H)) \leq \frac{1}{\lambda} \rightarrow 0 \text{ as } \lambda \rightarrow \infty.$$
\end{lem}
\begin{proof}
By definition of expectation we have $$E(H) = \int_{\Omega}H d\text{Pr} \overset{\text{by positivity of }H}{\geq} \int_{H > \lambda \cdot E(H)}H d\text{Pr} \geq \text{Pr}(H > \lambda \cdot E(H)) \cdot (\lambda \cdot E(H)).$$ Now, if $\text{Pr}(H>0) = 0$ then $H = 0$ almost surely so that $E(H) = 0$ and $$\text{Pr}(H > \lambda \cdot E(H))=\text{Pr}(H > 0) = 0 < \frac{1}{\lambda}.$$ Otherwise, $\text{Pr}(H>0)>0 \Longrightarrow E(H) = \int_{\Omega}H d\text{Pr} > 0$ and the desired inequality follows when dividing both sides of the equation above by $\lambda \cdot E(H)$.
\end{proof}
We end this section with a very well-known elementary fact about Markov chains having symmetric transition matrices that will also be used in the proof of theorem~\ref{GeiringerLikeThmForMCTMain}.
\begin{prop}\label{symmetricTransMatrixLem}
Let $M$ be any Markov chain determined by a symmetric transition matrix. Then the uniform distribution is a stationary distribution of the Markov chain $M$ (notice that $M$ is not assumed to be irreducible).
\end{prop}
\begin{proof}
The reader may easily see that the Markov transition matrix is doubly-stochastic or verify that the uniform distribution is stationary directly from the detailed balance equations.\footnote{This is also a particular case of the well-known reversibility property of Markov chains.}
\end{proof}
\subsection{Deriving the Geiringer-like Theorem (Theorem~\ref{GeiringerLikeThmForMCTMain}) for the MCT algorithm}\label{specificThmDerivationSubsect}
We now recall the setting of section~\ref{settingSect}. At first we will prove the theorem for a mildly extended family of recombination transformations $\widetilde{F}$ where in addition to the transformations in definition~\ref{recombActOnPopsDef} $\widetilde{F}$ also contains
all the transpositions (or swaps) of the rollouts in a population and these are selected with positive probability (a detailed description appears in paragraph following equation~\ref{indicatorFunctDefEq}). Since every transposition of rollouts is a bijection on the set of all
populations, theorem~\ref{GeneralGeiringerThm} still applies, except that the equivalence classes will be enlarged by a factor of $(b \cdot m)!$ i.e. $[P_m]_{\widetilde{\mathcal{F}}} = (b \cdot m)! \cdot [P_m]_{\mathcal{F}}$ (this is so because every permutation is a composition of transpositions). Thanks to the assumption we will be in a position to apply the tools based on lemma~\ref{mainGeiringerLikeLemma}, namely equation~\ref{computingTheRatioEq}. This assumption will be dropped at the end via apparent symmetry considerations. Indeed, any permutation $\pi$ of the rollouts in a population $Q \in [P_m]$ naturally commutes with all the recombination transformations in definition~\ref{RecombStagePopTransDefn} thereby providing a family of bijections between the equivalence class $[P_m]_{\mathcal{F}}$ and each of the $(b \cdot m)!$ disjoint pieces comprising the partition of the equivalence class $[P_m]_{\widetilde{\mathcal{F}}}$. Furthermore, permutations preserve the multisets of rollouts within a population so that the frequencies of occurrence of various subsets in the corresponding pieces will be preserved and, thereby, the conclusion of theorem~\ref{GeiringerLikeThmForMCTMain} with the family of recombination transformations $\mathcal{F}$ replaced by $\widetilde{\mathcal{F}}$ will be exactly the same.

Recall the schema $$h = (\alpha, \, i_1, \, i_2, \ldots, i_{k-1}, \, x_k)$$ of height $k-1 \geq 0$ in the statement of theorem~\ref{GeiringerLikeThmForMCTMain}. Notice that thanks to proposition~\ref{schemaPartOrderHollandShallow} we can write the given schema $h$ as
\begin{equation}\label{schemaDecompEqProof}
h = h_k \subseteq h_{k-1} \subseteq h_{k-2} \subseteq \ldots \subseteq h_2 \subseteq h_1
\end{equation}
where $h_1 = (\alpha, \, i_1, \, \#)$ and, in general, when $1 \leq j < k$ $$h_j = (\alpha, \, i_1, \, i_2, \, \ldots, \, i_j, \, \#)$$ are Holland schemata. Thanks to equation~\ref{computingTheRatioEq}, $\forall \, m \in \mathbb{N}$ we have
\begin{equation}\label{compRatioForMCTEq}
\lim_{t \rightarrow \infty}\Phi(h, \, P_m, \, t) = \frac{1}{b}E_{\rho}\left(\mathcal{X}(h_1, \, \Box)|_{[P_m]_{\widetilde{\mathcal{F}}}}\right) \cdot \prod_{q=1}^{k-1}\frac{|\mathcal{V}(P_m, \, h_{q+1})|}{|\mathcal{V}(P_m, \, h_q)|}
\end{equation} and, taking the limit as $m \rightarrow \infty$,
$$\lim_{m \rightarrow \infty}\lim_{t \rightarrow \infty}\Phi(h, \, P_m, \, t) = $$
\begin{equation}\label{compRatioForMCTWithLimitEq}
=\frac{1}{b}\lim_{m \rightarrow \infty}E_{\rho}\left(\mathcal{X}(h_1, \, \Box)|_{[P_m]_{\widetilde{\mathcal{F}}}}\right) \cdot \prod_{q=1}^{k-1}\lim_{m \rightarrow \infty}\frac{|\mathcal{V}(P_m, \, h_{q+1})|}{|\mathcal{V}(P_m, \, h_q)|}
\end{equation}
where $\rho$ is the uniform distribution on $[P]_m$. First of all, notice that $\forall, \, m \in \mathbb{N}$ the random variable $\mathcal{X}(h_1, \, \Box)|_{[P_m]_{\widetilde{\mathcal{F}}}}$ is a constant function which is equal to $\text{Order}(\alpha \downarrow i_1, \, P_m) = \text{Order}(\alpha \downarrow i_1, \, P)$ (see remark~\ref{equivClassIndepRem} and proposition~\ref{popRatioFacts}). It follows trivially then that $E_{\pi}\left(\mathcal{X}(h_1, \, \Box)|_{[P_m]_{\widetilde{\mathcal{F}}}}\right) = \text{Order}(\alpha \downarrow i_1, \, P)$ giving us the first ratio factor in the right hand side of equation~\ref{GeiringerThmMainEq}. In particular, when $h = h_1$ is a schema of height $0$ ending with a $\#$, there is no need to take the limit as $m \rightarrow \infty$ regardless of whether or not the population $P$ is homologous. To deal with the remaining ratios in the general case, when the population $P$ is not necessarily homologous, we will exploit the classical and elementary Markov inequality (lemma~\ref{markovInequality} in a rather elegant manner) to set up the stage for the application of lemmas~\ref{estimationOfTransLem} and \ref{estimStationaryRatiosLem} in the following manner.

Consider the random variable $H_i: [P_m] \rightarrow \mathbb{N}$ where $[P_m]$ is equipped with the uniform probability measure $\rho$, measuring the height of the $i^{th}$ rollout in the population $Q \in [P_m]$. In other words, $$H_i(Q) = \text{the height of the } i^{\text{th}} \text{ rollout in the population }Q.$$ Notice that $\forall, \, i$ and $j$ with $1 \leq i \leq j \leq b \cdot m$ the random variables $H_i$ and $H_j$ are identically distributed (indeed, thanks to theorem~\ref{GeneralGeiringerThm}, the swap of the rollouts $i$ and $j$ in the population $P_m$ is an isomorphism of the probability space $[P_m]$ with itself, call it $\tau$, such that $H_i \circ \tau = H_j$ and vice versa). In particular, these random variables have the same expectation. Thanks to remark~\ref{averageHightRem} and proposition~\ref{popRatioFacts}, we deduce that
$$E(H_1) = \frac{\sum_{i = 1}^{b \cdot m}E(H_i)}{b \cdot m} = \frac{E \left(\sum_{i = 1}^{b \cdot m}H_i\right)}{b \cdot m}=$$
\begin{equation}\label{expSumOfRollsEqn}
= \frac{\text{Total}(P_m)}{b \cdot m} = \frac{m \cdot \text{Total}(P)}{b \cdot m} = \frac{\text{Total}(P)}{b}.
\end{equation}
Notice that the right hand side of equation~\ref{expSumOfRollsEqn} does not depend on $m$. In other words, $\forall \, m \in \mathbb{N}$ the expected height of the first rollout in the population $P_m$ is the same and is equal to $\frac{\text{Total}(P)}{b}$. At the same time, according to proposition~\ref{popRatioFacts}, the functions
\begin{equation}\label{increasingFunctionEq}
\text{Order}(\alpha \downarrow j, \, P_m) \rightarrow \infty \text{ and } \text{Order}(i \downarrow j, \, P_m) \rightarrow \infty \text{ as } m \rightarrow \infty.
\end{equation}
The above observation opens the door for the application of Markov inequality that will, in turn, allow us to exploit lemma~\ref{estimStationaryRatiosLem} with the aim of estimating the desired ratios involved in equation~\ref{computingTheRatioEq} and then showing that the upper and the lower bounds on these fractions converge to the corresponding ratios involved in the right hand side of equation~\ref{GeiringerThmMainEq} in the conclusion of the statement of theorem~\ref{GeiringerLikeThmForMCTMain}. We now proceed in detail. Let $\delta > 0$ be an arbitrary small number (informally speaking, $\delta \ll 1$). Choose $M \in \mathbb{N}$ large enough so that $$\delta^2 \cdot M > \frac{\text{Total}(P)}{b} = E(H_1)$$ (see equation~\ref{expSumOfRollsEqn}). For $m > M$ let
\begin{equation}\label{smallSetEqDefn}
U_m^{\delta} = \{Q \, | \, Q \in [P_m] \text{ and } H_1(Q) > \delta \cdot m\}.
\end{equation}
and observe that the Markov inequality (lemma~\ref{markovInequality}) tells us that
$$\rho(U_m^{\delta}) = \rho(\{Q \, | \, H_1(Q) > \delta \cdot m\}) = \rho\left(H_1 > \frac{1}{\delta} \cdot (\delta^2 \cdot m)\right) \overset{\text{since }m > M}{\leq}$$ \begin{equation}\label{smallSetPropEq}
\overset{\text{and by definition of }U_m^{\delta} \text{ in equation~\ref{smallSetEqDefn}}}{\leq} \rho\left(H_1 > \frac{1}{\delta} \cdot E(H_1)\right) \overset{\text{by Markov inequality}}{\leq} 1 / \frac{1}{\delta} = \delta
\end{equation}
where $\rho$ denotes the uniform probability distribution on the set $[P_m]$.

As the reader probably anticipates by now, our aim is to show that each of the ratios of the form $$\lim_{m \rightarrow \infty}\frac{|\mathcal{V}(P_m, \, h_{q+1})|}{|\mathcal{V}(P_m, \, h_q)|} = \frac{\text{Order}(i_q \downarrow i_{q+1}, \, P)}{\sum_{j \in i_q \downarrow (P)}\text{Order}(i_q \downarrow j, \, P)+ i_q \downarrow_\Sigma (P)}$$ so that equation~\ref{GeiringerThmMainEq} in the conclusion of theorem~\ref{GeiringerLikeThmForMCTMain} would follow from equation~\ref{compRatioForMCTEq} when taking the limit of both sides as $m \rightarrow \infty$. First of all, let us take care of the ``trivial extremes" when for some $q$ with $1 \leq q \leq k-1$ we have either ($\text{Order}(i_{q-1} \downarrow i_q, \, P) = 0$) or ($\forall \, j \neq i_q$ we have $\text{Order}(i_{q-1} \downarrow j, \, P) = 0$ and $i_{q-1} \downarrow_\Sigma (P) =0$)) or (($x_k \in \Sigma$) and (either $i_{k-1} \downarrow_\Sigma (P) =0$ or $x_k \notin i_{k-1} \downarrow (P)$) or ($\forall \, j \in \mathbb{N}$ we have $\text{Order}(i_{k-1} \downarrow j, \, P) = 0$ and $x_k$ is the only terminal label member of the set $i_{k-1} \downarrow P$ i.e. $i_{k-1} \downarrow P \cap \Sigma = \{x_k\}$)) or ($x_k = \#)$. According to proposition~\ref{popRatioFacts}, the statement above holds for a population $P$ if and only if $\forall \, m \in \mathbb{N}$ it holds when the population $P$ is replaced with $P_m$. In the case when either $\text{Order}(i_{q-1} \downarrow i_q, \, P_m) = 0$ or $i_{k-1} \downarrow_\Sigma (P_m) =0$ or $x_k \notin i_{k-1} \downarrow (P)$, no individual fitting the schema $h$ is present in any population $Q \in [P_m]$ so that $\forall \, m$ and $t \in \mathbb{N}$ we have $\Phi(P_m, \, h, \, t) = 0$. Thereby the left hand side of equation~\ref{GeiringerThmMainEq} is trivially $0$. The right hand side is $0$ as well in this case since the numerator of one of the fractions in the product is $0$ (see the convention remark in the statement of theorem~\ref{GeiringerLikeThmForMCTMain}). This finishes the verification of one trivial extreme case. Suppose now for some index $q$ it is the case that $\forall \, j \neq i_q$ we have $\text{Order}(i_{q-1} \downarrow j, \, P) = 0$ and $i_{q-1} \downarrow_\Sigma (P) =0$. In this case we observe that any individual occurring in a population $Q \in [P_m]$ which fits the schema $h_{q-1}$, also fits the schema $h_q$. In particular, the sets $\mathcal{V}(P_m, \, h_{q+1})$ and $\mathcal{V}(P_m, \, h_q)$ are equal and we trivially have $\forall \, m \in \mathbb{N}$ $\frac{|\mathcal{V}(P_m, \, h_{q+1})|}{|\mathcal{V}(P_m, \, h_q)|} = 1$. Of course, the corresponding ratio $$\frac{\text{Order}(i_q \downarrow i_{q+1}, \, P)}{\sum_{j \in i_q \downarrow (P)}\text{Order}(i_q \downarrow j, \, P)+ i_q \downarrow_\Sigma (P)} = 1$$ as well since $\text{Order}(i_q \downarrow j, \, P)$ is the only nonzero contributing summand in the denominator. The last factor ratio is supposed to coincide with the ratio $\frac{|\mathcal{V}(P_m, \, h)|}{|\mathcal{V}(P_m, \, h_{k-1})|}$. This ratio is either $0$ or $1$ in the extreme cases and verifying the validity of equation~\ref{GeiringerThmMainEq} is entirely analogous to the above.
We now move on to the interesting case when none of the trivial extremes above happen. For schemata $x$ and $y$ we write $x \setminus y = S_x \cap (S_y)^c$ (see definition~\ref{schemaForMCTPopDefHolland}) to denote the set of rollouts fitting the schema $x$ and not fitting the schema $y$. Rather than estimating or, in case of homologous population $P$, evaluating exactly the ratios of the form $\frac{|\mathcal{V}(P_m, \, h_{q+1})|}{|\mathcal{V}(P_m, \, h_q)|}$ we estimate and, in case of homologous recombination, evaluate the ratios of the form $\frac{|\mathcal{V}(P_m, \, h_{q+1})|}{|\mathcal{V}(P_m, \, h_q \setminus h_{q+1})|}$ since these are more convenient to tackle using the tools in section~\ref{LumpQuotSubsect}. The following very simple fact demonstrates the connection between the two:
\begin{lem}\label{bridgeRatioLem}
Suppose that $\forall \, m$ whenever $1 \leq q < k-1$
$$\frac{|\mathcal{V}(P_m, \, h_{q+1})|}{|\mathcal{V}(P_m, \, h_q \setminus h_{q+1})|} = \frac{\emph{Order}(i_q \downarrow i_{q+1}, \, P)}{\sum_{j \in i_q \downarrow (P) \emph{ and }j \neq i_{q+1}}\emph{Order}(i_q \downarrow j, \, P)+ i_q \downarrow_\Sigma (P)}$$ and neither the numerator nor the denominator of any of the fractions is $0$. Then $$\frac{|\mathcal{V}(P_m, \, h_{q+1})|}{|\mathcal{V}(P_m, \, h_q)|} = \frac{\emph{Order}(i_q \downarrow i_{q+1}, \, P)}{\sum_{j \in i_q \downarrow (P)}\emph{Order}(i_q \downarrow j, \, P)+ i_q \downarrow_\Sigma (P)}.$$ Likewise, if $$\lim_{m \rightarrow \infty}\frac{|\mathcal{V}(P_m, \, h_{q+1})|}{|\mathcal{V}(P_m, \, h_q \setminus h_{q+1})|} = \frac{\emph{Order}(i_q \downarrow i_{q+1}, \, P)}{\sum_{j \in i_q \downarrow (P) \emph{ and }j \neq i_{q+1}}\emph{Order}(i_q \downarrow j, \, P)+ i_q \downarrow_\Sigma (P)}$$ and for all sufficiently large $m$ neither the numerator nor the denominator of any of the fractions involved vanishes, then $$\lim_{m \rightarrow \infty}\frac{|\mathcal{V}(P_m, \, h_{q+1})|}{|\mathcal{V}(P_m, \, h_q)|} = \frac{\emph{Order}(i_q \downarrow i_{q+1}, \, P)}{\sum_{j \in i_q \downarrow (P)}\emph{Order}(i_q \downarrow j, \, P)+ i_q \downarrow_\Sigma (P)}.$$
\end{lem}
\begin{proof}
Clearly $\mathcal{V}(P_m, \, h_q) = \mathcal{V}(P_m, \, h_{q+1}) \uplus \mathcal{V}(P_m, \, h_q \setminus h_{q+1})$ where $\uplus$ emphasizes that this is a union of disjoint sets. The rest is just a matter of careful verification: we have
\begin{equation}\label{intermLemPfEq1}
\frac{|\mathcal{V}(P_m, \, h_{q+1})|}{|\mathcal{V}(P_m, \, h_q)|} = \frac{|\mathcal{V}(P_m, \, h_{q+1})|}{|\mathcal{V}(P_m, \, h_{q+1})| + |\mathcal{V}(P_m, \, h_q \setminus h_{q+1})|}=\frac{1}{1+\frac{|\mathcal{V}(P_m, \, h_q \setminus h_{q+1})|}{|\mathcal{V}(P_m, \, h_{q+1})|}}.
\end{equation}
Taking the limit as $m \rightarrow \infty$ on both sides of equation~\ref{intermLemPfEq1} yields
\begin{equation}\label{intermLemPfEq2}
\lim_{m \rightarrow \infty}\frac{|\mathcal{V}(P_m, \, h_{q+1})|}{|\mathcal{V}(P_m, \, h_q)|} = \frac{1}{1+\lim_{m \rightarrow \infty}\frac{|\mathcal{V}(P_m, \, h_q \setminus h_{q+1})|}{|\mathcal{V}(P_m, \, h_{q+1})|}}
\end{equation}
The right hand sides of equations~\ref{intermLemPfEq1} and \ref{intermLemPfEq2} are easily computed directly from the corresponding formulas in the assumptions and each of them is:
$$\frac{1}{1+\frac{\sum_{j \in i_q \downarrow (P) \text{ and }j \neq i_{q+1}}\text{Order}(i_q \downarrow j, \, P)+ i_q \downarrow_\Sigma (P)}{\text{Order}(i_q \downarrow i_{q+1}, \, P)}}=$$
$$=\frac{1}{\frac{\text{Order}(i_q \downarrow i_{q+1}, \, P)}{\text{Order}(i_q \downarrow i_{q+1}, \, P)}+\frac{\sum_{j \in i_q \downarrow (P) \text{ and }j \neq i_{q+1}}\text{Order}(i_q \downarrow j, \, P)+ i_q \downarrow_\Sigma (P)}{\text{Order}(i_q \downarrow i_{q+1}, \, P)}}=$$$$\frac{\text{Order}(i_q \downarrow i_{q+1}, \, P)}{\sum_{j \in i_q \downarrow (P)}\text{Order}(i_q \downarrow j, \, P)+ i_q \downarrow_\Sigma (P)}$$ yielding the asserted  conclusions.
\end{proof}
Entirely analogously,
\begin{lem}\label{bridgeRatioLem1}
Suppose that $\forall \, m$
$$\frac{|\mathcal{V}(P_m, \, h_k)|}{|\mathcal{V}(P_m, \, h_{k-1} \setminus h_k)|} = \frac{1}{\sum_{j \in i_{k-1} \downarrow (P)}\emph{Order}(i_{k-1} \downarrow j, \, P)+ i_{k-1} \downarrow_{\Sigma}(P)-1}$$ and the denominators do not vanish. Then $$\frac{|\mathcal{V}(P_m, \, h_k)|}{|\mathcal{V}(P_m, \, h_{k-1})|} = \frac{1}{\sum_{j \in i_{k-1} \downarrow (P)}\emph{Order}(i_{k-1} \downarrow j, \, P)+ i_{k-1} \downarrow_{\Sigma}(P)}.$$ Likewise, if $$\lim_{m \rightarrow \infty}\frac{|\mathcal{V}(P_m, \, h_k)|}{|\mathcal{V}(P_m, \, h_{k-1} \setminus h_k)|} = \frac{1}{\sum_{j \in i_{k-1} \downarrow (P)}\emph{Order}(i_{k-1} \downarrow j, \, P)+ i_{k-1} \downarrow_{\Sigma}(P)-1}$$ and for all sufficiently large $m$ the denominators of any of the fractions involved vanishes, then $$\lim_{m \rightarrow \infty}\frac{|\mathcal{V}(P_m, \, h_k)|}{|\mathcal{V}(P_m, \, h_{k-1})|} = \frac{1}{\sum_{j \in i_{k-1} \downarrow (P)}\emph{Order}(i_{k-1} \downarrow j, \, P)+ i_{k-1} \downarrow_{\Sigma}(P)}.$$
\end{lem}
To estimate or, in the special case of homologous population $P$, to compute exactly, the ratios $\frac{|\mathcal{V}(P_m, \, h_{q+1})|}{|\mathcal{V}(P_m, \, h_q \setminus h_{q+1})|}$ the following strategy will be employed. For a given $m \in \mathbb{N}$ consider the set of all populations $\mathcal{V}(P_m, \, h_q)$ (i.e. the set of these populations in $[P_m]$ the first individual of which fits the schema $h_q$). Let now $\pi_{q, \, m}$ denote the uniform probability measure on the set $\mathcal{V}(P_m, \, h_q)$. We then have
\begin{equation}\label{ratioOfSizesAndDistribEq}
\frac{|\mathcal{V}(P_m, \, h_{q+1})|}{|\mathcal{V}(P_m, \, h_q \setminus h_{q+1})|} = \frac{\frac{|\mathcal{V}(P_m, \, h_{q+1})|}{|\mathcal{V}(P_m, \, h_q)|}}{\frac{|\mathcal{V}(P_m, \, h_q) \setminus h_{q+1})|}{|\mathcal{V}(P_m, \, h_q)|}} = \frac{\pi_{q, \, m}(\mathcal{V}(P_m, \, h_{q+1}))}{\pi_{q, \, m}(\mathcal{V}(P_m, \, h_q \setminus h_{q+1}))}
\end{equation}
and, more generally, $\forall$ set of rollouts $S$,
\begin{equation}\label{ratioOfSizesAndDistribGenEq}
\frac{|\mathcal{V}(P_m, \, h_{q+1} \cap S)|}{|\mathcal{V}(P_m, \, h_q \setminus h_q)|} = \frac{\frac{|\mathcal{V}(P_m, \, h_{q+1} \cap S)|}{|\mathcal{V}(P_m, \, h_q)|}}{\frac{|\mathcal{V}(P_m, \, h_q) \setminus h_{q+1})|}{|\mathcal{V}(P_m, \, h_q)|}} = \frac{\pi_{q, \, m}(\mathcal{V}(P_m, \, h_{q+1} \cap S))}{\pi_{q, \, m}(\mathcal{V}(P_m, \, h_q \setminus h_{q+1}))}
\end{equation}
The idea behind equations~\ref{ratioOfSizesAndDistribEq} and~\ref{ratioOfSizesAndDistribGenEq} is to construct a Markov chain with a uniform stationary distribution on the state space $\mathcal{V}(P_m, \, h_q)$ thereby opening the door to an application of lemma~\ref{estimStationaryRatiosLem}. It seems the easiest construction to accomplish our task uses proposition~\ref{symmetricTransMatrixLem}. Recall the transformations of the form $\chi_{i, \, x, \, y}$ as in definitions~\ref{rolloutPartCrossDefn} and~\ref{recombActOnPopsDef} from definition~\ref{RecombStagePopTransDefn}. We now construct our Markov chain, call it $\mathcal{M}_q$, on the set $\mathcal{V}(P_m, \, h_q)$ where $q < k$ as follows: given a population of rollouts $Q_t \in \mathcal{V}(P_m, \, h_q)$ at time $t$, let $(i_q, \, x)$ be the state in the first rollout and $q^{\text{th}}$ position in the population $Q_t$. Consider the set
$$\text{States}_m (i_q \downarrow Q_t) = \{(j, \, z) \, | \, j \in i_q \downarrow (Q_t), \, z \in A \times m \text{ and the state } (j, \, z)$$
$$\text{appears in the population }Q_t \text{ following a state with equivalence class }i_q\} \cup$$
\begin{equation}\label{restrictSucfamDefEq}
\cup \{(f, j) \, | \, 1 \leq j \leq m \text{ and } f \in i_q \downarrow_{\Sigma} P\}.
\end{equation}
Now select a state or a terminal label; call either one of these $v$, from the set finite set $\text{States}_m (i_q \downarrow Q_t)$ uniformly at random. Since each state appears uniquely in a population $Q_t$, by definition of the set $\text{States}_m (i_q \downarrow Q_t)$ in \ref{restrictSucfamDefEq}, the state preceding the element $v$ selected from $\text{States}_m (i_q \downarrow Q_t)$, call it $u$, is of the form $u = (i_q, \, y)$ where $y \in A \times m$. Now let $Q_{t+1} = \chi_{i_q, \, x, \, y}(Q_t)$. Notice that there are two mutually exclusive cases here:

\textit{Case 1:} The states $u$ and $(i_q, \, x)$ appear in different rollouts (or, equivalently, the state $u$ does not appear in the first rollout since the state  $(i_q, \, x)$ does by definition). In this case $Q_{t+1} \neq Q_t$ and the state in the first rollout of the population $Q_{t+1}$ in the $q+1^{\text{st}}$ position is $v$. In this case we will say that the element $v$ is \emph{mobile}.

\textit{Case 2:} The states $u$ and $(i_q, \, x)$ appear in the same rollout (of course, it has to be the first rollout). In this case $Q_{t+1} = Q_t$. We will say that the element $v$ is \emph{immobile}.

Notice that in either of the cases, the population $Q_{t+1} \in \mathcal{V}(P_m, \, h_q)$ so that the Markov process is well defined on the set of populations $\mathcal{V}(P_m, \, h_q) \subseteq [P_m]$. We now emphasize the following simple important facts:
\begin{lem}\label{constSizeLem}
$\forall \, Q \in \mathcal{V}(P_m, \, h_q)$ $|\emph{States}_m (i_q \downarrow Q)| = m \cdot |\emph{States}_1 (i_q \downarrow P)|$ and \\
$|\emph{States}_1 (i_q \downarrow P)| = \sum_{j \in i_q \downarrow (P)}\emph{Order} (i_q \downarrow j, \, P)+i_q \downarrow_{\Sigma}(P)$
\end{lem}
\begin{proof}
The fact that $|\text{States}_1 (i_q \downarrow P)| = \sum_{j \in i_q \downarrow (P)}\text{Order} (i_q \downarrow j, \, P)+i_q \downarrow_{\Sigma}(P)$ follows directly from the definitions. Definition of the set $\text{States}_m (i_q \downarrow Q)$ in \ref{restrictSucfamDefEq} together with remark~\ref{equivClassIndepRem} tell us that $\text{States}_m (i_q \downarrow Q) = \text{States}_1 (i_q \downarrow P_m)$ (where $P_m$ plays the role of $P$ for the time being) so that $$\text{States}_m (i_q \downarrow Q) = |\text{States}_1 (i_q \downarrow P_m)| = \sum_{j \in i_q \downarrow (P_m)}\text{Order} (i_q \downarrow j, \, P_m)+i_q \downarrow_{\Sigma}(P_m)=$$$$\overset{\text{by proposition~\ref{popRatioFacts}}}{=}\sum_{j \in i_q \downarrow (P)}m \cdot \text{Order} (i_q \downarrow j, \, P)+m \cdot i_q \downarrow_{\Sigma}(P) = $$$$= m \cdot \left(\sum_{j \in i_q \downarrow (P)}\text{Order} (i_q \downarrow j, \, P)+i_q \downarrow_{\Sigma}(P) \right) \overset{\text{by the already proven fact}}{=} m \cdot |\text{States}_1 (i_q \downarrow P)|.$$
\end{proof}
Another very simple important observation is the following:
\begin{lem}\label{transProbOfChainObservLem}
Given any two populations $Q$ and $Q' \in \mathcal{V}(P_m, \, h_q)$, let $p^q_{Q \rightarrow Q'}$ denote the transition probability of the Markov chain $\mathcal{M}_q$ as constructed above. Then either $p^q_{Q \rightarrow Q'} = 0$ or $p^q_{Q \rightarrow Q'} = \frac{1}{m \cdot |\emph{States}_1 (i_q \downarrow P)|}$. Moreover, $p^q_{Q \rightarrow Q'} = p^q_{Q' \rightarrow Q}$ and the uniform distribution is a stationary distribution of the Markov chain $\mathcal{M}_q$.
\end{lem}
\begin{proof}
From the construction it is clear that if $p^q_{Q \rightarrow Q'} \neq 0$ then there must be an element $s \in \text{States}_m (i_q \downarrow Q)$ which appears in a rollout in the population $Q$ different from the first one and it is the state at the $q^{th}$ position of the first rollout of the population $Q'$ while definition~\ref{RecombStagePopTransDefn} tells us that the state $(i_q, \, x)$ in the $q^{th}$ position of the first rollout of the population $Q$ appears in $Q'$ in some rollout that is not the first one (the former position of the state $s$ that is now in position $q$ of the first rollout of $Q'$) and it is also a member of the set $\text{States}_m (i_q \downarrow Q')$ according to the way $\text{States}_m (i_q \downarrow Q')$ is introduced in \ref{restrictSucfamDefEq}. According to lemma~\ref{constSizeLem} $\text{States}_m (i_q \downarrow Q) = \text{States}_m (i_q \downarrow Q') = m \cdot |\text{States}_1 (i_q \downarrow Q)|$ so that the desired conclusion that $p^q_{Q \rightarrow Q'} = p^q_{Q' \rightarrow Q}$ follows from the construction of the Markov chain $\mathcal{M}_q$. The uniform probability distribution is a stationary distribution of the Markov chain $\mathcal{M}_q$ since we have just shown that the Markov transition matrix is symmetric (see also proposition~\ref{symmetricTransMatrixLem}).
\end{proof}
Recall the generalized transition probabilities introduced in definition~\ref{transProbDefn}. For the remaining part of this section it is convenient to introduce the following definition:
\begin{defn}\label{numbOfMobileElmsDefn}
Given a population $Q \in \mathcal{V}(P_m, \, h_{q+1})$, let $\text{Mobile}_q(Q)$ denote the number of mobile elements  (see \textit{case} 1 above) in the set $\text{States}_m (i_q \downarrow Q)$ that move the population $Q$ away from the set $\mathcal{V}(P_m, \, h_{q+1})$ (and hence, into the set $\mathcal{V}(P_m, \, h_q \setminus h_{q+1})$) under the application of the Markov chain $\mathcal{M}_q$ as constructed above. Dually, given $Q \in \mathcal{V}(P_m, \, h_q \setminus h_{q+1})$, let $\text{Mobile}_q(Q)$ denote the number of mobile elements in the set $\text{States}_m (i_q \downarrow Q)$ that move the population $Q$ away from the set $\mathcal{V}(P_m, \, h_q \setminus h_{q+1})$ (and hence, into the set $\mathcal{V}(P_m, \, h_{q+1})$).
\end{defn}
Suppose, for the time being, that the set $\mathcal{V}(P_m, \, h_{q+1}) \neq \emptyset$. Given a population $Q \in \mathcal{V}(P_m, \, h_{q+1})$, notice that
\begin{equation}\label{numbOfMovingTransUpperEq}
\text{Mobile}_q(Q) \leq \begin{cases}
\sum_{j \in i_q \downarrow Q \text{ and } j \neq i_{q+1}}\text{Order}(i_q \downarrow j)(Q)+i_q \downarrow_{\Sigma}(Q) & \text{if }q<k-1\\
\sum_{j \in i_q \downarrow Q}\text{Order}(i_q \downarrow j)(Q)+i_q \downarrow_{\Sigma}(Q)-m & \text{if }q = k-1
\end{cases}
\end{equation}
Notice that in case the population $P$ is homologous (and hence so are $P_m$ and $Q$) there are no immobile elements in the population $Q$ so that the inequality~\ref{numbOfMovingTransUpperEq} turns into an exact equation. In general, from \textit{case} 2 above it is clear that the total number of all the immobile elements is crudely bounded above by the height of the first rollout in the population $Q$, $H_1(Q)$. We now obtain a lower bound on the total number of mobile elements in the set $\text{States}_m(i_q \downarrow Q)$ that move the population $Q$ away from the set $\mathcal{V}(P_m, \, h_{q+1})$ into the set $\mathcal{V}(P_m, \, h_q \setminus h_{q+1})$: this number is at least
$$\text{Mobile}_q(Q) \geq$$
\begin{equation}\label{numbOfMovingTransLowerEq}
\geq \begin{cases}
\sum_{j \in i_q \downarrow Q \text{ and } j \neq i_{q+1}}\text{Order}(i_q \downarrow j)(Q)+i_q \downarrow_{\Sigma}(Q)-H_1(Q) & \text{if }q<k-1\\
\sum_{j \in i_q \downarrow Q}\text{Order}(i_q \downarrow j)(Q)+i_q \downarrow_{\Sigma}(Q)- m - H_1(Q)& \text{if }q = k-1
\end{cases}
\end{equation}
Analogously, if the population $Q \in \mathcal{V}(P_m, \, h_q \setminus h_{q+1})$ then the total number of mobile elements in the set $\text{States}_m (i_q \downarrow Q)$ that move the population $Q$ away from the set $\mathcal{V}(P_m, \, h_q \setminus h_{q+1})$ (and hence, into the set $\mathcal{V}(P_m, \, h_{q+1})$)
\begin{equation}\label{numbOfMovingTransUpperInEq}
\text{Mobile}_q(Q) \leq \begin{cases}
\text{Order}(i_q \downarrow i_{q+1})(Q) & \text{if }q<k-1\\
m & \text{if }q = k-1
\end{cases}
\end{equation}
and, as before, the inequality turns into an exact equation in the case when $Q$ is a homologous population. At the same time
\begin{equation}\label{numbOfMovingTransLowerInEq}
\text{Mobile}_q(Q) \geq \begin{cases}
\text{Order}(i_q \downarrow i_{q+1})(Q) - H_1(Q) & \text{if }q<k-1\\
m - H_1(Q) & \text{if }q = k-1
\end{cases}
\end{equation}
In view of proposition~\ref{popRatioFacts} and remark~\ref{equivClassIndepRem} inequalities~\ref{numbOfMovingTransUpperEq}, \ref{numbOfMovingTransLowerEq}, \ref{numbOfMovingTransUpperInEq} and \ref{numbOfMovingTransLowerInEq} can be rewritten verbatim replacing $\text{Order}(i_q \downarrow i_{q+1})(Q)$ with $m \cdot \text{Order}(i_q \downarrow i_{q+1})(P)$, and $\text{Order}(i_q \downarrow j)(Q)$ with $m \cdot \text{Order}(i_q \downarrow j)(P)$.

For the case of homologous population $Q$ the situation is particularly simple:
\begin{lem}\label{homologousPopTransLem}
Suppose the population $P$ is homologous. Suppose further, that neither one of the sets $\mathcal{V}(P_m, \, h_{q+1})$ and $\mathcal{V}(P_m, \, h_q \setminus h_{q+1})$ is empty. Then $\forall \, m \in \mathbb{N}$ we have
$$p^q_{\mathcal{V}(P_m, \, h_{q+1}) \rightarrow \mathcal{V}(P_m, \, h_q \setminus h_{q+1})} = \begin{cases}
\frac{\sum_{j \in i_q \downarrow P \emph{ and } j \neq i_{q+1}}\emph{Order}(i_q \downarrow j)(P)+i_q \downarrow_{\Sigma}(P)}{\sum_{j \in i_q \downarrow P}\emph{Order}(i_q \downarrow j)(P)+i_q \downarrow_{\Sigma}(P)} & \text{if }q<k-1 \\
\frac{\sum_{j \in i_q \downarrow P}\emph{Order}(i_q \downarrow j)(P)+i_q \downarrow_{\Sigma}(P)-1}{\sum_{j \in i_q \downarrow P}\emph{Order}(i_q \downarrow j)(P)+i_q \downarrow_{\Sigma}(P)} & \text{if }q=k-1
\end{cases},$$
$$p^q_{\mathcal{V}(P_m, \, h_q \setminus h_{q+1}) \rightarrow \mathcal{V}(P_m, \, h_{q+1})} = \begin{cases}
\frac{\emph{Order}(i_q \downarrow i_{q+1})(P)+i_q \downarrow_{\Sigma}(P)}{\sum_{j \in i_q \downarrow P}\emph{Order}(i_q \downarrow j)(P)+i_q \downarrow_{\Sigma}(P)} & \text{if }q<k-1 \\
\frac{1}{\sum_{j \in i_q \downarrow P}\emph{Order}(i_q \downarrow j)(P)+i_q \downarrow_{\Sigma}(P)} & \text{if }q=k-1
\end{cases}.$$
Consequently, $\forall \, m \in \mathbb{N}$
$$\frac{\pi_{q, \, m}(\mathcal{V}(P_m, \, h_{q+1}))}{\pi_{q, \, m}(\mathcal{V}(P_m, \, h_q \setminus h_{q+1}))} =
\begin{cases}
\frac{\emph{Order}(i_q \downarrow i_{q+1})(P)+i_q \downarrow_{\Sigma}(P)}{\sum_{j \in i_q \downarrow P \emph{ and }j \neq i_{q+1}}\emph{Order}(i_q \downarrow j)(P)+i_q \downarrow_{\Sigma}(P)} & \text{if }q<k-1 \\
\frac{1}{\sum_{j \in i_q \downarrow P}\emph{Order}(i_q \downarrow j)(P)+i_q \downarrow_{\Sigma}(P)-1} & \text{if }q=k-1
\end{cases}$$
\end{lem}
\begin{proof}
The first and the second conclusions follow from equations~\ref{numbOfMovingTransUpperEq} and \ref{numbOfMovingTransUpperInEq} combined with lemma~\ref{transProbOfChainObservLem}, definition~\ref{transProbDefn} and comment following equation~\ref{numbOfMovingTransLowerInEq}. The last conclusion is an immediate application of equation~\ref{stationaryDistribRatioEq} to the lumping quotient of the Markov chain $\mathcal{M}_q$ into the two states $A = \mathcal{V}(P_m, \, h_{q+1})$ and $B = \mathcal{V}(P_m, \, h_q \setminus h_{q+1})$.
\end{proof}
All that remains to do now to establish theorem~\ref{GeiringerLikeThmForMCTMain} in the special case of homologous population $P$ is to show that whenever $1 \leq q \leq k-1$ and none of the ``trivial extremes" takes place (see the beginning of this subsection), the sets $\mathcal{V}(P_m, \, h_{q+1})$ and $\mathcal{V}(P_m, \, h_q \setminus h_{q+1})$ are nonempty. This will be done later jointly with the corresponding fact needed for the general case. Meanwhile, we return to the estimation of the ratios of the form $\frac{\pi_{q, \, m}(\mathcal{V}(P_m, \, h_{q+1}))}{\pi_{q, \, m}(\mathcal{V}(P_m, \, h_q \setminus h_{q+1}))}$ in the general case. Suppose, for now, the following statement is true:
$$\forall \, q \text{ with } 1 \leq q < k \; \exists \, const(q) \in (0, \, 1) \text{ such that }\forall \, \text{ sufficiently large } m$$
\begin{equation}\label{prelimAssumpFact1}
\text{ we have } \rho_m(\mathcal{V}(P_m, \, h_{q+1}))> const(q) \text{ and } \rho_m(\mathcal{V}(P_m, \, h_q \setminus h_{q+1})) > const(q)
\end{equation}
In the general case of non-homologous population $P$ the presence of immobile states significantly complicates the situation. This is where Markov inequality comes to the rescue telling us that as $m$ increases the height of the first rollout (and hence the number of immobile states) being large becomes more and more rare event so that the bounds in the inequalities~\ref{numbOfMovingTransUpperEq} and \ref{numbOfMovingTransLowerEq} as well as inequalities~\ref{numbOfMovingTransUpperInEq} and \ref{numbOfMovingTransLowerInEq} get closer and closer together. We now proceed in detail. Recall the construction of the sets $U_m^{\delta}$ starting with equation~\ref{expSumOfRollsEqn} and ending with inequality~\ref{smallSetPropEq}. Let $\delta > 0$ be given. According to inequality~\ref{smallSetPropEq} $\exists \, M_1$ large enough so that $\forall \, m > M_1$ we have $\rho_m(U_{m}^{\delta \cdot const(q+1)}) < \delta \cdot const(q+1)$. where $const(q+1)$ is as in the assumption statement~\ref{prelimAssumpFact1}. We now have
$$\frac{\pi_{q, \, m}\left(\mathcal{V}(P_m, \, h_{q+1}) \cap U_{m}^{\delta \cdot const(q+1)}\right)}{\pi_{q, \, m}(\mathcal{V}(P_m, \, h_{q+1}))} \leq \frac{\pi_{q, \, m}\left(U_{m}^{\delta \cdot const(q+1)}\right)}{\pi_{q, \, m}(\mathcal{V}(P_m, \, h_{q+1}))}=$$$$=\frac{\frac{|U_{m}^{\delta \cdot const(q+1)}|}{|\mathcal{V}(P_m, \, h_q)|}}{\frac{|\mathcal{V}(P_m, \, h_{q+1})|}{|\mathcal{V}(P_m, \, h_q)|}}=\frac{|U_{m}^{\delta \cdot const(q+1)}|}{|(\mathcal{V}(P_m, \, h_{q+1})|} = \frac{\frac{|U_{m}^{\delta \cdot const(q+1)}|}{|[P_m]_{\widetilde{\mathcal{F}}}|}}{\frac{|\mathcal{V}(P_m, \, h_{q+1})|}{|[P_m]_{\widetilde{\mathcal{F}}}|}}=$$
\begin{equation}\label{condDistribRatioEq1}
 = \frac{\rho_m\left(U_{m}^{\delta \cdot const(q+1)}\right)}{\rho_m(\mathcal{V}(P_m, \, h_{q+1}))} \leq \frac{\delta \cdot const(q+1)}{const(q+1)} = \delta.
\end{equation}
Analogously, $$\frac{\pi_{q, \, m}\left(\mathcal{V}(P_m, \, h_q \setminus h_{q+1}) \cap U_{m}^{\delta \cdot const(q+1)}\right)}{\pi_{q, \, m}(\mathcal{V}(P_m, \, h_q \setminus h_{q+1}))} \leq $$
\begin{equation}\label{condDistribRatioEq2}
\leq \frac{\pi_{q, \, m}\left(U_{m}^{\delta \cdot const(q+1)}\right)}{\pi_{q, \, m}(\mathcal{V}(P_m, \, h_q \setminus h_{q+1}))}=\frac{\rho_m\left(U_{m}^{\delta \cdot const(q+1)}\right)}{\rho_m(\mathcal{V}(P_m, \, h_q \setminus h_{q+1}))} \leq \delta
\end{equation}
Now observe that as long as a population $Q \in \mathcal{V}(P_m, \, h_{q+1}) \setminus U_{m}^{\delta \cdot const(q+1)}$, the hight of the first rollout $H_1(Q) \leq (\delta \cdot const(q+1)) \cdot m \leq \delta \cdot m$ (recall how the sets of the form $U_m^{\epsilon}$ are introduced from~\ref{smallSetEqDefn}). Now, for $q < k-1$ inequalities~\ref{numbOfMovingTransUpperEq}, \ref{numbOfMovingTransLowerEq} and lemma~\ref{constSizeLem} tell us that for $\forall \, m > M_1$ we have
$$\frac{m \cdot \left(\left(\sum_{j \in i_q \downarrow (P), \, j \neq i_{q+1}}\text{Order}(i_q \downarrow j, \, P)\right) + i_q \downarrow_{\Sigma}(P)\right) - \delta \cdot m}{m \cdot |\text{States}_1(i_q \downarrow P)|} \leq$$
$$
\leq p_{Q \rightarrow \mathcal{V}(P_m, \, h_q \setminus h_{q+1})} \leq \frac{m \cdot \left(\left(\sum_{j \in i_q \downarrow (P), \, j \neq i_{q+1}}\text{Order}(i_q \downarrow j, \, P)\right) + i_q \downarrow_{\Sigma}(P)\right)}{m \cdot |\text{States}_1(i_q \downarrow P)|}
$$
so that dividing the numerator and the denominator by $m$ gives
$$\frac{\left(\sum_{j \in i_q \downarrow (P), \, j \neq i_{q+1}}\text{Order}(i_q \downarrow j, \, P)\right) + i_q \downarrow_{\Sigma} - \delta}{|\text{States}_1(i_q \downarrow P)|} \leq$$
\begin{equation}\label{TransProbEstimEq1}
\leq p_{Q \rightarrow \mathcal{V}(P_m, \, h_q \setminus h_{q+1})} \leq \frac{\sum_{j \in i_q \downarrow (P), \, j \neq i_{q+1}}\text{Order}(i_q \downarrow j, \, P) + i_q \downarrow_{\Sigma}(P)}{|\text{States}_1(i_q \downarrow P)|}
\end{equation}
Entirely analogous and, by now, well familiar to the reader reasoning with inequality~\ref{condDistribRatioEq2} playing the role of inequality~\ref{condDistribRatioEq1} shows that whenever $m > M_1$ and a population $Q \in \mathcal{V}(P_m, \, h_q \setminus h_{q+1}) \setminus U_{m}^{\delta \cdot const(q+1)}$ we have
\begin{equation}\label{TransProbEstimEq2}
\frac{\text{Order}(i_q \downarrow i_{q+1}, \, P) - \delta}{|\text{States}_1(i_q \downarrow P)|} \leq p_{Q \rightarrow \mathcal{V}(P_m, \, h_{q+1})} \leq \frac{\text{Order}(i_q \downarrow i_{q+1}, \, P)}{|\text{States}_1(i_q \downarrow P)|}
\end{equation}
Now inequalities~\ref{condDistribRatioEq1}, \ref{condDistribRatioEq2}, \ref{TransProbEstimEq1} and \ref{TransProbEstimEq2} allow us to apply lemma~\ref{estimStationaryRatiosLem} with $A = \mathcal{V}(P_m, \, h_{q+1})$, $B = \mathcal{V}(P_m, \, h_q \setminus h_{q+1})$ and $U = U_{m}^{\delta \cdot const(q+1)}$ and concluding that $\forall \, m > M_1$ we have $$\frac{(1 - \delta) \cdot \frac{\text{Order}(i_q \downarrow i_{q+1}, \, P) - \delta}{|\text{States}_1(i_q \downarrow P)|}}{(1-\delta) \cdot \left(\frac{\sum_{j \in i_q \downarrow (P), \, j \neq i_{q+1}}\text{Order}(i_q \downarrow j, \, P)+ i_q \downarrow_{\Sigma}(P)}{|\text{States}_1(i_q \downarrow P)|} \right) + \delta}\leq \frac{\pi_{q, \, m}(\mathcal{V}(P_m, \, h_{q+1}))}{\pi_{q, \, m}(\mathcal{V}(P_m, \, h_q \setminus h_{q+1}))} \leq$$$$\leq \frac{(1 - \delta) \cdot \frac{\text{Order}(i_q \downarrow i_{q+1}, \, P)}{|\text{States}_1(i_q \downarrow P)|}+\delta}{(1-\delta) \cdot \left(\frac{\sum_{j \in i_q \downarrow (P), \, j \neq i_{q+1}}\text{Order}(i_q \downarrow j, \, P)+ i_q \downarrow_{\Sigma}(P)-\delta}{|\text{States}_1(i_q \downarrow P)|} \right)}.$$ Multiplying the numerator and the denominator of the leftmost and the rightmost fractions by the constant $|\text{States}_1(i_q \downarrow P)|$ which does not depend on $m$ we obtain $$\frac{(1 - \delta)\cdot\left(\text{Order}(i_q \downarrow i_{q+1}, \, P) - \delta \cdot |\text{States}_1(i_q \downarrow P)|\right)}{(1 - \delta) \cdot \left(\sum_{j \in i_q \downarrow (P), \, j \neq i_{q+1}}\text{Order}(i_q \downarrow j, \, P)+ i_q \downarrow_{\Sigma}(P) \right) + \delta \cdot |\text{States}_1(i_q \downarrow P)|} \leq$$$$\leq \frac{\pi_{q, \, m}(\mathcal{V}(P_m, \, h_{q+1}))}{\pi_{q, \, m}(\mathcal{V}(P_m, \, h_q \setminus h_{q+1}))} \leq$$
\begin{equation}\label{mainInequalPfEq}
\frac{(1 - \delta)\cdot \text{Order}(i_q \downarrow i_{q+1}, \, P) + \delta \cdot |\text{States}_1(i_q \downarrow P)|}{(1 - \delta) \left(\sum_{j \in i_q \downarrow (P), \, j \neq i_{q+1}}\text{Order}(i_q \downarrow j, \, P)+ i_q \downarrow_{\Sigma}(P) - \delta \cdot |\text{States}_1(i_q \downarrow P)|\right)}
\end{equation}
Now simply observe that the leftmost and the rightmost sides of the inequality~\ref{mainInequalPfEq} are both differentiable (and, hence, continuous) functions of $\delta$ on the domain $(-0.5, \, 0.5)$ (notice that the denominators do not vanish on this domain thanks to the assumption that neither of the trivial extremes takes place). It follows immediately then that both, the leftmost and the rightmost sides of the inequality~\ref{mainInequalPfEq} converge to the same value, namely to the desired ratio $$R = \frac{\text{Order}(i_q \downarrow i_{q+1}, \, P)}{\sum_{j \in i_q \downarrow (P), \, j \neq i_{q+1}}\text{Order}(i_q \downarrow j, \, P)+ i_q \downarrow_{\Sigma}(P)}$$ as $\delta \rightarrow 0$. From the definition of a limit of a real-valued function at a point, it follows that given any $\epsilon > 0$ we can choose small enough $\delta > 0$ such that both, the leftmost and the rightmost sides of the inequality~\ref{mainInequalPfEq} are within $\epsilon$ error of $R$. We have now shown that depending on this $\delta$ we can then choose sufficiently large $M$ so that the ratio $\frac{\pi_{q, \, m}(\mathcal{V}(P_m, \, h_{q+1}))}{\pi_{q, \, m}(\mathcal{V}(P_m, \, h_q \setminus h_{q+1}))}$, being squeezed between the two quantities within the $\epsilon$ error of $R$, is itself within the error at most $\epsilon$ of $R$. In summary, we have finally proved the following
\begin{lem}\label{limitRatioNonhomologousLem1}
Assume that the statement in \ref{prelimAssumpFact1} is true. Then whenever $1 < q < k-1$ we have $$\lim_{m \rightarrow \infty}\frac{\pi_{q, \, m}(\mathcal{V}(P_m, \, h_{q+1}))}{\pi_{q, \, m}(\mathcal{V}(P_m, \, h_q \setminus h_{q+1}))} = \frac{\emph{Order}(i_q \downarrow i_{q+1}, \, P)}{\sum_{j \in i_q \downarrow (P), \, j \neq i_{q+1}}\emph{Order}(i_q \downarrow j, \, P)+ i_q \downarrow_{\Sigma}(P)}.$$
\end{lem}
An entirely analogous argument shows the following:
\begin{lem}\label{limitRatioNonhomologousLem2}
Assume that the statement in \ref{prelimAssumpFact1} is true. Then $$\lim_{m \rightarrow \infty}\frac{\pi_{k-1, \, m}(\mathcal{V}(P_m, \, h_k))}{\pi_{k-1, \, m}(\mathcal{V}(P_m, \, h_{k-1} \setminus h_k))} =$$$$=\frac{1}{\sum_{j \in i_{k-1} \downarrow (P)}\emph{Order}(i_{k-1} \downarrow j, \, P)+ i_{k-1} \downarrow_{\Sigma}(P)-1}.$$
\end{lem}
According to lemmas~\ref{bridgeRatioLem} and \ref{bridgeRatioLem1}, equations~\ref{ratioOfSizesAndDistribEq} and \ref{ratioOfSizesAndDistribGenEq}, lemmas~\ref{limitRatioNonhomologousLem1}, \ref{limitRatioNonhomologousLem2}, \ref{homologousPopTransLem} and equations~\ref{compRatioForMCTEq} and \ref{compRatioForMCTWithLimitEq}, all that remains to be proven to establish theorem~\ref{GeiringerLikeThmForMCTMain} is the following:
\begin{lem}\label{technConditionLem}
Suppose neither of the trivial extremes takes place. Then the statement in equation~\ref{prelimAssumpFact1} is true. Furthermore, in case of homologous recombination the statement is true for all $m$ (not only for large enough $m$).
\end{lem}
\begin{proof}
We proceed by induction on the index $q$. First of all, recall from the beginning of the current subsection that we have already shown that $\forall \, m \in \mathbb{N}$ we have $$\rho_m(\mathcal{V}(P_m, \, h_1)) \overset{\text{by lemma~\ref{mainGeiringerLikeLemma}}}{=} \lim_{t \rightarrow \infty}\Phi(h_1, \, P_m, \, t) = \frac{\text{Order}(\alpha \downarrow i_1, \, P)}{b}>0$$ where the last inequality holds because none of the trivial extremes takes place so that $\text{Order}(\alpha \downarrow i_1, \, P) \neq 0$ (recall that $\rho_m$ denotes the uniform probability distribution on $[P_m]$ so that $\rho_m(\mathcal{V}(P_m, \, h_1)) = \frac{|\mathcal{V}(P_m, \, h_1)|}{[P_m]}$). Since $\mathcal{V}(P_m, \, h_1) = \mathcal{V}(P_m, \, h_2) \uplus \mathcal{V}(P_m, \, h_1 \setminus h_2)$ we also have $\rho_m(\mathcal{V}(P_m, \, h_2)) + \rho_m(\mathcal{V}(P_m, \, h_1 \setminus h_2)) = \rho_m(\mathcal{V}(P_m, \, h_1)) = \frac{\text{Order}(\alpha \downarrow i_1, \, P)}{b} = const_0$ where $1 \geq const_0 > 0$ and $const_0$ is independent of $m$. It follows then that at least one of the following is true: $\rho_m(\mathcal{V}(P_m, \, h_2)) \geq \frac{const_0}{2}$ or $\rho_m(\mathcal{V}(P_m, \, h_1 \setminus h_2)) \geq \frac{const_0}{2}$. In the general case, choose $M_1$ large enough so that $\forall$ $m > M_1$ we have $\rho_m(U_m^{\frac{const_0}{4}}) \leq \frac{const_0}{4}$ (recall the part of the proof starting with equation~\ref{expSumOfRollsEqn} and ending with inequality~\ref{smallSetPropEq}). It follows then that either $$\rho_m\left(\mathcal{V}(P_m, \, h_2) \setminus U_m^{\frac{const_0}{4}}\right) \geq \frac{const_0}{4} \text{ or } \rho_m\left(\mathcal{V}(P_m, \, h_1 \setminus h_2) \setminus U_m^{\frac{const_0}{4}}\right) \geq \frac{const_0}{4}.$$ An already familiar argument exploiting corollary~\ref{EstimOfTransCor}, inequalities~\ref{numbOfMovingTransUpperEq}, \ref{numbOfMovingTransLowerEq}, \ref{numbOfMovingTransUpperInEq}, \ref{numbOfMovingTransLowerInEq} and lemma~\ref{constSizeLem} shows that, thanks to the assumption that no trivial extremes take place, and observing that $1 - \frac{const_0}{4} \geq \frac{1}{4}$ for all large enough $m$ the ratios $$\frac{p^1_{\left(\mathcal{V}(P_m, \, h_2) \setminus U_m^{\frac{const_0}{4}}\right) \rightarrow \mathcal{V}(P_m, \, h_1 \setminus h_2)}}{p^1_{\mathcal{V}(P_m, \, h_1 \setminus h_2) \rightarrow \mathcal{V}(P_m, \, h_2)}} \geq \kappa_1$$ and, likewise, $$\frac{p^1_{\left(\mathcal{V}(P_m, \, h_1 \setminus h_2) \setminus U_m^{\frac{const_0}{4}}\right) \rightarrow \mathcal{V}(P_m, \, h_2)}}{p^1_{\mathcal{V}(P_m, \, h_2) \rightarrow \mathcal{V}(P_m, \, h_1 \setminus h_2)}} \geq \kappa_2$$ where both, $\kappa_1$ and $\kappa_2 > 0$ and independent of $m$. Now we apply lemma~\ref{littleTechnVarifyLem} to the sets $B = \mathcal{V}(P_m, \, h_2) \setminus U_m^{\frac{const_0}{4}}$ and $A = \mathcal{V}(P_m, \, h_1 \setminus h_2)$ in the case when $\rho_m\left(\mathcal{V}(P_m, \, h_2) \setminus U_m^{\frac{const_0}{4}}\right) \geq \frac{const_0}{4}$ or to the pair of sets $B = \mathcal{V}(P_m, \, h_1 \setminus h_2) \setminus U_m^{\frac{const_0}{4}}$ and $A = \mathcal{V}(P_m, \, h_2)$ in the case when $\rho_m\left(\mathcal{V}(P_m, \, h_1 \setminus h_2) \setminus U_m^{\frac{const_0}{4}}\right) \geq \frac{const_0}{4}$, tells us that if we let $const(1) = \min\{\frac{const_0}{4}, \, \frac{const_0}{4}\cdot \kappa_1, \, \frac{const_0}{4}\cdot \kappa_2\}$ then the statement in \ref{prelimAssumpFact1} is true for $q=1$. This establishes the base case of induction. Now observe that if the statement in \ref{prelimAssumpFact1} holds for some $q$ then it is true, in particular, that $\exists$ a constant $const(q)$ independent of $m$ such that for all large enough $m$ we have $\mathcal{V}(P_m, \, h_q) > const(q)$. Now the validity of the statement in \ref{prelimAssumpFact1} for $q+1$ follows from an entirely analogous argument to the one in the base case of induction with $const(q)$ playing the role of $const(0)$ and the Markov chain $\mathcal{M}_q$ replacing the Markov chain $\mathcal{M}_1$. In the case of homologous recombination, an even simpler (since there is no need to worry about the height of the first rollout), analogous argument shows that the statement in \ref{prelimAssumpFact1} holds $\forall \, m$.
\end{proof}
\section{A Further Strengthening of the General Finite Population Geiringer Theorem for Evolutionary Algorithms}\label{furthStrengthSect}
\subsection{A Form of the Classical Contraction Mapping Principle for a Family of Maps having the same Fixed Point}\label{ContractionMappingSubsect}
The material of this section requires familiarity with elementary point set topology or with basic theory of metric spaces (see, for instance, \cite{SimmonsG}). Throughout this section $(X, \, d)$ denotes a complete metric space. We recall the following from classical theory of metric spaces:
\begin{defn}\label{contractMapDef}
We say that a map $f: X \rightarrow X$ is a \emph{contraction} on $X$ if $\exists \, k < 1$ such that $\forall \, x, \, y \in X$ we have $d(f(x), \, f(y)) \leq k \cdot d(x, y)$. We also call $k$ a \emph{contraction rate}.\footnote{Evidently contraction rate is not unique with such a notion. Nonetheless, the minimal contraction rate does exist since it is the $\inf\{k \, | \, k \text{ is a contraction rate}\}$.} We may then say that $f$ is a contraction with contraction rate at most $k$.
\end{defn}
The classical result known as contraction mapping principle states the following:
\begin{thm}[Contraction Mapping Principle]\label{classContrMapThm}
Suppose $(X, \, d)$ is a complete metric space and $f: X \rightarrow X$ is a contraction on $X$ in the sense of definition~\ref{contractMapDef}. Then $\exists!$ $z \in X$ such that $\forall \, y \in X$ we have $\lim_{n \rightarrow \infty}f^n(y) = z$.
\end{thm}
\begin{proof}
The proof can be found in nearly every textbook on point set topology such as \cite{SimmonsG}, for instance.
\end{proof}
In our application we will exploit the following natural extension of definition~\ref{contractMapDef}:
\begin{defn}\label{contractMapFamilyDefn}
Suppose $(X, \, d)$ is a complete metric space. We say that a family of maps $\mathcal{F} \subseteq \{f \, | \, f: X \rightarrow X\}$ is an \emph{equi-contraction family} if $\exists \, k < 1$ such that $\forall \, f \in \mathcal{F}$ and $\forall \, x, \, y \in X$ we have $d(f(x), \, f(y)) \leq k \cdot d(x, y)$.
\end{defn}
Evidently, if the family $\mathcal{F}$ of contractions is finite, one can take the maximum of a set $K = \{k_f \, | \, \forall \, x, \, y \in X$ we have $d(f(x), \, f(y)) \leq k_f \cdot d(x, y)\}$ so that we immediately deduce the following important (for our application) corollary:
\begin{cor}\label{finiteContractFamProp}
If $\mathcal{F}$ is any finite family of contractions on the metric space $X$ then $\mathcal{F}$ is an equi-contraction family.
\end{cor}
The classical contraction mapping principle says that every contraction map on a complete metric space has a unique fixed point. Here we need a slight extension of theorem~\ref{classContrMapThm}, which probably appears as an exercise in some point set topology or real analysis textbook, but for the sake of completeness it is included in our paper.
\begin{thm}\label{ContractMapPrincipleForEquiFam}
Suppose we are given an equi-contraction family $\mathcal{F}$ on the complete metric space $(X, \, d)$. Suppose further that every $f \in \mathcal{F}$ has the same unique fixed point $z$ (in accordance with theorem~\ref{classContrMapThm}). Consider any sequence of composed functions $g_1 = f_1, \, g_2 = f_2 \circ g_1 \, \ldots, g_n = f_n \circ g_{n-1}$ where each $f_i \in \mathcal{F}$ (it is allowed for $f_i = f_j$ when $i \neq j$). Then $\forall \, y \in X$ $lim_{n \rightarrow \infty}g_n(y) = z$ exponentially fast for some constant $k < 1$. In particular, the convergence rate does not depend either on the sequence $\{g_i\}_{i=1}^{\infty}$ (as long as it is constructed in the manner described above). Moreover, in case $d$ is a bounded metric (i.e. $\sup_{x, \, y \in X}d(x, \, y) < \infty$), the convergence rate does not depend even on the choice of the initial point $y \in X$.
\end{thm}
\begin{proof}
Since all the functions $f_i$ have the same fixed point $z$, it is clear by induction that $\forall \, n$ we have $g_n(z) = z$. Since $\mathcal{F}$ is an equi-contraction family, in accordance with definition~\ref{contractMapFamilyDefn} $\exists \, k < 1$ such that $d(f(x), \, f(y)) \leq k \cdot d(x, y)$. We now have $d(g_1(y), \, z) = d(f_1(y), \, f_1(z)) \leq k \cdot d(y, \, z)$. If $d(g_m(y), \, z) \leq k^m \cdot d(y, \, z)$, then $d(g_{m+1}(y), \, z) = d(f_{m+1}(g_m(y)), \, f_{m+1}(z)) \leq k \cdot d(g_m(y), \, z) \leq k \cdot (k^m \cdot d(y, \, z)) = k^{m+1} \cdot d(y, \, z)$ so that by induction it follows that $\forall \, n \in \mathbb{N}$ we have $d(g_n(y), \, z) \leq k^n \cdot d(y, \, z)$. But $k<1$ so that $d(g_n(y), \, z) \rightarrow 0$ exponentially fast as $n \rightarrow \infty$ which is another way of stating the first desired conclusion. If $\sup_{x, \, y \in X}d(x, \, y) < \infty$ then $d(g_n(y), \, z) \leq k^n \cdot d(y, \, z) \leq k^n \cdot \sup_{x, \, y \in X}d(x, \, y)$.
\end{proof}
\subsection{What does Theorem~\ref{ContractMapPrincipleForEquiFam} tell us about Markov Chains?}\label{MarkChainApplSect}
Suppose $\mathcal{M}$ is a Markov chain on a finite state space $\mathcal{X}$ with transition matrix $P = \{p_{x \rightarrow y}\}_{x, \, y \in \mathcal{X}}$. Clearly $P$ extends to the linear map on the free vector space $\mathbb{R}^{\mathcal{X}}$ spanned by the point mass probability distributions which form an orthonormal basis of this   vector space (isomorphic to $\mathbb{R}^{|\mathcal{X}|}$, of course) under the $L_1$ norm defined as the sum of the absolute values of the coordinates: $\|\sum_{x \in \mathcal{X}}r_x x\|_{L_1}=\sum_{x \in \mathcal{X}}|r_x|$. The linear endomorphism $P$ defined by the matrix $\{p_{x \rightarrow y}\}_{x, \, y \in \mathcal{X}}$ with respect to the basis $\mathcal{X}$ restricts to the probability simplex
\begin{equation}\label{probSimplexEqDef}
\triangle_{\mathcal{X}} = \left\{\sum_{x \in \mathcal{X}}r_x x \, | \, \forall \, x \in \mathcal{X} \, 0 \leq r_x \leq 1 \sum_{x \in \mathcal{X}}r_x = 1\right\}
\end{equation}
(which is closed and bounded in $\mathbb{R}^{\mathcal{X}}$ and hence is compact which is way stronger than we need). The following well-known fact from basic Markov chain theory allows us to apply the tools from subsection~\ref{ContractionMappingSubsect}. For the sake of completeness a proof is included.
\begin{thm}\label{contractThmForMarkov}
Suppose $\mathcal{M}$ with notation as above is an \textit{irreducible} Markov chain. (meaning that $\forall \, x, y \in \mathcal{X}$ we have $p_{x \rightarrow y}>0$). Then $P = \{p_{x \rightarrow y}\}_{x, \, y \in \mathcal{X}}: \triangle_{\mathcal{X}} \rightarrow \triangle_{\mathcal{X}}$ (see equation~\ref{probSimplexEqDef}) is a contraction (see definition~\ref{contractMapDef}) on the complete and bounded probability simplex $\triangle_{\mathcal{X}}$ with respect to the metric induced by the $L_1$ norm i.e. $\|\vec{u}\|_{L_1} = \sum_{x \in \mathcal{X}}|u_x|$ where $\vec{u} = \sum_{x \in \mathcal{X}}u_x$.\footnote{Of course, the total variation norm, which is a constant scaling of the $L_1$ norm by a factor of $\frac{1}{2}$, can be used in place of the $L_1$ norm alternatively.} Moreover, the contraction rate (see definition~\ref{contractMapDef}) is at most $1-|\mathcal{X}|\epsilon$ where $\epsilon > 0$ is any number smaller than $\min_{x, \, y \in \mathcal{X}}p_{x \rightarrow y}$.
\end{thm}
\begin{proof}
First notice that given any Markov transition matrix $R = \{r_{x \rightarrow y}\}_{x, \, y \in \mathcal{X}}$, and
any two probability distributions $\pi$ and $\sigma \in \triangle_{\mathcal{X}}$, we have $$\|R(\pi - \sigma)\|_{L_1} = \sum_{y \in \mathcal{X}}\left|\sum_{x \in \mathcal{X}}r_{x \rightarrow y}(\pi(x)-\sigma(x))\right| \leq \sum_{y \in \mathcal{X}}\sum_{x \in \mathcal{X}}r_{x \rightarrow y}|\pi(x)-\sigma(x)| =$$$$= \sum_{x \in \mathcal{X}}\sum_{y \in \mathcal{X}}r_{x \rightarrow y}|\pi(x)-\sigma(x)| = \sum_{x \in \mathcal{X}}|\pi(x)-\sigma(x)| = \|\pi - \sigma\|_{L_1}.$$ In summary, we have shown that $$\forall \, \text{ Markov transition matrix }R = \{r_{x \rightarrow y}\}_{x, \, y \in \mathcal{X}} \text{ on the state space } \mathcal{X} \text{ and }$$$$\forall \text{ probability distributions }\pi, \, \sigma \in \triangle_{\mathcal{X}} \text{ we have}$$
\begin{equation}\label{genFactAboutMarkInProof}
\|R(\pi - \sigma)\|_{L_1} = \|R(\pi) - R(\sigma)\|_{L_1} \leq \|\pi - \sigma\|_{L_1}
\end{equation}
There is one more simple fact we observe: let $J$ denote an $\mathcal{X} \times \mathcal{X}$ matrix with all entries
equal to $1$. Given any vector $\vec{u} = \sum_{x \in \mathcal{X}} u_x x$, we have $J \cdot \vec{u} = \vec{v} = \sum_{x \in \mathcal{X}} v_x x$ where $\forall \, y \in \mathcal{X}$ we have $v_y = \sum_{x \in \mathcal{X}} u_x$ independently of $y$. It is clear then that the kernel of the matrix $J$, $$Ker(J) = \{\vec{u} \, | \, \vec{u} = \sum_{x \in \mathcal{X}} u_x x \text{ and } \sum_{x \in \mathcal{X}} u_x = 0\}.$$ In particular, if $\pi$ and $\sigma$ are probability distributions on $\mathcal{X}$, then the sums of coordinates $\sum_{x \in \mathcal{X}}(\pi(x)) = \sum_{x \in \mathcal{X}}(\sigma(x)) = 1$ so that the vector $\pi - \sigma \in Ker(J)$ i.e. $J(\pi - \sigma) = 0$. In summary, we deduce the following:
\begin{equation}\label{JmatrEqProof}
\forall \text{ probability distributions } \pi \text{ and } \sigma \in \triangle_{\mathcal{X}} \text{ we have } J(\pi - \sigma) = 0.
\end{equation}
The assumption that $p_{x \rightarrow y}>0$ together with the assumption that $\mathcal{X}$ is a finite set imply that we can find a positive number $\epsilon>0$ such that
$0<\epsilon<\min\{p_{x \rightarrow y} \, | \, x, \, y \in \mathcal{X}\}$. Let $N = |\mathcal{X}|$ denote the size of the state space $\mathcal{X}$ and notice that by the choice of $\epsilon$ in the previous sentence, $\forall \, x \in \mathcal{X}$ we have $N \cdot \epsilon < \sum_{y \in \mathcal{X}}p_{x \rightarrow y} = 1$ so that $\alpha = 1-N \epsilon>0$. We can now write
\begin{equation}\label{PMatrDecomposeEqProof}
P=(P - \epsilon J)+\epsilon J = \alpha \left(\frac{1}{\alpha}(P - \epsilon J)\right) + \epsilon J = \alpha Q + \epsilon J
\end{equation}
where $Q = \frac{1}{\alpha}(P - \epsilon J) = \{q_{x \rightarrow y}\}_{x, \, y \in \mathcal{X}}$ is a stochastic matrix, i.e. $\forall \, x \in \mathcal{X}$ the sum of the entries $$\sum_{y \in \mathcal{X}}q_{x \rightarrow y} = \sum_{y \in \mathcal{X}}\frac{p_{x \rightarrow y}-\epsilon}{\alpha} = \frac{\sum_{y \in \mathcal{X}}(p_{x \rightarrow y}-\epsilon)}{1-N\epsilon} = \frac{\left(\sum_{y \in \mathcal{X}}p_{x \rightarrow y}\right) - N\epsilon}{1 - N\epsilon}=1.$$ so that $Q$ is
a matrix representing a Markov chain on the state space $\mathcal{X}$. Now, given any two distributions $\pi$ and $\sigma \in \triangle_{\mathcal{X}}$, using the decomposition of the matrix $P$ given in equation~\ref{PMatrDecomposeEqProof} together with the facts expressed in equation~\ref{JmatrEqProof} we obtain $$P(\pi - \sigma) = (\alpha Q + \epsilon J)(\pi - \sigma) = \alpha Q(\pi - \sigma) + \epsilon J(\pi - \sigma) = \alpha Q(\pi - \sigma)$$ so that, since $Q$ is a matrix which represents a Markov chain, the fact expressed in equation~\ref{genFactAboutMarkInProof} readily gives us the desired conclusion that
$$\|P(\pi - \sigma)\|_{L_1} = \|\alpha Q(\pi - \sigma)\|_{L_1} = \alpha \|Q(\pi - \sigma)\|_{L_1} \leq \alpha \|\pi - \sigma\|_{L_1}$$ which shows that $P$ is a contraction since we demonstrated before that $0 < \alpha < 1$.
\end{proof}
In corollary~\ref{finiteContractFamProp} we saw that any finite family of contraction maps is an equi-contraction family. For Markov transition matrices (also called stochastic matrices in the literature) significantly more is true. The following notion is naturally motivated by definition~\ref{contractMapFamilyDefn} and theorem~\ref{contractThmForMarkov}.
\begin{defn}\label{contractFamMarkovDef}
Given a family of Markov transition matrices $$\mathcal{F} = \{\{p_{x \rightarrow y}^i\}_{x, \, y \in \mathcal{X}} \, | \, i \in \mathcal{I}, \pi \in \triangle_{\mathcal{X}} \text{ and } \forall \, i \in \mathcal{I} \text{ and } \forall \, y \in \mathcal{X} \text{ we have } $$$$\sum_{x \in \mathcal{X}}p_{x \rightarrow y}^i \pi_x = \pi_y \text{ and } \beta = \inf_{i \in \mathcal{I}, \, x \text{ and } y \in \mathcal{X}}p_{x \rightarrow y}^i >0\}$$ indexed by some set $\mathcal{I}$, sharing a common stationary distribution $\pi$ and such that the greatest lower bound of all the entries from all the matrices in $\mathcal{F}$, let's call it $\beta$, is strictly positive (or, equivalently, is not $0$) we say that $\mathcal{F}$ is a family of \emph{interchangeable} Markov transition matrices \emph{with lower bound} $\beta$.
\end{defn}
Apparently, theorem~\ref{contractThmForMarkov} immediately implies the following
\begin{cor}\label{immidiateCorForMarkTrans}
Every interchangeable family $\mathcal{F}$ of Markov transition matrices with lower bound $\beta$ is an equi-contraction family with a common contraction rate at most $\alpha = 1 - |\mathcal{X}|\epsilon$ for any $\epsilon$ with $0 < \epsilon < \beta$.
\end{cor}
Moreover, families of interchangeable Markov transition matrices can often be easily expended as follows.
\begin{cor}\label{equicontrProbSimplexCor}
Suppose that a family $\mathcal{F}$ of Markov transition matrices over the same state space $\mathcal{X}$ is interchangeable with lower bound $\beta$. Then so is the convex hull of the family $\mathcal{F}$, $$\triangle(\mathcal{F}) = \{T \, | \, T = \sum_{i=1}^k t_i M_i \text{ where } k \in \mathbb{N} \text{ and } \forall \, 0<i<k \text{ we have } 0< t_i \leq 1 \, \sum_{i=1}^k t_i = 1\}.$$
\end{cor}
\begin{proof}
Given a matrix $T = \{t_{x \rightarrow y}\}_{x, \, y \in \mathcal{X}} \in \triangle(\mathcal{F})$, we can write $T = \sum_{j=1}^k t_j M_j \in \triangle(\mathcal{F})$ with $M_j = \{p_{x \rightarrow y}^j\}_{x, \, y \in \mathcal{X}} \in \mathcal{F}$, $0 < t_j \leq 1$ and $\sum_{j=1}^k t_j = 1$. But then $\forall x, \, y \in \mathcal{X}$ we have $t_{x \rightarrow y} = \sum_{j=1}^k t_j \cdot p_{x \rightarrow y}^j \geq \sum_{j=1}^k t_j \cdot b = b$ so that the desired conclusion follows at once.
\end{proof}
Combining theorem~\ref{contractThmForMarkov}, corollary~\ref{finiteContractFamProp} and corollary~\ref{equicontrProbSimplexCor} readily gives the following
\begin{cor}\label{equicontrPolyhhidraCor}
Suppose we are given a finite family $\mathcal{F}$ of Markov transition matrices such that all the entries of each matrix $M \in \mathcal{F}$ are strictly positive. Then $\triangle(\mathcal{F})$ is an equi-contraction family.
\end{cor}
Corollary~\ref{equicontrPolyhhidraCor} extends the applicability of the finite population Geiringer theorem appearing in \cite{MitavRowGeirMain} and in \cite{MitavRowGeirGenProgr} (and, possibly some other homogenous-time Markov chain constructions) to non-homogenous time Markov chains generated by arbitrary stochastic processes in the sense below.
\begin{thm}\label{MarkovNonMarkovSubtleExtCor}
Consider any finite set $\mathcal{X}$. Let $\mathcal{F}$ denote a finite family of Markov transition matrices on $\mathcal{X}$ such that all the entries of each matrix $M \in \mathcal{F}$ are strictly positive and all the matrices in $\mathcal{F}$ have a common stationary distribution $\pi$. Now consider any stochastic process $\{Z_n\}_{n=1}^{\infty}$ with each $Z_n = (F_n, \, X_n)$ on $\mathcal{F} \times \mathcal{X}$ having the following properties:
\begin{equation}\label{reqStoch1Eq}
F_0 \text{ and } X_0 \text{ are independent random variables.}
\end{equation}
$$\text{For } n \geq 1 \; F_n \text{ does not depend on } X_n, \, X_{n+1}, \ldots, \text{(however, it may depend on }$$
\begin{equation}\label{reqStoch2Eq}
X_0, \, X_1, \ldots, X_{n-1} \text{ as well as many other implicit parameters).}
\end{equation}
The stochastic process $X_n$ is a non-homogenous time Markov chain on $\mathcal{X}$ with transition matrices $F_n(w)$. More explicitly
$$\text{If } F_k(\omega) = \{p_{x \rightarrow y}^{k}\}_{x, \, y \in \mathcal{X}} \text{ then } \forall \, y \in \mathcal{X} \text{ we have }$$
\begin{equation}\label{reqStoch3Eq}
P(X_n = y) = \sum_{x \in \mathcal{X}}P(X_{n-1} = x)p^{n-1}_{x \rightarrow y}.
\end{equation}
Then the non-homogenous time Markov chain converges to the unique stationary distribution $\pi$ exponentially fast regardless of the initial distribution of $X_0$. More precisely, $\exists \, \alpha \in (0, \, 1)$ such that $\forall \, t \in \mathbb{N}$ we have $$\|P(X_t = \cdot) - \pi\|_{L_1} \leq \alpha^t$$ where $P(X_t = \cdot)$ denotes the probability distribution of the random variable $X_t$.
\end{thm}
\begin{proof}
Observe that if we want to compute the distribution of $X_1$ given the distribution of $X_0$, we need to select a Markov transition matrix $M = \{m_{x \rightarrow y}\}_{x, \, y \in \mathcal{X}} \in \mathcal{F}$ with respect to the probability distribution of $F_0$ which is \emph{independent} of $X_0$. The value of $X_1$ is then obtained by selecting a value $x$ of $X_0$ with respect to the initial distribution $P(X_0 = \cdot)$ and then obtaining the next state $X_1 = y$ with probability $P(X_1 = y) = m_{x \rightarrow y}$. Thereby $\forall \, y \in \mathcal{X}$ we may write
$$P(X_1 = y) = \sum_{M \in \mathcal{F}}\sum_{x \in \mathcal{X}}P(F_0 = M \text{ and } X_0 = x)m_{x \rightarrow y} \overset{\text{by independence}}{=}$$
$$= \sum_{M \in \mathcal{F}}\sum_{x \in \mathcal{X}}P(X_0 = x)P(F_0 = M)m_{x \rightarrow y}=$$
\begin{equation}\label{totalProbEqProof1}
= \sum_{x \in \mathcal{X}}P(X_0 = x)\sum_{M \in \mathcal{F}}P(F_0 = M)m_{x \rightarrow y}.
\end{equation}
Since $\mathcal{F}$ is a finite set, $\forall \, M \in \mathcal{F}$ we have $P(F_0 = M) \in [0, 1]$ and $\sum_{x \in \mathcal{X}}P(X_0 = x) = 1$, we deduce that
the matrix $T_0 = \sum_{M \in \mathcal{F}}P(F_0 = M) \cdot M \in \triangle(\mathcal{F})$ is a Markov transition matrix and equation~\ref{totalProbEqProof1} can be alternatively written in the vector form as
\begin{equation}\label{totalProbEqProof1Alt}
P(X_1 = \cdot) = T_0 \cdot P(X_0 = \cdot).
\end{equation}
Continuing inductively, if we assume
\begin{equation}\label{totalProbEqProof1Ass}
P(X_k = \cdot) = T_{k-1} \circ \ldots \circ T_1 \circ T_0 \cdot P(X_0 = \cdot)
\end{equation}
for $k \geq 1$ where the Markov transition matrices $T_i \in \triangle(\mathcal{F})$, then it follows analogously to the above reasoning that $$P(X_{k+1} = y) = \sum_{M \in \mathcal{F}}\sum_{x \in \mathcal{X}}P(F_k = M \text{ and } X_k = x)m_{x \rightarrow y} \overset{\text{by independence}of F_k \text{ and }X_k}{=}$$
$$= \sum_{x \in \mathcal{X}}P(X_k = x)\sum_{M \in \mathcal{F}}P(F_k = M)m_{x \rightarrow y}$$ so that for the same reasons as before we may conclude that
$$P(X_{k+1} = \cdot) = T_k \cdot P(X_k = \cdot) = T_k \cdot (T_{k-1} \circ \ldots \circ T_1 \circ T_0 \cdot P(X_0 = \cdot))=$$
\begin{equation}\label{totalProbEqProof2Alt}
= T_k \circ T_{k-1} \circ \ldots \circ T_1 \circ T_0 \cdot P(X_0 = \cdot).
\end{equation}
where $T_k \in \triangle(\mathcal{F}) = \sum_{M \in \mathcal{F}}P(F_k = M) \cdot M \in \triangle(\mathcal{F})$ for the same reason as $T_0 \in \mathcal{F}$. We now conclude by induction that $\forall \, t \in \mathbb{N}$ we have
\begin{equation}\label{totalProbEqProof2Ass}
P(X_t = \cdot) = T_{t-1} \circ \ldots \circ T_1 \circ T_0 \cdot P(X_0 = \cdot)
\end{equation}
where $\forall \, i \in \mathbb{N} \cup \{0\}$ we have $T_i \in \triangle(\mathcal{F})$.
According to corollary~\ref{equicontrPolyhhidraCor} the family of Markov transition matrices $\triangle(\mathcal{F})$ is an equi-contraction family with the same common stationary distribution $\pi$ and now the desired conclusion follows immediately from theorem~\ref{ContractMapPrincipleForEquiFam}.
\end{proof}
\begin{rem}\label{paradoxRem}
It is interesting to notice that the non-homogenous time Markov process $X_n$ in theorem~\ref{MarkovNonMarkovSubtleExtCor} may be generated by non-Markovian processes   $F_n$ where the Markov transition matrices $F_n$ depend not only on the past history $F_0, \, F_1, \ldots, F_{n-1}$ but also on the history of the stochastic process $X_n$ itself. This property is interesting not only from the mathematical point of view but also in regard to the main subject of the current paper: the application to the Monte Carlo Tree search method. Due to the past history in a certain game as well as other possibly hidden circumstances (such as human mood, psychological state etc.), a player may suspect the states being interchangeable to bigger or smaller degree. Theorems like \ref{MarkovNonMarkovSubtleExtCor} demonstrate that in most cases this will not matter in the limiting case which strengthens the theoretical foundation in support of the main ideas presented in this work.
\end{rem}
One can extend theorem~\ref{MarkovNonMarkovSubtleExtCor} further to be applicable to a wider class of families of Markov transition matrices having a common stationary distribution than just these having all positive entries.
\begin{defn}\label{irreducibleAperiodicFamDef}
We say that a family $\mathcal{F}$ of Markov transition matrices is \emph{irreducible and aperiodic} with a common stationary distribution $\pi$ if $\pi$ is a stationary distribution of every matrix in $\mathcal{F}$ and $\exists \, k \in \mathbb{N}$ such that $\forall$ sequence of transformations $\{M_i\}_{i=1}^k$ with $M_i \in \mathcal{F}$ the composed Markov transition matrix $T = M_1 \circ M_2 \circ \ldots \circ M_k$ has strictly positive entries and $\pi$ is a stationary distribution of every Markov transition matrix $M \in \mathcal{F}$. We also say that $k$ is the \emph{common reachable index}.
\end{defn}
If we were to start with a finite irreducible and aperiodic family of Markov transition matrices $\mathcal{F}$ with a common reachable index $k$ in the sense of definition~\ref{irreducibleAperiodicFamDef} then the corresponding family
\begin{equation}\label{defOfCompFamOfMarkTransEq}
\widetilde{\mathcal{F}} = \{L \, | \, L = M_1 \circ M_2 \circ \ldots \circ M_k \text{ with } M_i \in \mathcal{F}\}
\end{equation}
has the size $|\widetilde{\mathcal{F}}| = |\mathcal{F}|^k < \infty$ and every matrix in the family $\widetilde{\mathcal{F}}$ has strictly positive entries. It follows immediately from corollary~\ref{equicontrPolyhhidraCor} that $\triangle(\widetilde{\mathcal{F}})$ is an equi-contraction family.
Now suppose that we are dealing with the same stochastic process as described in the statement of theorem~\ref{MarkovNonMarkovSubtleExtCor} with the only exception that the family $\mathcal{F}$ is a finite irreducible and aperiodic family with a common reachable index $k$ rather than ``a finite family of Markov transition matrices on $\mathcal{X}$ such that all the entries of each matrix $M \in \mathcal{F}$ are strictly positive". Notice that the proof of theorem~\ref{MarkovNonMarkovSubtleExtCor} does not use the assumption that the Markov transition matrix entries are strictly positive up to the last step following equation~\ref{totalProbEqProof2Ass}. Therefore, it follows that the same equation holds for a finite irreducible and aperiodic family of Markov transition matrices, i.e. \begin{equation}\label{nonHomogMarkovEq}
\forall \, t \in \mathbb{N} \text{ we have } P(X_t = \cdot) = T_{t-1} \circ \ldots \circ T_1 \circ T_0 \cdot P(X_0 = \cdot)
\end{equation}
where $\forall \, i \in \mathbb{N}$ we have $T_i \in \triangle(\mathcal{F})$. We now observe the following simple fact.
\begin{lem}\label{convHullInclLem}
The family of linear transformations (and Markov transition matrices in particular)
$$\widetilde{\triangle(\mathcal{F})} \subseteq \triangle(\widetilde{\mathcal{F}})$$ where
\begin{equation}\label{TriangleTransforms}
\widetilde{\triangle(\mathcal{F})} = \{T \, | \, T = T_1 \circ T_2 \circ \ldots \circ T_k \text{ with } T_i \in \triangle(\mathcal{F})\}
\end{equation}
and the family $\triangle(\widetilde{\mathcal{F}})$ is the convex hull of the family $\widetilde{\mathcal{F}}$  introduced in equation~\ref{defOfCompFamOfMarkTransEq} in the sense of the defining equation in corollary~\ref{equicontrProbSimplexCor}.
\end{lem}
\begin{proof}
Given a transformation
\begin{equation}\label{defOfTransfTProofEq}
T = T_1 \circ T_2 \circ \ldots \circ T_k \in \widetilde{\triangle(\mathcal{F})},
\end{equation}
since each $T_i \in \triangle(\mathcal{F})$, we have
\begin{equation}\label{convEqProof}
\forall \, i \text{ with } 1 \leq i \leq k \text{ we have } T_i = \sum_{j=1}^{l(i)}t_j^i M_{j(i)} \text{ with } 0 \leq t_j^i \leq 1 \text{ and } \sum_{j=1}^{l(i)}t_j^i = 1.
\end{equation}
Plugging equation~\ref{convEqProof} into equation~\ref{defOfTransfTProofEq} and using the linearity of $T_i$s we obtain
$$T = \left( \sum_{j=1}^{l(1)}t_j^1 M_{j(1)} \right) \circ \left(\sum_{j=1}^{l(2)}t_j^2 M_{j(2)} \right) \circ
\ldots \circ \left(\sum_{j=1}^{l(i)}t_j^i M_{j(i)} \right) \circ \ldots \circ \left(\sum_{j=1}^{l(k)}t_j^k M_{j(k)} \right)=$$$$=\sum_{j(1)=1}^{l(1)}\sum_{j(2)=1}^{l(2)} \ldots \sum_{j(k)=1}^{l(k)} \left(\prod_{i=1}^k t_{j(i)}^i\right)M_{j(1)} \circ M_{j(2)} \circ \ldots \circ M_{j(k)} \in \widetilde{\mathcal{F}}$$ since
$0 \leq \prod_{i=1}^k t_{j(i)}^i \leq 1$ and
$$\sum_{j(1)=1}^{l(1)}\sum_{j(2)=1}^{l(2)} \ldots \sum_{j(k)=1}^{l(k)} \left(\prod_{i=1}^k t_{j(i)}^i\right) = \left(\sum_{j=1}^{l(1)} t_j^1 \right) \left(\sum_{j=1}^{l(2)} t_j^2 \right) \ldots \left(\sum_{j=1}^{l(k)} t_j^k \right) = 1$$ from equation~\ref{convEqProof} so that the desired conclusion follows at once.
\end{proof}
Now continue with equation~\ref{nonHomogMarkovEq} so that we can write
$$\forall \, t \in \mathbb{N} \text{ we have } P(X_t = \cdot) = T_{t-1} \circ \ldots \circ T_1 \circ T_0 \cdot P(X_0 = \cdot)=$$$$\underset{r-\text{fold composition}}{=\underbrace{T_{t-1} \circ \ldots \circ T_{m \cdot k + 1} \circ T_{m \cdot k}}} \underset{k-\text{fold composition}}{\circ \underbrace{T_{m \cdot k - 1} \circ \ldots \circ T_{(m-1) \cdot k + 1} \circ T_{(m-1) \cdot k}}} \circ \ldots $$$$ \ldots \circ \underset{k-\text{fold composition}}{\underbrace{T_{2k-1} \circ \ldots \circ T_{k+1} \circ T_k}} \circ \underset{k-\text{fold composition}}{\underbrace{T_{k-1} \ldots \circ T_1 \circ T_0}} \cdot P(X_0 = \cdot)=$$
\begin{equation}\label{composeEq}
= T_{t-1} \circ \ldots \circ T_{m \cdot k + 1} \circ T_{m \cdot k} \circ F_{m-1} \circ F_{m-2} \circ \ldots \circ F_1 \circ F_0 \cdot P(X_0 = \cdot)
\end{equation}
where $m = \lfloor \frac{t}{k} \rfloor$ and $r<k$ is the remainder after dividing $t$ by $k$ and each $F_i \in \widetilde{\triangle(\mathcal{F})} \subseteq \triangle(\widetilde{\mathcal{F}})$ thanks to lemma~\ref{convHullInclLem}. Since $\triangle(\widetilde{\mathcal{F}})$ is an equi-contraction family (see equation~\ref{defOfCompFamOfMarkTransEq} and the discussion which follows this equation), it follows immediately that we can find a constant $\alpha \in [0, \, 1)$ such that $$\|F_{m-1} \circ F_{m-2} \circ \ldots \circ F_1 \circ F_0 \cdot P(X_0 = \cdot)\|_{L_1} < \alpha^m \cdot \|P(X_0 = \cdot)\|_{L_1}.$$ Furthermore, according to equation~\ref{genFactAboutMarkInProof} which concludes the first part of the proof of theorem~\ref{contractThmForMarkov}, we also have $$\|T_{t-1} \circ \ldots \circ T_{m \cdot k + 1} \circ T_{m \cdot k} \circ F_{m-1} \circ F_{m-2} \circ \ldots \circ F_1 \circ F_0 \cdot P(X_0 = \cdot)\|_{L_1}=$$
$$=\|(T_{t-1} \circ \ldots \circ T_{m \cdot k + 1} \circ T_{m \cdot k}) \circ (F_{m-1} \circ F_{m-2} \circ \ldots \circ F_1 \circ F_0 \cdot P(X_0 = \cdot))\|_{L_1} \leq $$$$\leq \|F_{m-1} \circ F_{m-2} \circ \ldots \circ F_1 \circ F_0 \cdot P(X_0 = \cdot)\|_{L_1} < \alpha^m \cdot \|P(X_0 = \cdot)\|_{L_1}.$$
The observations above lead to the following extension of theorem~\ref{MarkovNonMarkovSubtleExtCor}.
\begin{thm}\label{MarkovNonMarkovSubtleExtCorExt}
Consider any finite set $\mathcal{X}$. Suppose $\mathcal{F}$ a is a finite irreducible and aperiodic family with a common reachable index $k$ and all the matrices in $\mathcal{F}$ have a common stationary distribution $\pi$. Now consider any stochastic process $\{Z_n\}_{n=1}^{\infty}$ with each $Z_n = (F_n, \, X_n)$ on $\mathcal{F} \times \mathcal{X}$ having the following properties:
\begin{equation}\label{reqStoch1Eq1}
F_0 \text{ and } X_0 \text{ are independent random variables.}
\end{equation}
$$\text{For } n \geq 1 \; F_n \text{ does not depend on } X_n, \, X_{n+1}, \ldots, \text{(however, it may depend on }$$
\begin{equation}\label{reqStoch2Eq1}
X_0, \, X_1, \ldots, X_{n-1} \text{ as well as many other implicit parameters).}
\end{equation}
The stochastic process $X_n$ is a non-homogenous time Markov chain on $\mathcal{X}$ with transition matrices $F_n(w)$. More explicitly
$$\text{If } F_k(\omega) = \{p_{x \rightarrow y}^{k}\}_{x, \, y \in \mathcal{X}} \text{ then } \forall \, y \in \mathcal{X} \text{ we have }$$
\begin{equation}\label{reqStoch3Eq1}
P(X_n = y) = \sum_{x \in \mathcal{X}}P(X_{n-1} = x)p^{n-1}_{x \rightarrow y}.
\end{equation}
Then the non-homogenous time Markov chain converges to the unique stationary distribution $\pi$ exponentially fast regardless of the initial distribution of $X_0$. More precisely, $\exists \, \alpha \in (0, \, 1)$ such that $\forall \, t \in \mathbb{N}$ we have $$\|P(X_t = \cdot) - \pi\|_{L_1} \leq \alpha^{m(t)}$$ where $P(X_t = \cdot)$ denotes the probability distribution of the random variable $X_t$ and $m(t) = \lfloor \frac{t}{k} \rfloor$.
\end{thm}
\section{Conclusions and Upcoming Work}\label{conclSect}
This is the first in a sequel of papers leading to the development and applications of very promising and novel Monte Carlo sampling techniques for reinforcement learning in the setting of POMDPs (partially observable Markov decision processes). In this work we have established a version of Geiringer-like theorem with non-homologous recombination well-suitable for the development of dynamic programming Monte Carlo search algorithms to cope with randomness and incomplete information. More explicitly, the theorem provides an insight into how one may take full advantage of a sample of seemingly independent rollouts by exploiting symmetries within the space of observations as well as additional similarities that may be provided as expert knowledge. Due to space limitations the actual algorithms will appear in the upcoming works. Additionally, the general finite-population Geiringer theorem appearing in the PhD thesis of the first author as well as in \cite{MitavRowGeirMain} and \cite{MitavRowGeirGenProgr} has been further strengthened with the aim of amplifying the reasons why the above ideas are highly promising in applications, not mentioning the mathematical importance.


\begin{thebibliography}
\bibitem{UCBAgraval} Agrawal, R.
(1995). Sample Mean Based Index Policies with $O(\log n)$ Regret for the Multi-armed Bandit Problem. {\it Advances in Applied Probability}, {\bf 27}: 1054--1078.
\bibitem{Antonisse}  Antonisse, J. (1989).  A new interpretation of schema notation
that overturns the binary encoding constraint. \textit{Procedings of the Third International Conference
on Genetic Algorithms}. Morgan Kaufmann, pages  86--97.
\bibitem{UCBAuer} Auer, P.
(2002). Using Confidence Bounds for Exploration-Exploitation Trade-offs. {\it Journal of machine Learning Research}, {\bf 3}: 397--422.
\bibitem{MarkovInequality} Auger, A. and Doerr, B., ``Theory of Randomized Search Heuristics: Foundations and Recent Developments." World Scientific Publishing Company, 2011.
\bibitem{MCTGoChesl} Chaslot, G. Saito, J., Bouzy, B., Uiterwijk, J. and van den Herik, H. (2006) Monte-Carlo Strategies for Computer Go. \textit{Proceedings of the} $18^{\text{th}}$ \textit{Belgian-Dutch Conference on Artificial Intelligence}, pages 83-90.
\bibitem{MCTGoCoulomR} Coulom, R. (2007) Comparing Elo Ratings of Move Patterns in the Game of Go. \textit{Computer Games Workshop 2007}.
\bibitem{DummittFoote} Dummit, D. and Foote, R., ``Abstract Algebra." Prentice-Hall, 1991.
\bibitem{GeirOrigion} Geiringer, H.
(1944). On the Probability of Linkage in Mendelian Heredity. {\it Annals of Mathematical Statistics}, {\bf 15}: 25--57.
\bibitem{UCBKaelbling} Kaelbling, L.
(1994). Associative Reinforcement Learning: A Generate and Test Algorithm. {\it Machine Learning}, {\bf 15}: 299--319.
\bibitem{UCBKaelbling} Kaelbling, L.
(1994). Associative Reinforcement Learning: Functions in $k$-DNF. {\it Machine Learning}, {\bf 15}: 279--298.
\bibitem{KocsisL} Kocsis, L. and Szepesvari, C. (2006) bandit Based Monte-Carlo Planning. $15^{\text{th}}$\textit{ European Conference on Machine Learning}, pages 282-293.
\bibitem{MitavRowGeirGenProgr} Mitavskiy, B. and Rowe, J. (2005) A Schema-Based Version of Geiringer
Theorem for Nonlinear Genetic Programming with Homologous Crossover. \textit{Foundations of Genetic Algorithms 8 (FOGA-2005)}. Springer, pages 156-175.
\bibitem{MitavRowGeirMain} Mitavskiy, B. and Rowe, J.
(2006). An Extension of Geiringer Theorem for
a Wide Class of Evolutionary Algorithms. {\it Evolutionary
Computation}, {\bf (14)1}: 87--118.
\bibitem{LumpQuotMitavRoweWright} Mitavskiy, B., Rowe, J., Wright, A., Schmitt, L. (2006) Exploiting Quotients of Markov Chains to Derive Properties of the Stationary Distribution of the Markov chain associated to an Evolutionary Algorithm. \textit{Simulated Evolution and Learning (SEAL-2006)}, pages 726-733.
\bibitem{LumpQuotMitavCannings} Mitavskiy, B., Cannings C. (2007) An Improvement of the ``Quotient Construction" Method and Further Asymptotic Results on the Stationary Distribution of the Markov Chains Modeling Evolutionary Algorithms. \textit{IEEE Congress on Evolutionary Computation (CEC-2007)}, pages 2606-2613.
\bibitem{LumpQuotGPEM} Mitavskiy, B., Rowe, J., Wright, A., Schmitt, L. (2008). Quotients of Markov Chains and Asymptotic Properties of the Stationary Distribution of the Markov Chain Associated to an Evolutionary Algorithm. {\it Genetic Programming and Evolvable Machines}, {\bf 9(2)}: 109--123.
\bibitem{PoliSchema} Poli, R., Langdon, B. (1998) Schema Theory for Genetic Programming with One-point Crossover and Point Mutation. {\it Evolutionary Computation}, {\bf 6(3)}: 231--252.
\bibitem{PoliGeir} Poli, R., Stephens, C., Wright, A. and Rowe, J. (2003) A Schema Theory Based Extension of Geiringer's Theorem for Linear GP and Variable Length GAs Under Homologous Crossover. \textit{Foundations of Genetic Algorithms 7 (FOGA-2003)}. Morgan Kaufmann, pages 45-62.
\bibitem{ShmittL1} Schmitt, L.
(2001). Theory of Genetic Algorithms. {\it Theoretical Computer Science}, {\bf 259}: 1--61.
\bibitem{ShmittL2} Schmitt, L.
(2004). Theory of genetic algorithms II: models for genetic operators over the string-tensor
representation of populations and convergence to global optima for
arbitrary fitness function under scaling. {\it Theoretical Computer Science}, {\bf 310}: 181--231.
\bibitem{SimmonsG} Simmons, G. ``Introduction to Topology and Modern Analysis" 1963.
\bibitem{VoseM} Vose, M. ``The simple genetic algorithm: foundations and theory" MIT Press, 1999.
\bibitem{FF-Replan} Yoon, S., Fern, A. and Givan, R. (2007)  FF-Replan: A Baseline for Probabilistic Planning. \textit{International Conference on Automated Planning and Scheduling (ICAPS-2007)}.
\bibitem{ZinkewitchBucket} Zinkevich, Bowling, M., Burch, N. (2007)  A New Algorithm for Generating Equilibria in Massive Zero-Sum Games. \textit{AAAI Conference on Artificial Intelligence-2007}. pages 788-793.
%\bibitem{barabasi:2001} A. Barab\'{a}si, ``The Physics of the Web" {\it Physics World}, vol. 14(33), 2001.
%\bibitem[\protect\citeauthoryear{Barab\'{a}si}{2001}]{barabasi:2001} Barab\'{a}si, A. (2001). The Physics of the Web.
%{\it Physics World}, {\bf 14(33)}.
%\bibitem[\protect\citeauthoryear{Garey and Johnson}{1979}]{GareyJohn}
%Garey, M. and Johnson, D. (1979). {\it Computers and Intractability:
%A Guide to the Theory of NP-completeness}, W.H. Freeman and co., New
%York.
%\bibitem{GomHu} Gomory, R., Hu, T. (1969) \textit{Multiterminal Network
%Flows.} SIAM Journal on Applied Mathematics 9, pages 551-570.
%\bibitem[\protect\citeauthoryear{Gomory and Hu}{1969}]{GomHu} Gomory, R. and Hu, T.
%(1969). Multiterminal Network Flows. {\it SIAM Journal on Applied
%Mathematics}, {\bf 9}: 551--570.
%\bibitem[\protect\citeauthoryear{Hajek}{1982}]{Hajek} Hajek, B. (1982). Hitting Time and Occupation Time
%Bounds Implied by Drift Analysis with Applications. {\it Advances in Applied Probability}, {\bf 14}: 502--525.
%\bibitem[\protect\citeauthoryear{He and Yao}{2004}]{HeYao} He, J. and Yao, X. (2004). A study of drift
%analysis for estimating computation time of evolutionary algorithms
%{\it Natural Computing: an International Journal,} {\bf 3(1)}:
%21--35.
%\bibitem{Hu} Hu, T. (1974) \textit{Optimum Communication Spanning
%Trees.}SIAM Journal on Computing 3(3), pages 188-195.
%\bibitem[\protect\citeauthoryear{Hu}{1974}]{Hu} Hu, T. (1974). Optimum Communication Spanning
%Trees. {\it SIAM Journal on Computing,} {\bf 3(3)}: 188--195.
%\bibitem{ThomJensen} Jansen, T. and Theile, M. (2007) Stability in the self-organized
%evolution of networks. \textit{Proceedings of the Genetic and
%Evolutionary Computation Conference (GECCO 2007)}. ACM Press, New
%York, NY, pages 931-938.
%\bibitem[\protect\citeauthoryear{Jansen and Theile}{2007}]{Jensen} Jansen, T. and Theile, M. (2007). Stability in the self-organized
%evolution of networks {\it Proceedings of the Genetic and Evolutionary Computation Conference
%(GECCO 2007)}. ACM Press, New York, NY, pages 931-938.
%\bibitem{LehmannKauf} Lehmann, K., Kaufmann, (2005) M. Evolutionary
%Algorithms for the Self-organized Evolution of Networks.
%\textit{Proceedings of the Genetic and Evolutionary Computation
%Conference (GECCO 2005).} pages 563-570.
%\bibitem[\protect\citeauthoryear{Lehmann and Kaufmann}{2005}]{LehmannKauf} Lehmann, K., Kaufmann, M. Evolutionary
%Algorithms for the Self-organized Evolution of Networks {\it
%Proceedings of the Genetic and Evolutionary Computation Conference
%(GECCO 2005)}, pages 563-570.
%\bibitem{Watts} Watts, D. (2004) \textit{Small Worlds: the dynamics of Networks between Order and
%Randomness.} Prinston Unviersity Press.
%\bibitem{Chern} Motwani, R. and Raghavan, P., ``Randomized
%Algorithms." Cambridge University Press, New York (NY), 1995.
%\bibitem[\protect\citeauthoryear{Motwani and Raghavan}{1995}]{Chern}
%Motwani, R., Raghavan, P., (1995). {\it Randomized Algorithms}, Cambridge University Press.
%\bibitem[\protect\citeauthoryear{Syski}{1992}]{Syski}
%Syski, R. (1992). {\it Passage Times for Markov Chains}, IOS Press.
%\bibitem[\protect\citeauthoryear{Watts}{2004}]{Watts}
%Watts, D. (2004). {\it Small Worlds: the dynamics of Networks
%between Order and Randomness}, Prinston Unviersity Press.
\end{thebibliography}
\end{document}